\documentclass[twoside]{article}

\usepackage[accepted]{aistats2025}

\usepackage{import, graphicx, caption, subcaption, algorithm, algpseudocode, enumitem, adjustbox, dblfloatfix, afterpage}
\usepackage{booktabs}

\usepackage{amsfonts, amsthm, mathtools, thmtools, thm-restate, cancel, dsfont, aligned-overset} %

\DeclareMathOperator*{\argmax}{argmax}

\usepackage{natbib}
\setcitestyle{aysep={}} %Citation-related commands
\newcommand{\parencite}{\citep} % because of the switch from biblatex to natbib

\newtheorem{prop}{Proposition}
\newtheorem{prop_inf}[prop]{Informal Proposition}
\newtheorem{lem}{Lemma}

\newtheorem{defn}{Definition}

\usepackage{pgfplots}
\usepgfplotslibrary{external}
\everymath=\expandafter{\the\everymath\displaystyle}
\NewDocumentCommand{\incplt}{O{\columnwidth}m}{%
  \begin{center}
    \adjustbox{width=#1}{\import{./graphics/}{#2.pgf}}
  \end{center}
}
\pgfplotsset{width=10cm,compat=1.9}
\usepackage{xcolor}
\definecolor{hanblue}{rgb}{0.27, 0.42, 0.81}
\definecolor{deepred}{HTML}{900C3F}
\definecolor{deepgreen}{HTML}{2F6960}
\usepackage[hidelinks]{hyperref}
\hypersetup{
    colorlinks=true,
    linkcolor=hanblue,
    urlcolor=hanblue,
    citecolor=hanblue,
    anchorcolor=hanblue}

 \usepackage[capitalize,noabbrev]{cleveref}

\DeclarePairedDelimiter\parentheses{(}{)}
\DeclarePairedDelimiter\brackets{[}{]}

\newcommand{\lite}{\textsc{LITE}}
\newcommand{\ts}{\textsc{TS-MC}}
 
\newcommand{\ia}{\textsc{Independence Assumption}}
\newcommand{\ias}{\textsc{Indep. Assum.}}
\newcommand{\alite}{\textsc{A-LITE}}
\newcommand{\flite}{\textsc{F-LITE}}
\newcommand{\fvapor}{\textsc{F-VAPOR}}
\newcommand{\vapor}{\textsc{VAPOR}}
\newcommand{\est}{\textsc{EST}}
\newcommand{\ms}{\textsc{MEANS}}
\newcommand{\tss}{\textsc{TS}}

\NewDocumentCommand{\fnPr}{}{\mathbb{P}}
\RenewDocumentCommand{\Pr}{om}{\fnPr\IfValueT{#1}{_{#1}}\parentheses*{#2}}
\RenewDocumentCommand{\H}{mo}{\ensuremath{\mathrm{H}\IfValueTF{#2}{\!\left[#1\ \middle|\ #2\right]}{\brackets*{#1}}}}
\NewDocumentCommand{\Hsm}{mo}{\mathrm{H}\IfValueTF{#2}{[#1 \mid #2]}{\brackets{#1}}}

\begin{document}

\twocolumn[

\aistatstitle{LITE: Efficiently Estimating Gaussian Probability of Maximality}

\aistatsauthor{ Nicolas Menet \And Jonas Hübotter \And  Parnian Kassraie \And Andreas Krause}

\aistatsaddress{ ETH Zurich \And  ETH Zurich \And ETH Zurich \And ETH Zurich } ]
\addtocontents{toc}{\protect\setcounter{tocdepth}{0}}
\begin{abstract}
\looseness -1 We consider the problem of computing the {\em probability of maximality} (PoM) of a Gaussian random vector, i.e., the probability for each dimension to be maximal.
This is a key challenge in applications ranging from Bayesian optimization to reinforcement learning, where
the PoM not only helps with finding an optimal action, but yields a fine-grained analysis of the action domain, crucial in tasks such as drug discovery.
Existing techniques are costly, scaling polynomially in computation and memory with the vector size.
We introduce \lite{}, the first approach for estimating Gaussian PoM with {\em almost-linear time and memory} complexity.
\lite{} achieves SOTA accuracy on a number of tasks, while being in practice several orders of magnitude faster than the baselines. This also translates to a better performance on downstream tasks such as entropy estimation and optimal control of bandits. Theoretically, we cast \lite{} as entropy-regularized UCB and connect it to prior PoM estimators. %

\end{abstract}

\section{INTRODUCTION} \label{sec:intro}

Bayesian optimization~\parencite{garnett_bayesoptbook_2023} has emerged as a cornerstone for large-scale experimental design and automated discovery. Similarly, contextual bandits~\parencite{lattimore2020bandit} have been established as the leading model for personalized recommender systems~\parencite{li2010contextual} and have proven essential in the alignment of large language models~\parencite{christiano2017deep, rafailov2024direct}. Finally, reinforcement learning~\parencite{sutton2018reinforcement} has become indispensable in control systems and robotics~\parencite{kober2013reinforcement}. In spite of the vastly different application domains, these fields of study are highly related: they all adopt a Bayesian perspective on an unknown \textit{reward vector} $F$ over actions $\mathcal X$, whose posterior $p(F|\mathcal D)$ is used for informed decision-making, where $\mathcal{D}$ denotes the evidential data. %
Viewed as an interactive game between an agent and the world, these applications differ in the input context, the number of turns of the game (optimal trajectories vs.\ single-step optimal actions) and the definition of the reward. 
However, the key notion of 
\textit{probability of maximality} (PoM) naturally occurs in all these scenarios, by assisting the agent in solving the decision-making problem. 
PoM is the probability measure that \textit{Thompson sampling}~\parencite{thompson1933likelihood, russo2016information, russo2018tutorial, chapelle2011empirical} chooses actions from. Moreover, the entropy of this distribution is the objective that information-theoretic Bayesian optimization seeks to minimize~\parencite{hennig2012entropy, hernandez2014predictive, pmlr-v70-wang17e, hvarfner2022joint}. Lastly, under a suitable framing, PoM describes the data likelihood in inverse reinforcement learning~\parencite{thurstone1927method, guo2010gaussian, benavoli2021preferential}. \looseness=-1

As a concrete example, let us devise a recall-optimal bandit strategy for virtual screening in molecular design \parencite{gao2022sample, wang2023graph}. The goal of this task is to suggest a small set $E$ from a large domain of molecules $\mathcal{X}$, so that the probability of $E$ containing the optimal molecule, a.k.a.~the recall, is maximized. 
Figure~\ref{fig:quadcopter_recall} compares three solutions to this problem, and plots the recall as $|E|$ grows. Two baselines~\parencite{komiyama2015optimal} are provided by the naive methods of selecting $E$ via Thompson sampling (TS) or by choosing the top-$|E|$ molecules 
 with the largest expected rewards (MEANS). We propose to instead first estimate PoM using \lite{}, and then choose its $|E|$ largest entries.
The PoM-based method markedly outperforms the alternatives, and is in fact the provably optimal solution, under mild assumptions. \looseness=-1

\begin{figure}[ht]
    \centering
    \incplt[0.8\linewidth]{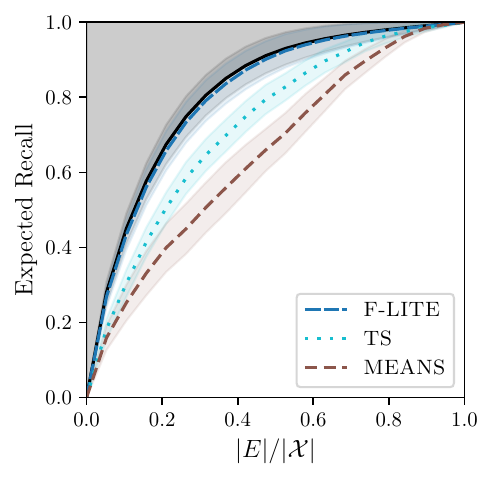}
    \vspace{-10pt}
    \caption{Selection of $E \subseteq \mathcal X$ according to the PoM estimates using \lite{} is near optimal (the gray shaded area is unachievable in expectation) and outperforms standard heuristics such as TS or selection based on the expected rewards (cf.~\cref{app:fig1_detail} for details).}
    \label{fig:quadcopter_recall}
\end{figure}

Despite the key role of Gaussian probability of maximality in Bayesian optimization~\parencite{garnett_bayesoptbook_2023}, contextual bandits~\parencite{krause2011contextual}, and reinforcement learning~\parencite{strens2000bayesian}, there has been limited investigation into its efficient estimation.
In practice, Thompson sampling is often used to calculate a Monte Carlo estimate of PoM~\parencite{hennig2012entropy}.
We refer to this technique as \ts{}\footnote{\cref{sec:TSE} presents a primer on TS and \ts{}.}, and demonstrate in Figure~\ref{fig:runtime_over_domain_size} that it becomes infeasible on sizable domains $|\mathcal X| \gg 1$, preventing large-scale real-world applications.
A handful of works, which we cover next,
provide explicit methods for estimation of PoM given a Gaussian distribution over the reward. \cref{fig:runtime_over_domain_size} compares our solution, \lite{}, with these works with respect to their computational complexity.\looseness -1

Avoiding a direct estimation of PoM, \est{}\footnote{EST is short for ``optimization as \textbf{est}imation with Gaussian processes in bandit settings''.} calculates a lower bound to Gaussian PoM \parencite{wang2016optimization} and provides a faster alternative to \ts{}.
However, as our results demonstrate, this comes at the cost of a lower accuracy (cf.~\cref{{tab:tv_distance_summary}}). 
\lite{} not only outperforms \est{}, but also computationally scales better to large domains. \looseness=-1

Our approach is most closely related to a recent result on probabilistic inference in reinforcement learning~\parencite{tarbouriech2024probabilistic} which proposed \vapor{},\!\footnote{\vapor{} is short for ``\textbf{v}ariational \textbf{a}pproximation
of the posterior \textbf{p}robability of \textbf{o}ptimality in \textbf{R}L''} a method for estimating sub-Gaussian PoM. \vapor{} numerically solves a variational objective to obtain an approximation to PoM.
In this work, we point out an interpretable near-closed-form solution to \vapor{}. Furthermore, we demonstrate that \lite{} achieves a significantly more accurate estimation of Gaussian PoM.\looseness=-1

Our work adds to the literature on Gaussian PoM estimation through the following contributions:

\begin{itemize}

    \item We introduce \lite{} (\textit{Linear-Time Independence-based Estimators}), a novel family of efficient estimators for computing Gaussian PoM with two variants: \textbf{A}-\lite{} and \textbf{F}-\lite{}, which are designed for higher \textbf{a}ccuracy or \textbf{f}aster runtime.%

    \item \lite{} scales almost-linearly in complexity as the domain size grows. This is enabled by our key idea of adopting an \ia{}, reducing the complexity by a factor of at least $|\mathcal X|$.

    \item We empirically analyze the statistical accuracy, time, and memory scaling of PoM estimation using \lite{} and existing baselines. \lite{} achieves the pareto-optimal performance for these criteria.

\end{itemize}

\begin{figure*}[ht]

    \begin{minipage}{0.5 \linewidth}
    \begin{center}
        \begin{tabular}{ll}
        \toprule
            \textbf{Method}  &\textbf{Operations} \\
            \midrule
            \ts{}% \parencite{hennig2012entropy}
            &\(\Theta(|\mathcal X|^3 + |\mathcal X|^2 %
            / \epsilon^2)\)\\
            \ias{} & \(\Theta( | \mathcal X| \sqrt{\log(1/\epsilon)} / \epsilon)\)\\
            \lite{} (ours) & $\Theta(|\mathcal X| \log(\log (| \mathcal X|)/\epsilon))$ \\
            \fvapor{} (ours) & $\Theta(|\mathcal X| \log(\log (| \mathcal X|)/\epsilon))$ \\ \\
            & \textbf{Memory} \\
            \midrule
            \ts{}% \parencite{hennig2012entropy} 
            & \(\Theta(|\mathcal X|^2)\)\\
            \ias{} & \(\Theta(|\mathcal X| + \sqrt{\log(1/\epsilon)} / \epsilon)\)\\
            \lite{} (ours) & $\Theta(|\mathcal X|)$ \\
            \fvapor{} (ours) & $\Theta(|\mathcal X|)$ \\
            \bottomrule
        \end{tabular}
    \end{center}
    \vspace{15pt}
    
    \end{minipage}\begin{minipage}{0.5\linewidth}
        \begin{center}
            \incplt[0.8\linewidth]{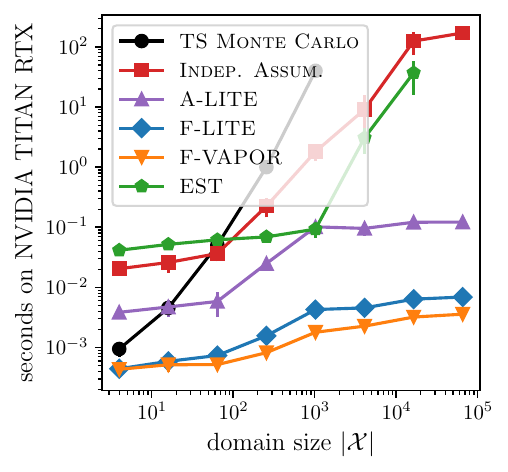}
        \end{center}
    \end{minipage}

    \vspace{-5pt}
    \caption[Description of asymptotic and empirical scaling of PoM estimators.]{Asymptotic and empirical scaling of PoM estimators.%
    ~Only \lite{} and \fvapor{} remain computationally feasible on large domains $|\mathcal X| \gg 1$ for the convergence threshold $\epsilon \in \Theta(1/|\mathcal X|)$. The minimal gap between \flite{} and \fvapor{} stems from evaluation of the slightly more expensive standard Gaussian cumulative distribution function as opposed to the exponential function. Appendix~\ref{sec:describing_fig_runtime_over_domain_size} details the experimental setup.}
    \label{fig:runtime_over_domain_size}
\end{figure*}

\section{PRELIMINARIES} \label{sec:problem_setting}

We study random reward functions over large but finite action domains~$\mathcal{X}$, concisely expressed as random vectors $F$ of length~$|\mathcal{X}|$.
These reward vectors are assumed to follow a multivariate Gaussian, i.e.,
\begin{equation*}
    F \sim \mathcal N(\mu_F, \Sigma_F)
\end{equation*} with mean $\mu_F$ and covariance matrix $\Sigma_F$.\!\footnote{As we suggest in Section~\ref{sec:future_work}, the 
Gaussian assumption may be relaxed
to all Lévy alpha-stable distributions.}
We let $F^* := \max\nolimits_x F_x$ and $X^* := \arg\max\nolimits_x F_x$ be its maximum and maximizer, respectively.
We assume the maximizer to be unique almost surely, which is satisfied automatically as long as $F$ does not contain same-mean, perfectly-correlated entries:

\begin{restatable}{ass}{uniqueMaximiser}\label{ass:unique_maximiser} $X^*$ is almost surely unique, which is equivalently expressed as $\textstyle\sum_{x \in \mathcal X} \mathbb P[x \in X^*] = 1$.
\end{restatable}\vspace{-5pt}

Under this model, we are interested in calculating the \textit{probability of maximality} (PoM), the probability of any coordinate being the maximizer: 
\begin{equation*}
    p_x \:\!:=\!\: \mathbb P[x \!\in\! X^*] \:\!=\!\: \mathbb P[F_x \!=\! F^*] \:\!=\!\: \mathbb P[F_x \!\geq\! F_z\ \forall z \!\not =\! x].
\end{equation*}

To see how $p_x$ can be used, consider the recall-optimal bandit problem, in which the goal is to find a set of $k$ arms $E \subseteq \mathcal X$ that maximizes the expected recall (true positives of maximizers). In other words, we solve
\begin{equation*}
    \argmax_{E \subseteq \mathcal X : |E|=k} \mathbb P[X^* \in E] = \argmax_{E \subseteq \mathcal X : |E|=k} \sum_{x \in E} p_x,
\end{equation*}
where equality holds under Assumption~\ref{ass:unique_maximiser}. This objective is maximized precisely by setting $E$ to the indices of the $k$ largest entries of PoM, motivating our study.

The PoM is an elusive quantity: direct numerical integration over the probability density function of $F$ must cover $|\mathcal X|$-dimensional space, and Monte Carlo integration based on $n$ i.i.d. Thompson samples (which we call \ts{}) %
converges very slowly at rate~$1/\sqrt{n}$~\parencite{morokoff1995quasi}. To make matters worse, PoM is usually rather small, scaling inversely with $|\mathcal X|$.\footnote{To see this, consider $F:\mathcal X \to \mathbb R$ as a discretization on a regular grid of a continuous Gaussian process on $[0,1]^d$. Then the existence of the PDF of $X^*$ mandates that PoM scale inversely to $|\mathcal X|$ as $|\mathcal X| \to \infty$. See also Appendix~\ref{sec:high_Thompson_sample_requirement_of_PoM}.}
Therefore, to yield useful approximations, estimators of PoM need to be ``$\epsilon$-accurate'' with $\epsilon \in \Theta(1/|\mathcal X|)$, that is, they need to run until $\epsilon$-convergence to their analytical limit. \looseness=-1

Figures~\ref{fig:runtime_over_domain_size} and~\ref{fig:accuracy_over_runtime} demonstrate that  \ts{}, the standard estimator for Gaussian PoM~\parencite{hennig2012entropy}, unfortunately does not scale to real-world domains (where $|\mathcal X|$ is often very large). Addressing these scalability issues, in this work we develop efficient estimators of Gaussian PoM that rely on the following key assumption:\looseness=-1

\begin{restatable}{ass}{Independence}\label{ass:independence} $F$ is such that PoM can be reasonably approximated assuming independent entries in $F$, i.e.,
$$p_x = \mathbb P[F_x \geq F_z\ \forall z \not = x] \approx \tilde p_x = \mathbb P[\tilde F_x \geq \tilde F_z\ \forall z \not = x]$$ where $F \sim \mathcal N(\mu_F, \Sigma_F)$ and $\tilde F \sim \mathcal N(\mu_F, \mathrm{diag}(\Sigma_F))$.
\end{restatable}\vspace{-5pt}

This mean-field approximation may hold by design, for instance in large-scale inverse reinforcement learning such as RLHF~\parencite{christiano2017deep}, or under a sufficiently coarse discretization of %
a continuous Gaussian process~\parencite{wang2016optimization, pmlr-v70-wang17e}.
As we show experimentally in Section~\ref{sec:results}, LITE effectively estimates PoM in presence of dependence structure. For further discussion on the bias introduced by Assumption~\ref{ass:independence}, we refer the reader to Appendix~\ref{sec:ndependence_assumption}.

\section{ALMOST-LINEAR TIME POM ESTIMATION WITH \lite{}} \label{sec:method}

We obtain the almost-linear-time estimator of PoM, LITE, in two steps. In the remainder of this paper we denote by $\phi$ the PDF and by~$\Phi$ the CDF of the standard Gaussian, and defer all proofs to Appendix~\ref{sec:proofs}.\looseness=-1

\paragraph{First step.}
Under the independence assumption, we consider ${\tilde F \sim \mathcal N(\mu_F, \mathrm{diag}(\sigma_{F_1}^2, \ldots, \sigma_{F_{|\mathcal X|}}^2))}$ instead of $F$, and obtain its PoM via 
\begin{equation}\label{eq:independence_assumption_simplification}
    \tilde p_x = \mathbb{P}[\tilde F_z \leq \tilde F_x\ \forall z \not= x] = \mathbb E \prod_{z \not = x} \mathbb P[ \tilde F_z \leq \tilde F_x \mid \tilde F_x] .
\end{equation}
This formulation enables us to evaluate a tractable one-dimensional integral instead of the intractable $|\mathcal X|$-dimensional integral under dependency structure.
We denote the integrand of Equation~\eqref{eq:independence_assumption_simplification} by
\begin{equation*}
    g^x(f) := \prod_{z \not = x} \mathbb P [\tilde F_z \leq f] = g(f) / \mathbb P[\tilde F_x \leq f]
\end{equation*}
with $g(f) := \textstyle\prod\nolimits_z \mathbb P[\tilde F_z \leq f]$. Through reuse of evaluations of $g(f)$, it costs as much to compute $(g^x(f_i))_{i=1}^n$ for one $x$ as it does for all $x \in \mathcal X$. A good choice of $n+1 \in \Theta(\sqrt{\log(1/\epsilon)} / \epsilon)$ shared integration points then guarantees uniformly $\epsilon$-convergent predictions:

\begin{prop_inf}[Formalized in Proposition~\ref{cor:computing_the_independence_estimator_for_Gaussian_processes_simplified}]
    Let $\epsilon \in (0,1/4]$.
    With $n+1 \in \Theta(\sqrt{\log(1/\epsilon)} / \epsilon)$ appropriately set integration points $f_0, \dots, f_n \in \mathbb{R} \cup \{\pm \infty\}$, we estimate Gaussian PoM $\forall x \in \mathcal X$ by \begin{align*}
        \tilde q_x := \sum_{i=0}^{n-1} \frac{g^x(f_{i+1}) + g^x(f_i)}{2} \mathbb P[\tilde F_x \!\in\! (f_i, f_{i+1}]].
    \end{align*}
    It then holds for all $x \in \mathcal X$ that $| \tilde p_x - \tilde q_x | \! \leq\! \epsilon.$
\end{prop_inf}

The shared integrand $g(f)$ is computed in $\Theta(|\mathcal{X}|)$ for a single integration point.
So, under the independence assumption, consistent estimation of PoM can be performed in just $\Theta(|\mathcal X| \sqrt{\log(1/\epsilon)}/\epsilon)$, our first significant runtime improvement over \ts{}.

\paragraph{Second step.}
To remove the linear scaling in $1/\epsilon$ that stems from numerical integration, we propose to approximate $g^x(f)$ %
with the CDF of a Gaussian: %
\begin{equation*}
    g^x(f) = \prod_{z \not = x} \Phi\!\left(\frac{f - \mu_{F_z}}{\sigma_{F_z}}\right) \approx \Phi\!\left(\frac{f-m_x}{s_x}\right).
\end{equation*}
Under this variational approximation, we can solve the integral of Equation~\eqref{eq:independence_assumption_simplification} in closed-form:
\begin{equation}\label{eq:lite_closed_form_PoM_prediction}
	\tilde p_x = \mathbb E[g^x(\tilde F_x)] \approx \mathbb E \Phi\!\left(\frac{\tilde F_x -m_x}{s_x}\right) = \Phi\Bigg(\frac{\mu_{F_x} - m_x}{\sqrt{\sigma_{F_x}^2 + s_x^2}}\Bigg).
\end{equation}
Both variants of \lite{} rely on Equation~\eqref{eq:lite_closed_form_PoM_prediction}, but differ in how they approximate $g^x$, i.e., in how they determine the free variables $m_x$ and $s_x$:
\begin{itemize}
\item \alite{} uses nested binary search to match the quartiles of $\Phi((\, \cdot\, -m_x) / s_x)$ to those of~$g^x$.
\item \flite{} sets $s_x = 0$ and leverages Assumption~\ref{ass:unique_maximiser} to find a shared normalizing threshold $m_x = \kappa^*$.
\end{itemize}

In the following, we focus our exposition on the ``fast'' (and simpler) variant \flite{}, even though we find in our experiments that the ``accurate'' variant \alite{} tends to be the more faithful estimator.
We include a detailed discussion of \alite{} in Appendix~\ref{sec:alite}.\looseness=-1 %

\subsection{Fast \lite{}}
\flite{} approximates the Gaussian PoM in Equation~\eqref{eq:lite_closed_form_PoM_prediction} with $s_x = 0$, which is suggested by concentration of measure of the maximum,\!\footnote{Proposition~\ref{cor:maximum_of_gaussians_converges_in_probability} in Appendix~\ref{sec:proofs} shows that the distribution of the maximum concentrates as $|\mathcal X| \to \infty$.} and leverages
Assumption~\ref{ass:unique_maximiser} to find a shared normalizing threshold~$\kappa^*$:
\begin{equation*}
    \tilde p_x \approx q_x := \Phi\!\left(\frac{\mu_{F_x} - \kappa^*}{\sigma_{F_x}}\right) \text{ with $\kappa^*$ s.t. } \sum_x q_x = 1.
\end{equation*}
Here, $\kappa^*$ can be found efficiently using binary search. %
We summarize \flite{} in Algorithm~\ref{alg:CME}. The boundaries of the binary search window and the implied complexity is derived in the following proposition:\looseness=-1

\begin{prop_inf}[Formalized in Proposition~\ref{cor:logarithmic_search_to_obtain_gaussian_probability_of_optimality}]
    Observe that $\textstyle\sum\nolimits_{x \in \mathcal X} \Phi\!\left((\mu_{F_x} - \kappa) / \sigma_{F_x}\right)$ is continuous and monotonically decreasing in $\kappa$.
    We determine bounds $\kappa_{low}, \kappa_{up}$ on $\kappa^*$ such that $\kappa_{up} - \kappa_{low} \in \Theta(\sqrt{\log |\mathcal X|})$.
    Therefore, with $\kappa^k$ the $k$-th iterate of binary search and $k \in \Theta(\log(\log(|\mathcal X|) / \epsilon))$ it holds for all $x \in \mathcal{X}$ that $|\mathbb P[F_x \geq \kappa^*] - \mathbb P[F_x \geq \kappa^k]| \leq \epsilon$.
\end{prop_inf}

Each iteration of binary search requires summing the entries $q_x$, and therefore the compute cost of \flite{} is almost-linear at $\Theta(|\mathcal X| \log(\log(|\mathcal X|) / \epsilon))$ operations.
This provides us with an efficient PoM estimator that can be applied to real-world tasks with large domains.\looseness=-1

\begin{algorithm}[t!]
\caption{\flite{}}\label{alg:CME}
\begin{algorithmic}
\Require $\mu_{F}, \sigma_{F}, \epsilon$
\State $\kappa_{low} \gets \mu_F^{min} + \sigma_F^{min} \cdot - \Phi^{-1}(1/|\mathcal X|)$
\State $\kappa_{up} \gets \mu_F^{max} + \sigma_F^{max} \cdot - \Phi^{-1}(1/|\mathcal X|)$%
\State $\text{max-error} \gets \epsilon$
\While{$\text{max-error} \geq \epsilon$}
\State $\kappa \gets \tfrac{1}{2}\kappa_{up} + \tfrac{1}{2}\kappa_{low}$
\State $s \gets \sum\nolimits_{x \in \mathcal X} \Phi(\frac{\mu_{F_x} - \kappa}{\sigma_{F_x}})$ %
\State \textbf{if} $s > 1$  \textbf{then} $\kappa_{low} \gets \kappa$ \textbf{else} $\kappa_{up} \gets \kappa$
\State $\text{max-error} \gets \max_{x \in \mathcal X} \Phi(\frac{\mu_{F_x} \text{-} \kappa_{low}}{\sigma_{F_x}}) - \Phi(\frac{\mu_{F_x} \text{-} \kappa_{up}}{\sigma_{F_x}})$%
\EndWhile
\State $q_x \gets \tfrac{1}{2}\Phi(\frac{\mu_{F_x} - \kappa_{low}}{\sigma_{F_x}}) + \tfrac{1}{2}\Phi(\frac{\mu_{F_x} - \kappa_{up}}{\sigma_{F_x}})\ \ \forall x \in \mathcal X$%
\\
\Return $(q_x / \sum\nolimits_{z \in \mathcal X} q_z)_{x \in \mathcal X}$ %
\end{algorithmic}
\end{algorithm}

\subsection{Properties of \flite{}}\label{sec:differentiating_through_flite}
%Before evaluating the fidelity of PoM estimation with \flite{}, we highlight some of its properties.

\paragraph{Differentiability.}
\flite{} admits a closed-form expression for the derivatives of the estimated PoMs w.r.t. the parameters $\mu_F$ and $\sigma_F$ of the Gaussian reward vector.
Such derivatives are essential for the use of PoM estimates as data likelihoods in machine learning.
For example, the likelihood of $k$-option preference feedback (a case of inverse RL) is measured by PoM~\parencite{christiano2017deep, bradley1952rank, thurstone1927method}, and derivatives are key to end-to-end learning of such preferences.

\begin{restatable}%
[]{prop}{gradientsOfCME}\label{thm:gradients_of_cme}
    Let $h_x := \phi\!\left(\frac{\mu_{F_x} - \kappa^*}{\sigma_{F_x}}\right) \frac{1}{\sigma_{F_x}}$. Then %
    \vspace{-6pt}\begin{align}\label{eq:flite_derivatives_mu}
        \frac{d q_x}{d \mu_{F_z}} & = h_x \cdot \left(\mathds{1}_{x=z}-\frac{h_z}{\sum_{w \in \mathcal X} h_w}\right)\\ \label{eq:flite_derivatives_sigma}
        \frac{d q_x}{d \sigma_{F_z}} & = h_x \cdot\left(\mathds{1}_{x=z} -  \frac{h_z}{\sum_{w \in \mathcal X} h_w}\right) \cdot \frac{\kappa^* - \mu_{F_z}}{\sigma_{F_z}}.
    \end{align}
\end{restatable}\vspace{-5pt}
Here, $h_x$ is a sensitivity factor.
Equations~\eqref{eq:flite_derivatives_mu} and~\eqref{eq:flite_derivatives_sigma} are remarkably interpretable: increasing $\mu_{F_z}$ renders $z$ a more likely and $x \not = z$ a less likely maximizer. Moreover, increasing $\sigma_{F_z}$ renders $z$ a more likely and $x \not = z$ a less likely maximizer if $q_z < 0.5$ (here uncertainty helps), otherwise $z$ becomes a less likely and $x \not = z$ a more likely maximizer.

\paragraph{Balancing two sources of exploration.}
Efficient exploration is a key challenge in many domains of machine learning, including Bayesian optimization and reinforcement learning.
The necessity for exploration in optimization arises when we are uncertain about the rewards of actions.
In estimation of PoM, we face the same challenge: a faithful estimate of PoM needs to account for what we do not know, and assign a larger PoM to points with low mean and large variance than to points with low mean and low variance.
Remarkably, we show in the following that \flite{} can be seen as a combination of two common exploration-inducing approaches: optimism in the form of an upper-confidence bound~\parencite{garnett_bayesoptbook_2023, jones2001taxonomy, srinivas2009gaussian, vanchinathan2015discovering,chen2017interactive}, short UCB, and entropy regularization~\parencite{ziebart2010modeling, neu2017unified, geist2019theory, mnih2016asynchronous, haarnoja2018soft}.%

\begin{restatable}%
[]{prop}{varPrincipleForCme}\label{thm:variational_principle_encoding_concentrated_maximum_estimator}
    Define the variational objective
    \begin{equation}\label{eq:variational_objective_with_CME_as_solution}
        \mathcal W(p) := \sum_{x \in \mathcal X} p_x \cdot \Big(\mu_{F_x} + \underbrace{\sqrt{2 \tilde I(p_x)} \cdot \sigma_{F_x}}_{\text{exploration bonus}} \Big),
    \end{equation} with the quasi-surprisal $\tilde I(u) := (\phi(\Phi^{-1}(u))/u)^2 / 2$.
    Then the maximizer of $\mathcal{W}$ among elements of the probability simplex is given by \flite{}, i.e., by $q$ with
    \begin{equation*}
        q_x := \Phi\!\left(\frac{\mu_{F_x} - \kappa^*}{\sigma_{F_x}}\right) \text{ with $\kappa^*$ s.t. } \sum_x q_x = 1.
    \end{equation*}
\end{restatable}\vspace{-5pt}
The quasi-surprisal $\tilde I(\cdot)$ behaves similarly to the surprisal $-\ln(\cdot)$, a key quantity in information theory~\parencite{cover1999elements}.
In fact, their asymptotics coincide: \begin{align*}
    \tilde I(1) = 0 = - \ln(1) \text{ and } \tilde I(u) \sim - \ln u \text{ as } u \to 0^+.
\end{align*}
The objective from Equation~\eqref{eq:variational_objective_with_CME_as_solution} is maximized for those probability distributions $p$ that are concentrated around points with large mean $\mu_{F_x}$ and points with large exploration bonus.
The uncertainty $\sigma_{F_x}$ about $F_x$ is the standard exploration bonus of UCB algorithms.
In Equation~\eqref{eq:variational_objective_with_CME_as_solution}, $\sigma_{F_x}$ is weighted by the quasi-surprisal, which acts as entropy regularization: it increases the entropy of $p$ by uniformly pushing $p_x$ away from zero.
The variational objective suggests that Thompson sampling~\parencite{thompson1933likelihood, russo2016information, russo2018tutorial, chapelle2011empirical}, i.e., sampling from PoM, achieves exploration through two means:\looseness=-1 \begin{enumerate}
    \item \textbf{Optimism:} by preferring points with large uncertainty $\sigma_{F_x}$ about the reward value $F_x$.
    \item \textbf{Decision uncertainty:} by assigning some probability mass to all $x$, that is, by remaining uncertain about which $x$ is the maximizer.
\end{enumerate}
Interestingly, the recall task from Figure~\ref{fig:quadcopter_recall} is solved by choosing actions with highest PoM.
Contrary to initial intuition, the good performance of \lite{} in the recall task indicates that optimism and decision uncertainty, normally associated with exploration, are also useful for pure exploitation.

\section{LANDSCAPE OF POM ESTIMATION} \label{sec:connections}
Motivated by the intimate relation between PoM estimation in the form of \flite{} and decision-making, we next connect PoM to several methods developed for Bayesian optimization and reinforcement learning.

\paragraph{Probability of improvement.}
\label{sec:connection_to_PI}
\flite{} measures the probability of improvement over the normalizing threshold $\kappa^*$: ${q_x := \Phi((\mu_{F_x} - \kappa^*)/\sigma_{F_x}) = \mathbb P[F_x \geq \kappa^*]}$.
Similarly, the true PoM can be seen as measuring a probability of improvement: $p_x = \mathbb P[F_x \geq F^*]$.
By comparing the two expressions, the normalizing threshold in \flite{} can be understood as a deterministic surrogate for the maximum.
Probability of improvement is widely known as an acquisition function in Bayesian optimization~\parencite{kushner1964new, garnett_bayesoptbook_2023, jones2001taxonomy, vzilinskas1992review}, with the threshold $\kappa^*$ typically set to the best observation.
\looseness=-1

\begin{table*}[t]
    \centering
    \begin{tabular}{lccccc}
    \toprule
        & Synthetic Distributions & 1-dim GP & 2-dim GP (\ref{sec:sampling_from_2d_laplacian_GP}) & DropWave (\ref{sec:drop-wave_during_sampling}) & Quadcopter\\
        \midrule
        \est{}%~\parencite{wang2016optimization}
        & $11.54 \pm 0.20$ & $45.6 \pm 2.7$ & $15.1 \pm 1.2$ & $5.17 \pm 0.64$&  $14.3 \pm 2.0$\\

        \vapor{}%~\parencite{tarbouriech2024probabilistic}
        & $\phantom{0}9.89 \pm 0.11$ & $37.0 \pm 2.0$ & $15.7 \pm 1.0$ & $5.70 \pm 0.72$ &  $17.2 \pm 2.5$\\

        \flite{} (ours) & $\phantom{0}4.65 \pm 0.08 $ & $\mathbf{13.7 \pm 1.0}$ & $ 10.9 \pm 0.7$ & $\mathbf{4.87 \pm 0.60}$ &  $11.1 \pm 1.4$\\

        \alite{} (ours) &  $\mathbf{\phantom{0}3.76 \pm 0.06} $ & $\mathbf{14.1 \pm 1.0} $ & $\mathbf{\phantom{0}7.5 \pm 0.5}$ & $\mathbf{4.32 \pm 0.53}$ &  $\mathbf{\phantom{0}8.7 \pm 0.9}$\\

        \midrule

        \ias{} & $\phantom{0}0.00 \pm 0.00$ & $\phantom{0}6.7 \pm 0.4$ & $\phantom{0}6.6 \pm 0.2$ & $3.85 \pm 0.54$ & $\phantom{0}9.0 \pm 1.0$\\
    \bottomrule
    \end{tabular}
    \caption{Mean and standard error of TV distance (averaged across $|\mathcal X|$ and BO-steps) in percentage \%. \alite{} and \flite{} consistently outperform competing efficient PoM estimators from the literature. The \ia{} is provided as an expensive baseline, since all considered efficient estimators build on it.}
    \label{tab:tv_distance_summary}
\end{table*}

\paragraph{Estimating the maximum reward value.}\label{sec:connection_to_EST}
The EST(-imate) algorithm~\parencite{wang2016optimization} proposes to approximate Gaussian PoM with its lower bound
\begin{equation*}
    \tilde p_x \approx \frac{\mathbb P[\tilde F_x \geq \tilde \kappa]}{1-\mathbb P[\tilde F_x \geq \tilde \kappa]} \prod\nolimits_{z\in\mathcal X}\mathbb P[\tilde F_z \leq \tilde \kappa],
\end{equation*}
where $\tilde \kappa = \mathbb E[\tilde F^*]$ with $\tilde F \sim \mathcal N(\mu_F, \mathrm{diag}(\Sigma_F))$. It then directly uses this lower bound as an acquisition function for Bayesian optimization. With the denominator being usually close to $1$, \est{} corresponds to a globally rescaled \flite{}, but using the expectation of $\tilde F^*$ %
instead of the normalizing threshold~$\kappa^*$ as a surrogate for the maximum. In our experiments, we linearly normalize the PoM predicted by \est{} to sum to $1$, creating a stronger baseline for us to compare against.%
\looseness=-1

\paragraph{UCB + entropy regularization.}\label{sec:connection_to_vapor}
In analogy to our variational formulation of \flite{}, \vapor{}~\parencite{tarbouriech2024probabilistic} proposes to maximize the variational objective
\begin{equation}\label{eq:vapor_variational_objective}
    \mathcal V(p) = \sum\nolimits_{x \in \mathcal X} p_x \cdot \left( \mu_{F_x} + \sqrt{2 \ln (1/p_x}) \cdot \sigma_{F_x} \right)
\end{equation}
on the probability simplex to estimate PoM.
To solve Equation~\eqref{eq:vapor_variational_objective}, they use Frank-Wolfe~\parencite{jaggi2013revisiting, lacoste2015global} with ${k \in \Theta(\epsilon^{-5} |\mathcal X|^4)}$ steps to ensure $\mathcal V(p^*) - \mathcal V(p) \leq \epsilon$ with no bounds on $\|p^*-p\|_\infty$~\parencite{bolte2023iterates}. Instead, we derive a previously unknown near-closed-form solution to \vapor{} whose iterates converge exponentially at a linear rate:\looseness=-1

\begin{restatable}%
[Fast \vapor{}]{prop}{vaporClosedFormSolution}\label{thm:closed_form_solution_vapor}
    The maximizer to Equation~\eqref{eq:vapor_variational_objective} on the probability simplex admits the closed-form expression
    \begin{equation*}
        v_x := v\!\left(\frac{\mu_{F_x} - \nu^*}{\sigma_{F_x}}\right) \text{ with $\nu^*$ s.t. } \sum_x v_x = 1,
    \end{equation*}
    where $v(c) := \exp(-(\sqrt{c^2 + 4} - c)^2 / 8)$.
    
    Moreover, to find $\nu^*$ we can use binary search with $k \in \Theta(\log(\sqrt{\log |\mathcal X|}/\epsilon))$ iterations, ensuring that the $k$-th iterate $v^k$ satisfies $\|v^* - v^k\|_\infty < \epsilon$.
\end{restatable}

Note the similarity to \flite{}: we have only replaced $\Phi$ by the sigmoidal $v$. As such, Algorithm~\ref{alg:CME} is easily adapted to obtain a novel almost-linear-time implementation of \vapor{}, which we call \fvapor{}.\looseness=-1

\section{EXPERIMENTS}\label{sec:results}

\begin{figure*}[t]
    \centering
    \incplt[0.8\linewidth]{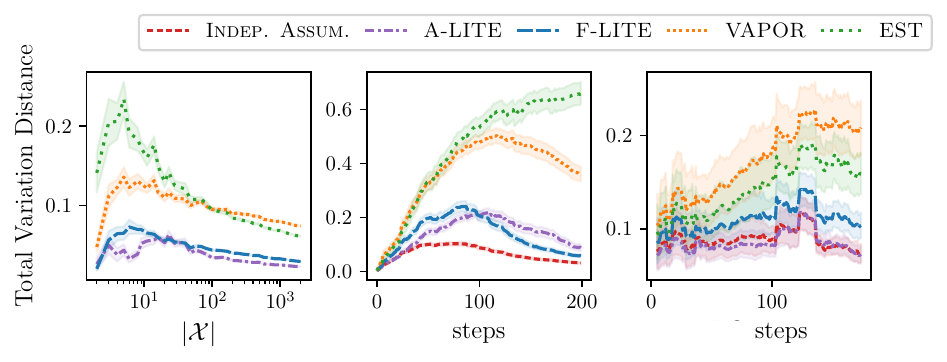}
    \vspace{-10pt}
    \begin{subfigure}{0.23\linewidth}
        \subcaption{Synthetic posteriors}
        \label{fig:estimators_tv_distance_to_independent_ground_truth}
    \end{subfigure}
    \begin{subfigure}{0.23\linewidth}
        \subcaption{$f_{true} \sim \mathcal{GP}$ [1-dim]}
        \label{fig:tv_distance_for_BO_of_GP_sample}
    \end{subfigure}
    \begin{subfigure}{0.23\linewidth}
        \subcaption{$f_{true}$ from Quadcopter}
        \label{fig:tv_distance_quadcopter}
    \end{subfigure}
    \vspace{-5pt}
    \caption{\lite{} universally outperforms \vapor{} and \est{} in terms of TV-distance to the ground-truth, which is estimated using the \ia{} (Figure~\ref{fig:estimators_tv_distance_to_independent_ground_truth}) and \ts{} (Figures~\ref{fig:tv_distance_for_BO_of_GP_sample},~\ref{fig:tv_distance_quadcopter}).} %
\end{figure*}

Next, we compare the PoMs estimated by \alite{} and \flite{} against the efficient baselines \est{} and \vapor{}.
We measure the total variation distance to the ``ground truth'' PoM obtained via expensive \ts{} as well as the root mean squared relative error on the down-stream task of entropy estimation. The \ia{} is computing an asymptotically exact estimate under Assumption~\ref{ass:independence}, which we report as a (up to significance) error lower bound for independence-based PoM estimators. The code is available at \url{https://github.com/lasgroup/LITE}.

\subsection{PoM Estimation}\label{sec:fidelity_of_PoM_estimation}
To compare PoM estimators in various settings (for various $(\mu_F, \Sigma_F)$), we rely on synthetic distributions as well as posteriors produced during Bayesian optimization. Table~\ref{tab:tv_distance_summary} provides a summary of our results.

\paragraph{Synthetic distributions.}
We obtain a set of synthetic $(\mu_F, \sigma_F)$ by independently sampling ${\mu_{F_x} \sim \mathcal U([0, 5])}$ and $\sigma_{F_x} \sim \mathcal U([1/2, 10])$ for all $x$. We employ Proposition~\ref{cor:computing_the_independence_estimator_for_Gaussian_processes_simplified} for the ground-truth PoM, i.e., estimation under the \ia{}. Figure~\ref{fig:estimators_tv_distance_to_independent_ground_truth} shows how \alite{} and \flite{} significantly outperform \vapor{} and \est{}. We remark that estimation of PoM seems to become easier on large domains. We suspect that more repetition in $\mu_{F}$ and $\sigma_F$ leads to a more uniform PoM that is easier to estimate. Similar results on alternative distributions over $\mu_F, \sigma_F$ are provided in Appendix~\ref{sec:fidelity_under_synthetic_posteriors_alternative_experiments}.%\looseness=-1

\paragraph{Samples from a Gaussian process.}\label{sec:main_text_fidelity_during_calibrated_BO}
Figure~\ref{fig:tv_distance_for_BO_of_GP_sample} shows the total variation distance between a ground-truth estimate using \ts{} and the PoM of the various estimators. The posteriors are derived from calibrated Bayesian optimization with $f_{true}$ sampled from a squared exponential prior on a one-dimensional domain. \alite{} and \flite{} outperform \vapor{} and \est{} by a large margin. \flite{} becomes most accurate at late stages of optimization, once $F^*$ becomes quite concentrated. The details of the experimental setup are in Appendix~\ref{sec:describing_fig_tv_distance_for_BO_of_GP_sample}. %

\paragraph{DropWave function.}
In practice, Bayesian optimization is run on a single test function and calibrated through marginal likelihood maximization of the prior parameters. Figure~\ref{fig:accuracy_over_runtime} demonstrates the accuracy/runtime operating points according to the various considered PoM estimators under different choices of the convergence parameter $\epsilon = 1/(\alpha \cdot |\mathcal X|)$. Here, $f_{true}$ is set to the drop-wave function, notorious for its difficulty in Bayesian optimization, quantized to $625$ points. Given sufficient compute, consistent estimation through \ts{} is recommended. However, as shown in Figure~\ref{fig:runtime_over_domain_size}, \ts{} scales worse than the \ia{} (and \lite{}) to large domains. Consequently, as the domain size $|\mathcal X|$ increases, the point at which \ts{} starts to outperform them is shifted to the right into a computationally infeasible region. Experimental details can be found in Appendix~\ref{sec:drop-wave_during_accuracy_runtime_details}. For additional experiments on drop-wave (which feature in Table~\ref{tab:tv_distance_summary}), see Appendix~\ref{sec:drop-wave_during_sampling}.
\begin{figure}[h]
    \incplt[0.9\linewidth]{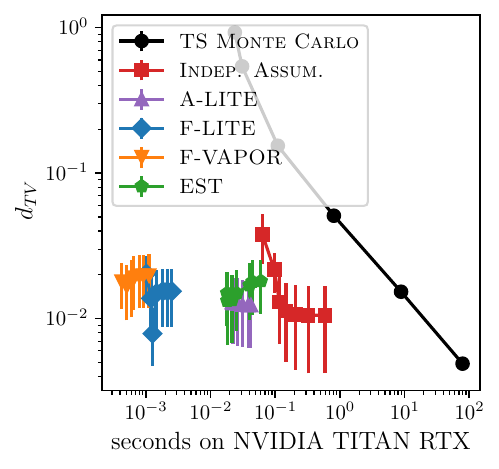}

    \vspace{-15pt}
    \caption{\ts{} is consistent, but only becomes competitive with high computational cost. Instead, the biased but efficient PoM estimators quickly converge to competitive accuracy at low computational cost. Here, $F$ is distributed according to the posteriors of Bayesian optimization with $f_{true}$ set to drop-wave.}
    \label{fig:accuracy_over_runtime}
    \vspace{-10pt}
\end{figure}

\paragraph{Quadcopter simulation.}\label{sec:main_text_quadcopter}
Finally, we consider the TV distance during Bayesian optimization of the parameters of a quadcopter controller~\parencite{hübotter2024transductiveactivelearningtheory}, see Figure~\ref{fig:tv_distance_quadcopter}. The ground-truth PoM is estimated using \ts{}. $f_{true}$, a function of the parameters, describes the degree to which a controller manages to stabilize a quadcopter in a simulated environment under randomly sampled perturbations. The controller presents eight degrees of freedom, 4 of which are solved using a heuristic, resulting in Bayesian optimization in four-dimensional space. To ensure tractable computation of a ground-truth PoM, we uniformly at random subsample the domain to $400$ discrete points. Details are in Appendix~\ref{sec:quadcopter_details}. As in the other experiments, estimation under the \ia{} is most accurate, swiftly followed by \alite{} and \flite{}. \vapor{} and \est{} are less performant in comparison.%\looseness=-1

\begin{table*}[t]
    \centering
    \begin{tabular}{lcccc}
    \toprule
        & 1-dim GP & 2-dim GP (\ref{sec:sampling_from_2d_laplacian_GP}) & DropWave (\ref{sec:drop-wave_during_sampling}) & Quadcopter\\
        \midrule
        \est{}%~\parencite{wang2016optimization} 
        & $215.5\ (195.4, 233.9)$ & $36.3\ (27.5, 43.3)$ & $\mathbf{\ 5.4\ (4.9, 5.8)}$ & $\mathbf{\ 3.4\ (2.9, 3.9)}$\\

        \vapor{}%~\parencite{tarbouriech2024probabilistic} 
        & $169.4\ (158.4, 179.7)$ & $33.7\ (26.9, 39.4)$ & $\ \, 8.2\ (5.9, 10.0)$ & $3.8\ (3.4, 4.2)$\\

        \flite{} (ours) & $\mathbf{35.9\ (34.7,  37.0)}\, $ & $\ \mathbf{12.4\ (10.8, 13.9)}$ & $\mathbf{\ 5.3\ (4.8, 5.7)}$ & $\mathbf{\ 2.6\ (1.9, 3.2)}$\\

        \alite{} (ours) &  $44.9\ (42.8, 47.0)\ \, $ & $\mathbf{11.0\ (9.5, 12.4)}\, $ & $\mathbf{\ 4.7\ (4.2, 5.1)}$ & $\mathbf{\ 3.0\ (2.4, 3.5)}$\\

        \midrule

        \ias{} & $18.4\ (17.8, 18.9)\ \,$ & $4.9\ (4.4, 5.4)\ \,$ & $4.6\ (4.2, 5.1)$ & $2.4\ (1.6, 3.0)$ \\
    \bottomrule
    \end{tabular}%
    \caption{Empirical root mean squared relative error of entropy in percentage \% (along with confidence bands). \alite{} and \flite{} consistently outperform competing efficient estimators of PoM. The confidence bands correspond to the square root of mean $\pm$ standard error of the squared relative error (averaged across BO-steps).}
    \label{tab:entropy_rmsre_summary}
\end{table*}

\subsection{PoM Entropy Estimation}\label{sec:entropy_estimation}
Information theory~\parencite{gray2011entropy} proposes to measure uncertainty with the Shannon entropy $H[X^*] := \sum\nolimits_x p_x \ln(1/p_x)$. Unfortunately, there is no known unbiased Monte Carlo estimator of entropy without access to $p_x$~\parencite{paninski2003estimation}. Furthermore, the standard procedure of using \ts{} provably underestimates the entropy unless many samples are used: let $q_x :=\sum\nolimits_{i=1}^n \mathds{1}_{x \in \arg\max f_i} / n$ for i.i.d. $f_i \sim p(f|\mathcal D)$. Then either $q_x \geq 1/n$ or $q_x = 0$, and hence it holds that

\begin{equation*}
    H[q_x] = \sum\nolimits_x q_x \ln(1/q_x) \leq \ln(n).
\end{equation*}

Only if $n$ exceeds $|\mathcal X|$ can entropy estimation using \ts{} span the full range of valid values $[0, \ln(|\mathcal X|)]$. As such, a runtime that scales in $\Theta(|\mathcal X|^3)$ would be required, which becomes prohibitive for large domains. In contrast, the exponential convergence of \lite{} allows efficient entropy estimation in $\Theta(|\mathcal X| \log(|\mathcal X|))$.

In our experiments, we report on the root mean squared relative error of PoM entropy estimation across multiple seeds of optimization, defined as
\begin{equation*}
    \sqrt{\frac{1}{m} \sum_{i=1}^m \left(\frac{H[E \mid \mathcal D^i] - H[X^* \mid \mathcal D^i]}{H[X^* \mid \mathcal D^i]}\right)^2}.
\end{equation*}
The ground-truth $H[X^* \mid \mathcal D^i]$ is estimated based on expensive \ts{}, whereas $H[E \mid \mathcal D^i]$ denotes the entropy estimation according to different PoM estimators. The relative error is a natural performance criterion, ensuring normalization across different stages of optimization and across various ground-truths $f_{true}$.

As Table~\ref{tab:entropy_rmsre_summary} %
demonstrates, the entropy of $X^*$ can be faithfully estimated based on the \ia{}.
Whereas the two variants of \lite{} remain competitive with the \ia{}, \vapor{} and \est{} are often much worse in their estimation of entropy. Here, the experimental setups correspond to Section~\ref{sec:fidelity_of_PoM_estimation}. In particular, the 1-dim GP experiment is described in Appendix~\ref{sec:describing_fig_tv_distance_for_BO_of_GP_sample}, the 2-dim GP experiment in Appendix~\ref{sec:sampling_from_2d_laplacian_GP}, DropWave in Appendix~\ref{sec:drop-wave_during_sampling}, and Quadrotor in Appendix~\ref{sec:quadcopter_details}.

\begin{figure}[ht]
    \centering
    \incplt[0.8\linewidth]{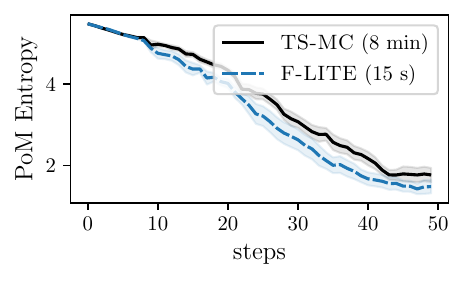}
    \vspace{-10pt}
    \caption{
    Entropy Search with \lite{} admits better computational and statistical efficiency than with \ts{}. We describe the setup in Appendix~\ref{sec:describing_fig_entropy_search}.}
    \label{fig:entropy_search}
\end{figure}

\subsection{Applications of PoM Entropy Estimation}

Entropy Search~\parencite{hennig2012entropy} is a widely used strategy in Bayesian optimization, which queries the reward at the point $x \in \mathcal X$ promising (in expectation) the largest reduction in entropy of Gaussian PoM. In Figure~\ref{fig:entropy_search}, we run calibrated Entropy Search on a $1$-dimensional Gaussian process with squared exponential kernel, discretized to a domain of size ${|\mathcal X| = 250}$. Simply replacing the standard PoM estimator (\ts{}) with \lite{} results in significantly shorter runtimes and better optimization trajectories, already for moderately large domains $\mathcal X$. This indicates that \lite{} can be used to markedly improve the scalability of Entropy Search.\looseness=-1

Finally, through its almost-linear time and memory complexity, the estimation of PoM entropy with \lite{} can be used to better understand the state of Bayesian optimization in large-scale settings where previous approaches for PoM entropy estimation would become intractable. %
To capture such a large-scale setting, we consider an objective $f_{true}$ set to a hyperplane in $1`000$ dimensions sampled to a finite domain with $|\mathcal X| = 10`000$ points.
On an NVIDIA A100 GPU, compared to the \ia{} and thus also \ts{}, \lite{} \emph{reduces computation time from $21$ days to $30$ seconds}.
We describe details in Appendix~\ref{sec:describing_fig_large_scale_exp}.\looseness=-1

\section{FUTURE WORK}\label{sec:future_work}
\paragraph{Generalization of \lite{}.} 
The developed methodology can be extended to distributions other than Gaussians. In fact, the \ia{} has a generalization to arbitrary distributions in the form of Proposition~\ref{thm:computing_the_independence_estimator} in Appendix~\ref{sec:proofs}. Moreover, the variational approximation of \lite{}, which allows analytical integration, can be extended to any Lévy alpha stable distribution: let $\mathbb P[F_x \leq f] = G((f-\mu_{F_x})/\sigma_{F_x})$ for a stable $G$, then approximating $g^x(f)$ with $G((f-m_x)/s_x)$ results in an analytical expression for PoM. Together, this indicates that \lite{} can be generalized to a much larger class of distributions than just Gaussians. In this work, we emphasize Gaussians due to their ubiquity across many applications domains and leave a more general analysis to future work.

\paragraph{Learning heteroscedastic reward models.}\hspace{-2pt}Given a random reward vector over actions, the data likelihood of (reward-maximizing) experts picking any one is precisely equal to PoM~\parencite{luce2005individual, christiano2017deep, bradley1952rank, thurstone1927method}. Through its closed-form derivatives stated in Proposition~\ref{sec:differentiating_through_flite}, \lite{} could allow efficient end-to-end learning of a (parametrized) reward model that simultaneously indicates the expected reward of an action, as well as its associated uncertainty. In contrast to Bradley-Terry and Thurstone, \lite{} naturally extends to heteroscedastic rewards, which could prove essential for faithful modeling of epistemic uncertainty.

\section{CONCLUSION}
In \lite{}, we developed estimators of Gaussian probability of maximality (PoM) that operate in near-linear efficiency with respect to the size of the Gaussian vector considered. In contrast, previous methods scale polynomially and thus quickly become computationally infeasible for moderately-sized vectors. Our empirical observations in multiple settings demonstrate that \lite{}, in comparison to \est{} and \vapor{}, delivers more accurate PoM estimates and results in better PoM entropy estimation. Theoretically, we revealed connections between \flite{} and the Bayesian optimization literature, spanning PI, \est{}, entropy-regularized UCB, and VAPOR. Based on a variational formulation of \lite{}, we uncovered how Thompson sampling achieves exploration by relying simultaneously on optimism and decision uncertainty, and how these two principles, unexpectedly, guide optimal behavior in a pure exploitation task.

Finally, we demonstrated that the achieved efficiency gains translate to better performance at downstream objectives such as recall-optimal control of bandits and Entropy Search. We envision that the scalability improvements achieved in this work inspire further development of algorithms that leverage the now-tractable notion of Gaussian PoM to tackle challenges in domains such as high-dimensional Bayesian optimization and reinforcement learning.

\section*{Acknowledgements}

We thank the anonymous reviewers for their valuable feedback on the paper.
This project was supported in part by the European Research Council~(ERC) under the European Union's Horizon 2020 research and Innovation Program Grant agreement no.~815943, and the Swiss National Science Foundation under NCCR Automation, grant agreement~51NF40~180545. Nicolas Menet was supported by the ETH Excellence Scholarship \& Opportunity Programme and Parnian Kassraie was supported by a Google Ph.D. Fellowship.

%\printbibliography

\bibliography{references}

%%%%%%%%%%%%%%%%%%%%%%%%%%%%%%%%%%%%%%%%%%%%%%%%%%%%%%%%%%%%
\newpage\section*{Checklist}

% % %%% BEGIN INSTRUCTIONS %%%
% The checklist follows the references. For each question, choose your answer from the three possible options: Yes, No, Not Applicable.  You are encouraged to include a justification to your answer, either by referencing the appropriate section of your paper or providing a brief inline description (1-2 sentences). 
% Please do not modify the questions.  Note that the Checklist section does not count towards the page limit. Not including the checklist in the first submission won't result in desk rejection, although in such case we will ask you to upload it during the author response period and include it in camera ready (if accepted).

% \textbf{In your paper, please delete this instructions block and only keep the Checklist section heading above along with the questions/answers below.}
% % %%% END INSTRUCTIONS %%%

\begin{enumerate}

\item For all models and algorithms presented, check if you include:
\begin{enumerate}
  \item A clear description of the mathematical setting, assumptions, algorithm, and/or model. [Yes]
  \item An analysis of the properties and complexity (time, space, sample size) of any algorithm. [Yes]
  \item (Optional) Anonymized source code, with specification of all dependencies, including external libraries. [Yes]
\end{enumerate}

\item For any theoretical claim, check if you include:
\begin{enumerate}
  \item Statements of the full set of assumptions of all theoretical results. [Yes]
  \item Complete proofs of all theoretical results. [Yes]
  \item Clear explanations of any assumptions. [Yes]     
\end{enumerate}

\item For all figures and tables that present empirical results, check if you include:
\begin{enumerate}
  \item The code, data, and instructions needed to reproduce the main experimental results (either in the supplemental material or as a URL). [Yes]
  \item All the training details (e.g., data splits, hyperparameters, how they were chosen). [Yes]
        \item A clear definition of the specific measure or statistics and error bars (e.g., with respect to the random seed after running experiments multiple times). [Yes]
        \item A description of the computing infrastructure used. (e.g., type of GPUs, internal cluster, or cloud provider). [Yes]
\end{enumerate}

\item If you are using existing assets (e.g., code, data, models) or curating/releasing new assets, check if you include:
\begin{enumerate}
  \item Citations of the creator If your work uses existing assets. [Not Applicable]
  \item The license information of the assets, if applicable. [Not Applicable]
  \item New assets either in the supplemental material or as a URL, if applicable. [Not Applicable]
  \item Information about consent from data providers/curators. [Not Applicable]
  \item Discussion of sensible content if applicable, e.g., personally identifiable information or offensive content. [Not Applicable]
\end{enumerate}

\item If you used crowdsourcing or conducted research with human subjects, check if you include:
\begin{enumerate}
  \item The full text of instructions given to participants and screenshots. [Not Applicable]
  \item Descriptions of potential participant risks, with links to Institutional Review Board (IRB) approvals if applicable. [Not Applicable]
  \item The estimated hourly wage paid to participants and the total amount spent on participant compensation. [Not Applicable]
\end{enumerate}

\end{enumerate}

\onecolumn
\appendix
\renewcommand*\contentsname{Contents of Appendix}
 \addtocontents{toc}{\protect\setcounter{tocdepth}{1}}
\aistatstitle{LITE: Efficiently Estimating Gaussian Probability of Maximality \\
Supplementary Materials}

\tableofcontents

\vspace{60pt}

\section{THOMPSON SAMPLING MONTE CARLO}\label{sec:TSE} %
Thompson sampling (TS)~\parencite{thompson1933likelihood, russo2016information, russo2018tutorial, chapelle2011empirical} is a strategy for Bayesian optimisation that naturally incorporates an \textit{exploration-exploitation trade-off}~\parencite{auer2002using}. 
By directly sampling from the posterior probability of maximality $\mathbb P[x \in X^* | \mathcal D]$, Thompson sampling effectively incorporates the knowledge of $F|\mathcal D$ that is relevant for Bayesian optimization: it efficiently explores when the position of the maximizer $X^*$ is unknown given the data $\mathcal D$, otherwise it exploits.

As a strategy for Bayesian optimization, Thompson sampling uses the fact that sampling from probability of maximality is much easier than computing it. Indeed, since $X^*$ is a function of $F$, to produce samples from $X^*$ it suffices to take the argmax of samples from $F$. Since by assumption the domain \(\mathcal X\) is finite, we may represent $F$ as 
\begin{equation}\label{eq:exhaustive_TS}
     F \overset{d}{=} \mu_{F} + \mathcal L \epsilon \sim \mathcal N(\mu_{F}, \Sigma_{F}),
\end{equation}
where $\mathcal L$ is the Cholesky decomposition of $\Sigma_{F}$ and \(\epsilon \sim \mathcal N(0, I_{|\mathcal X| \times |\mathcal X|})\). Exhaustive Thompson sampling explicitly computes $\mathcal L$ and uses Equation~\eqref{eq:exhaustive_TS} to produce a sample $f \sim p(F)$, which is subsequently processed into a sample from probability of maximality through $x^* = \arg\max\nolimits_{x\in \mathcal X} f(x) \sim \mathbb P[x \in X^*]$. As a result, it costs \(\Theta(|\mathcal X|^3 + n |\mathcal X|^2)\) to draw $n$ independent Thompson samples. In contrast, as we will see shortly, using Thompson sampling to numerically approximate $\mathbb P[x \in X^*]$ costs $\Theta(|\mathcal X|^4)$ due to requiring $n \in \Theta(|\mathcal X|^2)$ independent samples. Note that this supralinear scaling persists, even under more efficient methods of drawing Thompson samples such as parametrising $F$ via Random Fourier Features~\parencite{rahimi2007random}. 

\subsection{From Thompson Sampling to Thompson Sampling Monte Carlo}
\ts{} is the standard method for computing Gaussian probability of maximality~\parencite{hennig2012entropy}, since it is both simple and delivers unbiased and consistent estimates of the \textit{probability of maximality}. As such, it represents the ground-truth against which all other estimators are empirically compared. Given access to $n$ Thompson samples, one uses histogram binning to estimate PoM with each $x \in \mathcal X$ having its separate bin. More precisely, \textit{probability of maximality} is estimated through
\begin{equation*}
    \mathbb P[x \in X^*] = \mathbb E[\mathds{1}_{x \in X^*}] \approx \frac{1}{n} \sum\nolimits_{i=1}^n \mathds 1_{x \in x_i^*} \quad\text{with samples } x_i^* \sim p(x^*).
\end{equation*}
Each Thompson sample of \(X^*\) provides simultaneously a sample of \(\mathds 1_{x \in X^*}\) for all \(x \in \mathcal X\), amortising the cost of computation. As is customary for Monte Carlo based approaches, the accuracy is in $\Theta(n^{-1/2})$\footnote{This is a direct consequence of the central limit theorem and verified empirically in Section~\ref{sec:high_Thompson_sample_requirement_of_PoM}.}. Indeed, given independent samples \(X_i^* \sim p(x^*)\), Hoeffding's inequality guarantees that
\begin{equation*}
    \mathbb P[|\frac{1}{n} \sum\nolimits_{i=1}^n \mathds 1_{x \in X_i^*} - \mathbb P[x \in X^*]| \geq \epsilon] \leq 2 \exp(- 2 n \epsilon^2).
\end{equation*}
Hence, for \(n \geq \ln(2/\delta) / (2\epsilon^2) \) the probability that \ts{} deviates more than \(\epsilon\) from the ground truth at any fixed $x \in \mathcal X$ is at most $\delta$. Figure~\ref{fig:least_number_of_Thompson_samples_required} shows the estimates of \ts{} and indicates the least required samples. The number of samples $n$ must scale in $\Theta(|\mathcal X|^2)$ to reach an acceptable relative accuracy.

\begin{algorithm}
\caption{PoM estimation with \textsc{exhaustive} \ts{}}\label{alg:e-TSE}
\begin{algorithmic}
\Require $\mu_{F}, \Sigma_{F}, \epsilon, \delta$
\State $U,D \gets eig(\Sigma_{F})$\Comment{C: $\Theta(|\mathcal X|^3)$, M: $\Theta( |\mathcal X|^2)$}
\State $\Sigma_{F}^{1/2} \gets U D^{1/2}$
\State $counts \gets (0)_{k=1}^{|\mathcal X|}$
\State $n \gets \lceil \frac{\ln(2/\delta)}{2\epsilon^2} \rceil$
\For{$i=1,\ldots,n$}
    \State $f \gets \mu_{F} + \Sigma_{F}^{1/2} \cdot \varepsilon$ for \(\varepsilon \sim \mathcal N(0,I_{|\mathcal X| \times |\mathcal X|})\) \Comment{C: $\Theta(|\mathcal X|^2)$, M: \(\Theta(|\mathcal X|)\)}
    \State $idx \gets \arg\max_{x \in \mathcal X} f_x$
    \State $counts_{idx} \gets counts_{idx} + 1$
\EndFor
\State $p_x \gets counts/n$\\
\Return $(p_x)_{x \in \mathcal X}$
\end{algorithmic}
\end{algorithm}

\subsection{The High Variance of Thompson Sampling Monte Carlo}\label{sec:high_Thompson_sample_requirement_of_PoM}
Estimating probability of maximality (PoM) using \ts{} requires many samples to bring down the variance. While the central limit theorem already dictates that the error scale in $\Theta(n^{-1/2})$, Figure~\ref{fig:least_number_of_Thompson_samples_required} demonstrates empirically that going above $|\mathcal X|^2$ samples is indeed required for a smooth PoM estimate of high fidelity.%

\begin{figure}[ht]
    \centering
    \begin{subfigure}[b]{0.256\textwidth}
        \centering
        \includegraphics[width=\textwidth]{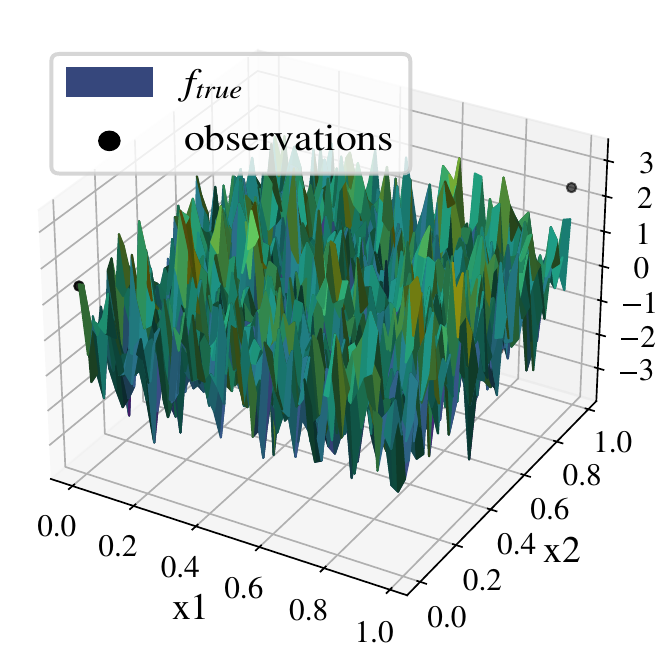}
        \caption{$f_{true}$ \& observations}
    \end{subfigure}
    \begin{subfigure}[b]{0.256\textwidth}
        \centering
        \includegraphics[width=\textwidth]{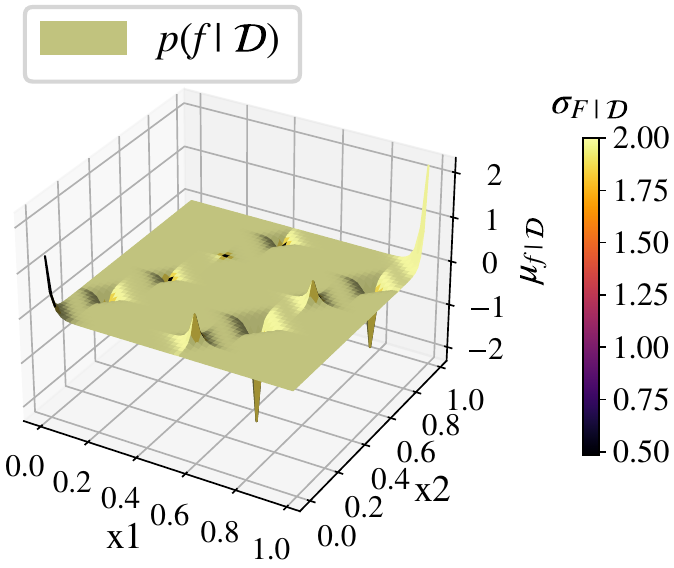}
        \caption{posterior belief}
    \end{subfigure}
    \begin{subfigure}[b]{0.256\textwidth}
        \centering
        \includegraphics[width=\textwidth]{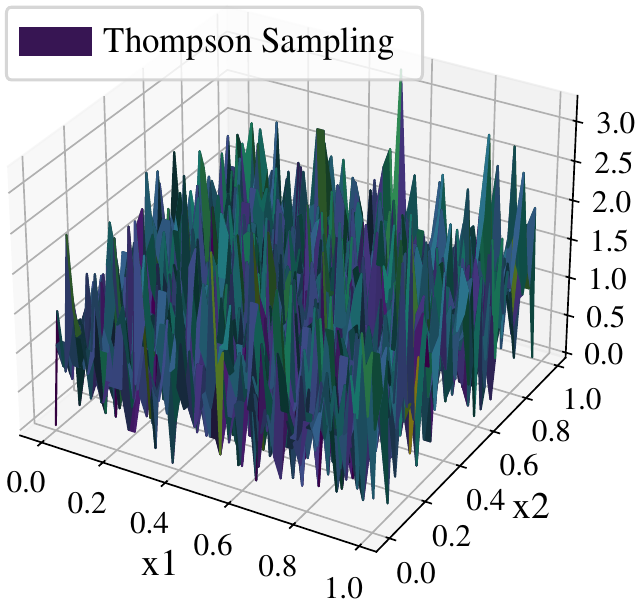}
        \caption{$0.04^2 \cdot |\mathcal X|^2$ samples}
    \end{subfigure}

    \begin{subfigure}[b]{0.256\textwidth}
        \centering
        \includegraphics[width=\textwidth]{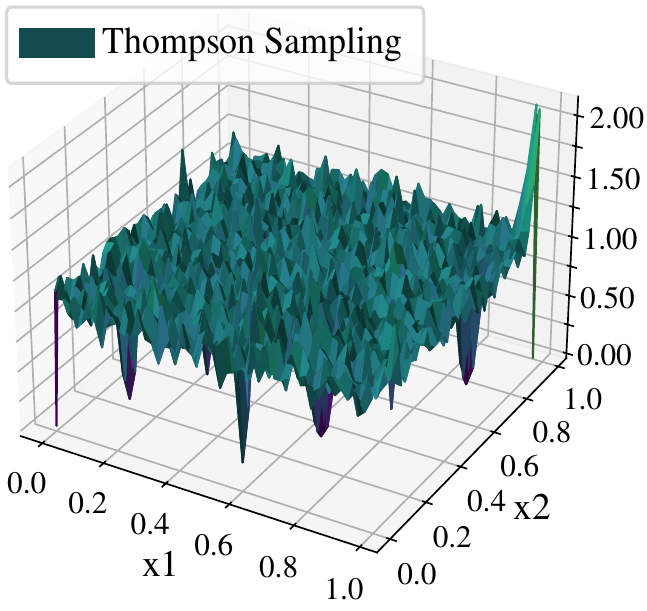}
        \caption{$0.2^2 \cdot |\mathcal X|^2$ samples}
    \end{subfigure}
    \begin{subfigure}[b]{0.264\textwidth}
        \centering
        \includegraphics[width=\textwidth]{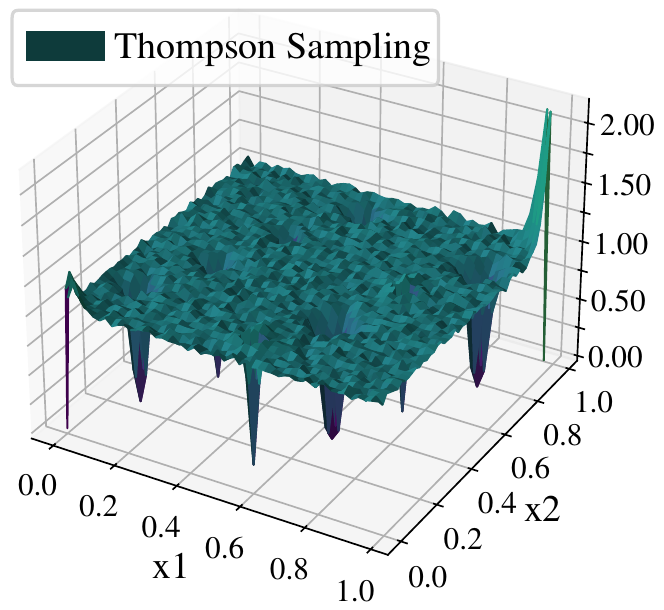}
        \caption{$1.0^2 \cdot |\mathcal X|^2$ samples}
    \end{subfigure}
    \begin{subfigure}[b]{0.256\textwidth}
        \centering
        \includegraphics[width=\textwidth]{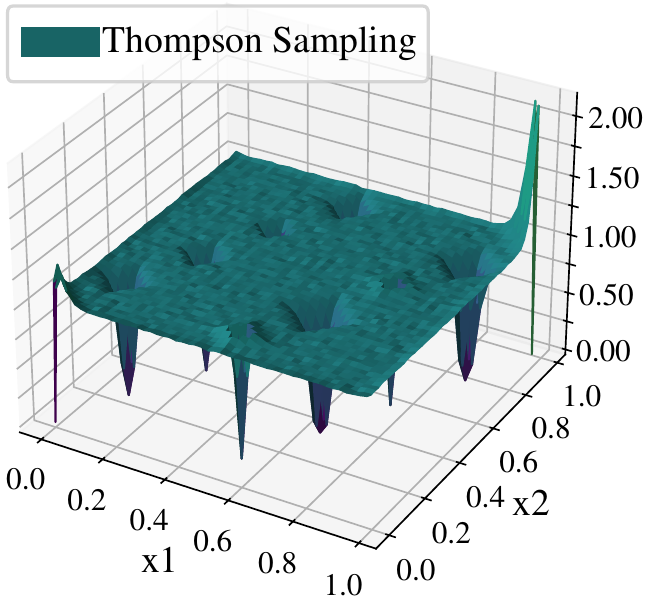}
        \caption{$5.0^2 \cdot |\mathcal X|^2$ samples}
    \end{subfigure}
    \caption{Demonstrating that $n \in \Theta(|\mathcal X|^2)$ samples are required (equivalently that $\epsilon \in \Theta(1/|\mathcal X|)$ is needed).}
    \label{fig:least_number_of_Thompson_samples_required}
\end{figure}

Here, the domain is a $50 \times 50$ grid resulting in $|\mathcal X| = 2'500$. $f_{true}$ is sampled from a centered Gaussian process with exponential kernel (length scale $0.02$, amplitude $1.0$). The prior belief over $f_{true}$ is a centered Gaussian process with exponential kernel (length scale $0.02$, amplitude $2.0$). $f_{true}$ is observed at $10$ regularly selected locations with homoscedastic additive centered Gaussian noise (\(\sigma_{noise} = 0.5\)). We vary \(1/(\epsilon\, \cdot |\mathcal X|)=:\alpha \in \{0.02, 0.4, 1.0, 5.0\}\) to observe the fidelity of \textsc{exhaustive} \ts{}. Finally, to mimic a probability density function, we divide the estimated \textit{probability of maximality} by $1/50^2$.

\newpage
\section{BIAS INTRODUCED BY NEGLECTING DEPENDENCY STRUCTURE}\label{sec:ndependence_assumption}

Let us develop some intuition on the estimation bias introduced by falsely assuming uncorrelated entries in the Gaussian reward vector (Assumption~\ref{ass:independence}). To that end, we consider some examples of discretized Gaussian processes that violate the independence assumption. 
\begin{figure}[ht]
    \centering
    \begin{subfigure}{0.49\linewidth}
        \centering
        \includegraphics[width=\linewidth]{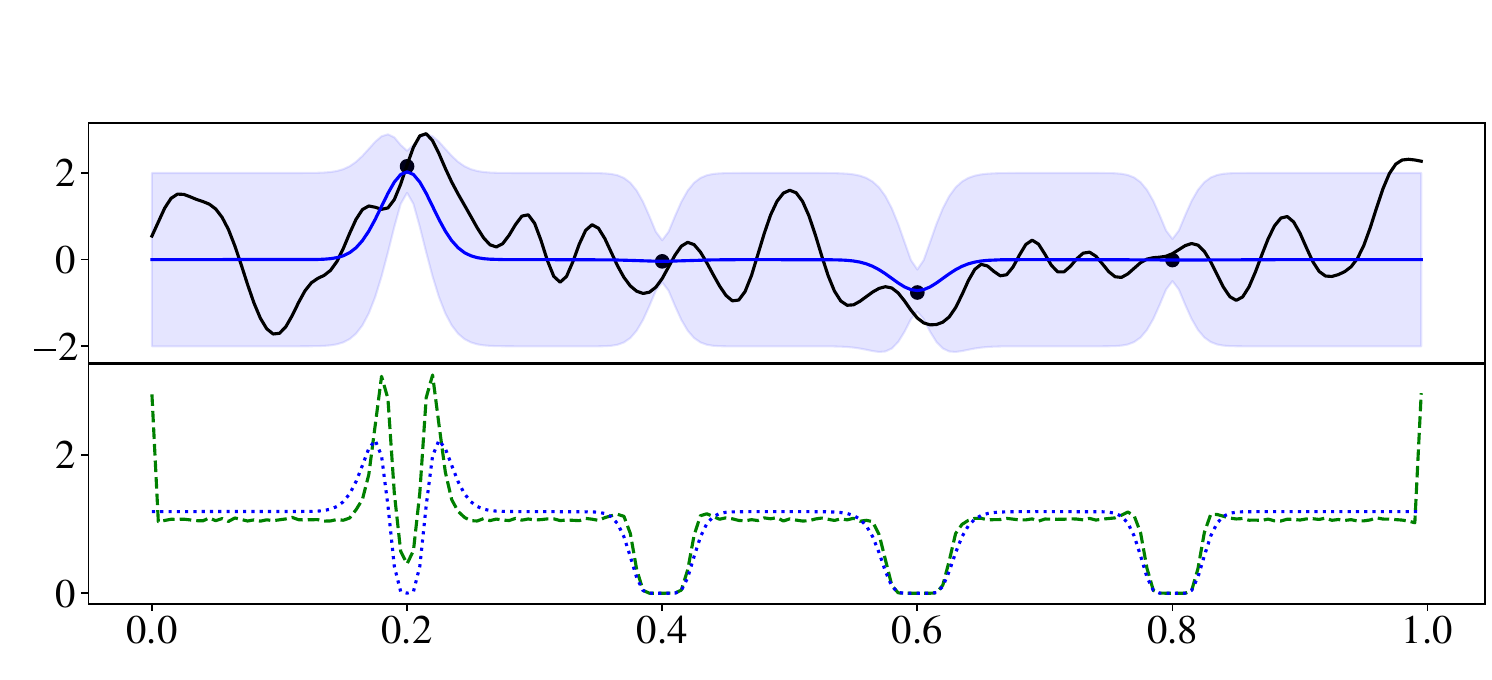}
        \caption{\(|\mathcal D| = 4\)}
    \end{subfigure}
    \begin{subfigure}{.49\linewidth}
        \centering
        \includegraphics[width=\linewidth]{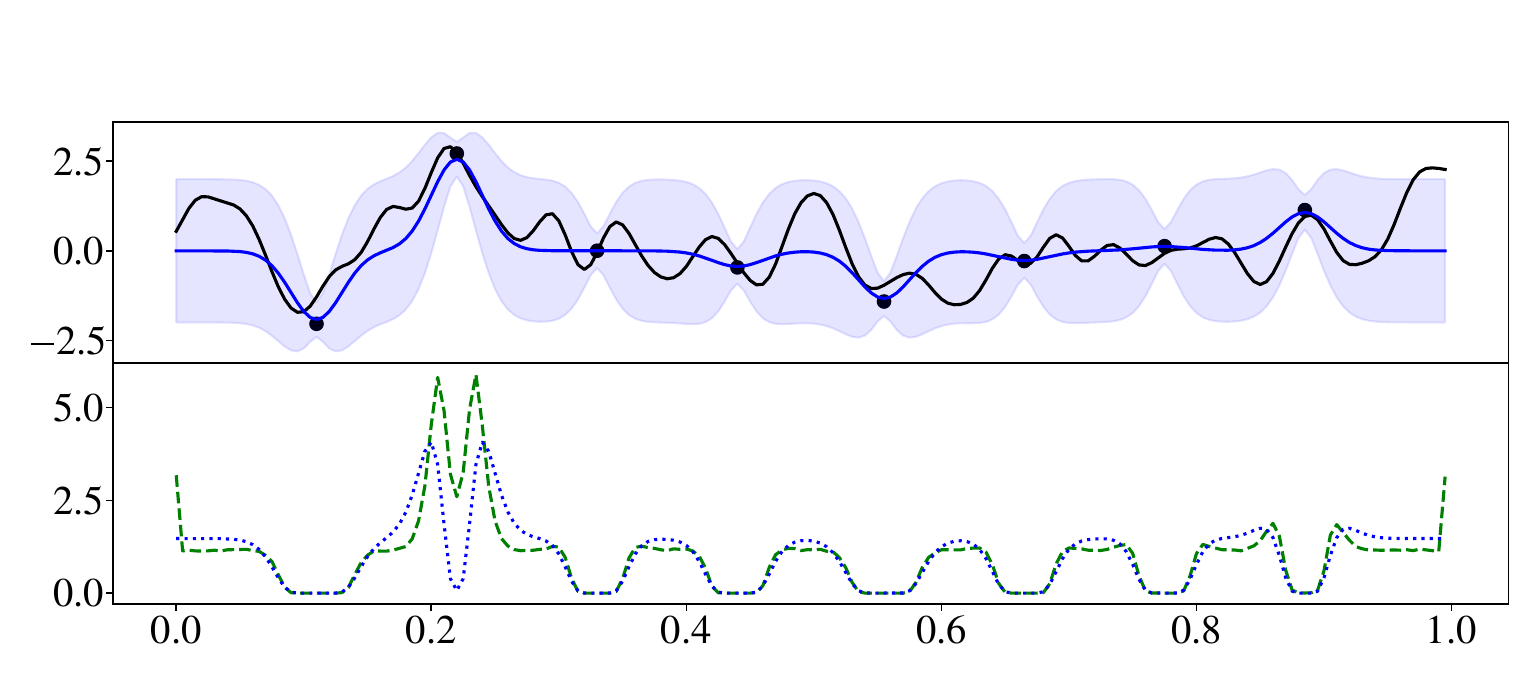}
        \caption{\(|\mathcal D| = 8\)}
        \label{fig:ie_effect_of_uniformity_of_posterior_b}
    \end{subfigure}
    \begin{subfigure}{0.495\linewidth}
        \centering
        \includegraphics[width=\linewidth]{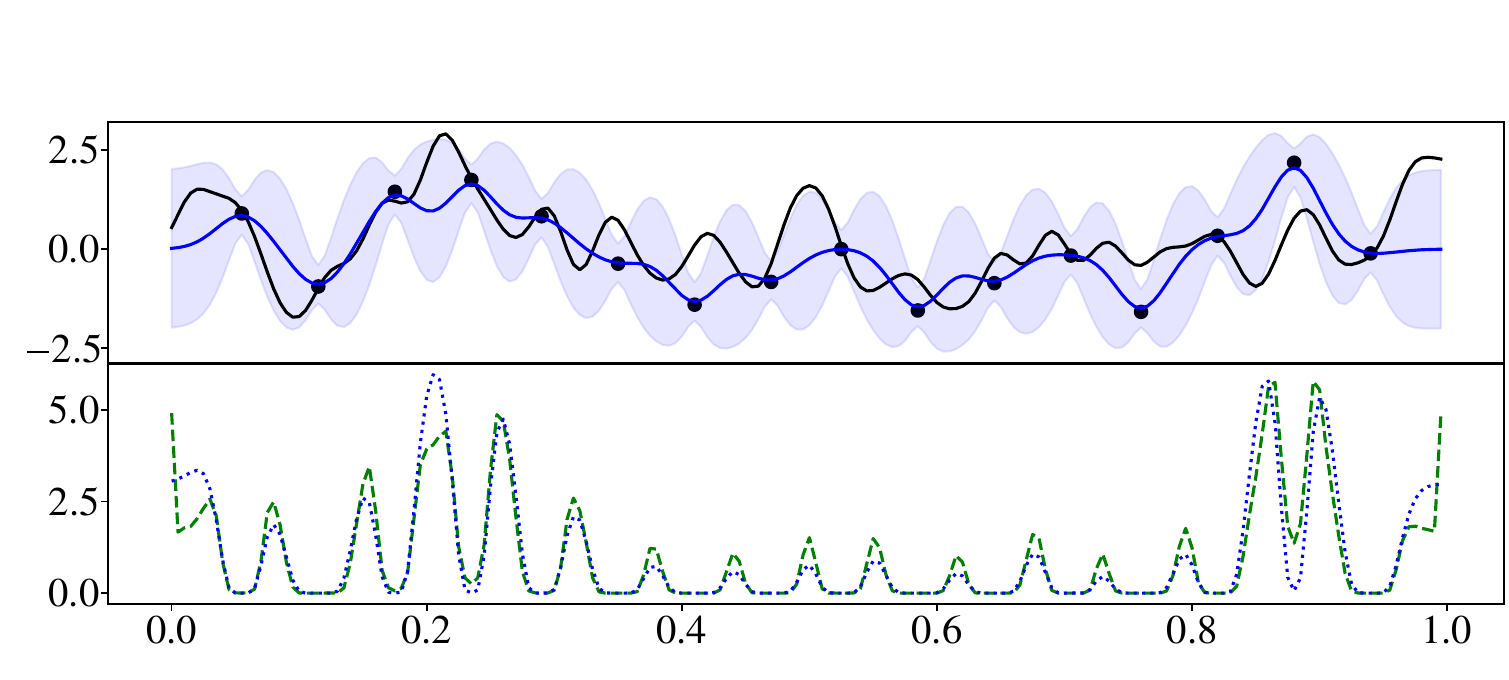}
        \caption{\(|\mathcal D| = 16\)}
    \end{subfigure}
    \begin{subfigure}{.495\linewidth}
        \centering
        \includegraphics[width=\linewidth]{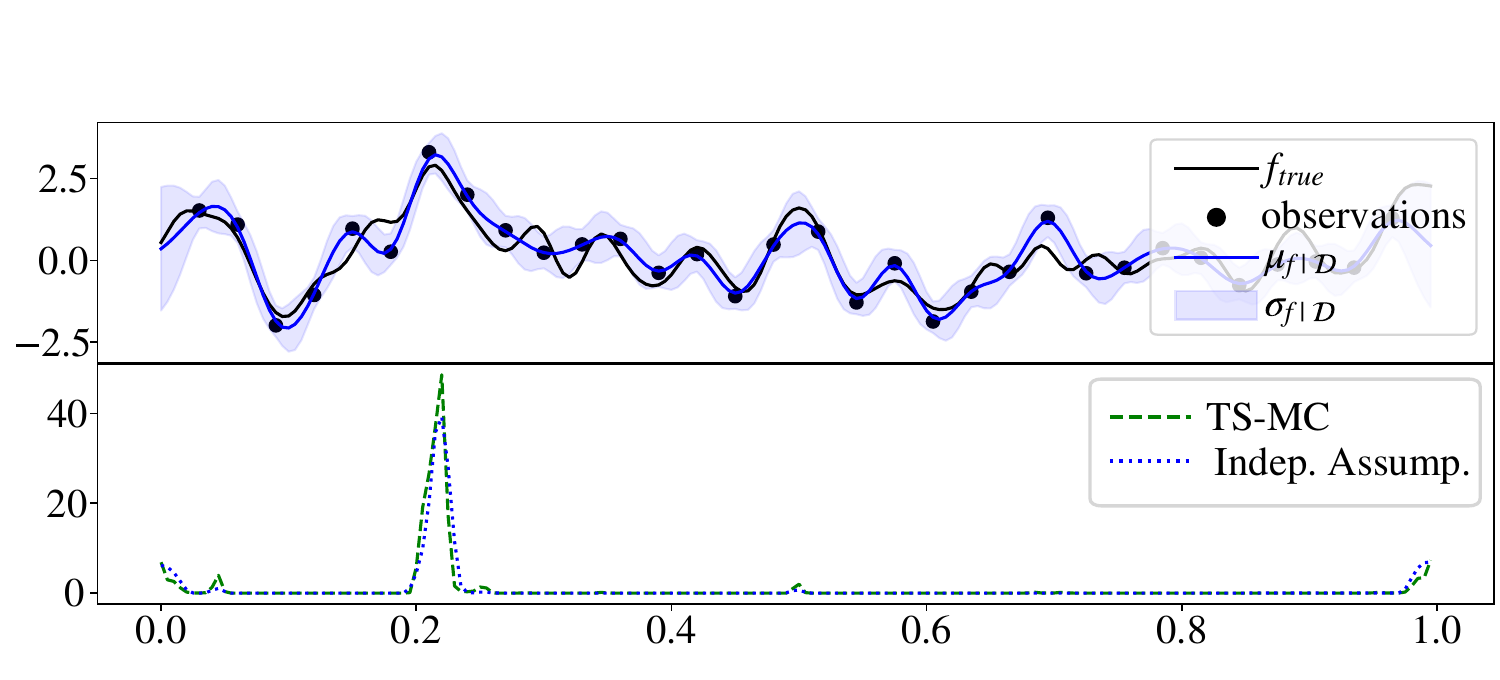}
        \caption{\(|\mathcal D| = 32\)}
    \end{subfigure}
    \caption{Falsely relying on Assumption~\ref{ass:independence} does not qualitatively change the estimation of Gaussian PoM, but rather leads to a more conservative prediction, underestimating the maximality of entries near the best observations. This can be understood as a consequence of Slepian's lemma \parencite{slepian1962one}.}
    \label{fig:ie_effect_of_uniformity_of_posterior}
\end{figure}
Figure~\ref{fig:ie_effect_of_uniformity_of_posterior} considers posteriors with varying degree of concentration of measure, covering different stages of Bayesian optimization. The figure demonstrates that PoM estimation based on Assumption~\ref{ass:independence} qualitatively captures the ground-truth PoM (here estimated using TS-MC). The degeneracies at the border of the ground-truth PoM correspond to dirac-deltas of the probability density function of PoM, but have small effective measure and as such are of little concern.
Theoretically, the dominant effect of falsely assuming independence can be understood by considering Slepian's lemma~\parencite{slepian1962one}, which implies that if $F \sim \mathcal N (\mu, \Sigma)$ and $\tilde F \sim \mathcal N(\mu, \mathrm{diag}(\Sigma))$ with $\Sigma_{i,j} \geq 0\ \forall i,j$, it holds that
\begin{equation*}
    \mathbb P[F^* > t] \leq \mathbb P[\tilde F^* > t] \ \forall t \in \mathbb R \implies \mathbb E[F^*] \leq \mathbb E[\tilde F^*].
\end{equation*}
In light of this, the minor differences that can be observed in Figure~\ref{fig:ie_effect_of_uniformity_of_posterior} between the PoM under Assumption~\ref{ass:independence} and the ground-truth PoM are explained as follows: Assumption~\ref{ass:independence} leads to over-estimating the maximum reward $F^*$ (Slepian's lemma), which results in overly-cauteous estimation of PoM, in particular under-estimating regions associated with promising observations. However, we stress that despite this bias towards uniformity, estimation under the \ia{} still manages to qualitatively capture the ground-truth PoM.

The experimental details of Figure~\ref{fig:ie_effect_of_uniformity_of_posterior} are as follows: the domain consists of $|\mathcal X| = 200$ equidistant points on which \(f_{true}\), a sample from a centered Gaussian process \(\mathcal{GP}\) with squared exponential kernel (length scale $0.02$, amplitude $1.0$), is evaluated. The prior belief over $f_{true}$ coincides with \(\mathcal{GP}\) except for the doubling of the amplitude to \(2.0\). $f_{true}$ is observed at $|\mathcal D|$ regularly selected locations with homoscedastic additive centered Gaussian noise (\(\sigma_{noise} = 0.5\)). We set the accuracy parameter to \(\epsilon = 1/({5 | \mathcal X|})\). The estimated probability mass functions (of PoM) are rescaled by $1/|\mathcal X|$ to simulate a probability density function.

Finally, to further argue for our method of neglecting dependency structure, we next demonstrate that the computational complexity of any unbiased estimator of PoM is lower bounded by the number of entries in $\Sigma_F$, i.e., it lies in $\Omega(|\mathcal X|^2)$. Lemma~\ref{lem:exampleNecessityFullCovarianceMatrixForUnbiasedEstimation}, through construction of a simple synthetic example with closed-form \textit{probability of maximality}, shows that in general knowledge on all entries in $\Sigma_F$ would be required. Note that the Lemma applies to all unbiased estimators, not just \ts{}, which under $\Theta(|\mathcal X|^2)$ samples scales in $\Theta(|\mathcal X|^4)$.
\begin{restatable}[Example to illustrate necessity of knowing the full covariance matrix for unbiased estimation]{lem}{exampleNecessityFullCovarianceMatrixForUnbiasedEstimation}\label{lem:exampleNecessityFullCovarianceMatrixForUnbiasedEstimation}
    Consider $F \sim \mathcal N(0, I + s(e_i e_j^T + e_j e_i^T))$ in $\mathbb R^{|\mathcal X|}$, where $(e_i)_j = \mathds{1}_{i=j}$, $i < j$, and $s \in [0,1)$. Then it holds that
    \begin{equation*}
        \lim_{s \to 1^-} \mathbb P[k \in \arg\max_{h} F_h] = \begin{cases} 
            \frac{1}{|\mathcal X|-1} & k \not= i \land k \not=j\\
            \frac{1/2}{|\mathcal X|-1} & otherwise
        \end{cases}.
    \end{equation*}
\end{restatable}
As is apparent, knowledge of the position $(i,j)$ of the non-zero (upper) off-diagonal entry is essential for an unbiased prediction of the probability of maximality as $s \to 1^-$. Without a sparse representation of $\Sigma_F$, obtaining the pair $(i,j)$ would require checking all upper diagonal entries in $\Omega(|\mathcal X|^2)$. In less synthetic examples sparsity may not be present—hence, in general, an unbiased estimator of \textit{probability of maximality} really requires at least $\Omega(|\mathcal X|^2)$ compute. As such, neglecting dependency structure in the covariance matrix is an essential ingredient for obtaining \textit{almost-linear} runtime in the size of the Gaussian reward vector $|\mathcal X|$: unless this bias is adopted the runtime would scale at least quadratically in $|\mathcal X|$.

\section{\alite{}}\label{sec:alite}

\alite{} is our \textit{accurate} instantiation of \lite{}, which relies on nested binary search to match quartiles. So, we want to determine $m_x$ and $s_x$ such that for all $x\in \mathcal X$ it holds that
\begin{equation*}
    \mathbb P[x \in \tilde X^*] = \mathbb E[g^x(F_x)] \approx \mathbb E[\Phi(\frac{F_x- m_x}{s_x})].
\end{equation*}
Since $g^x$, by virtue of being a \textit{cumulative distribution function} of a continuous random variable $\max\nolimits_{z \not=x} \tilde F_z$, is continuous and monotonously increasing, we could efficiently find its first and third quartiles $q_1$ and $q_3$ to any accuracy using binary search. Then, we could select $m_x$ and $s_x$ such that the Gaussian approximation has matching quartiles. However, with evaluations of $g^x$ costing $\Theta(|\mathcal X|)$, repeating the procedure for each $x \in \mathcal X$ would lead to a total cost in $\Omega(|\mathcal X|^2)$, already exceeding the desired budget. 

To get around this conundrum, we approximate in a first step $g(f) := \prod\nolimits_z \mathbb P[F_z \leq f]$ with $\Phi((f-m)/s)$ based on quartile matching, before $\forall x \in \mathcal X$ matching quartiles of a separate normal distribution $\Phi((f-m_x)/s_x)$ to $\Phi((f-m)/s)\ /\ \mathbb P[F_x \leq f] \approx g^x(f) := \prod\nolimits_{z \not= x} \mathbb P[F_z \leq f]$. To ensure stability we operate in log-space.
\begin{itemize}[leftmargin=2cm]
    \item[Step 1:] Quartile matching s.t. $\Phi(\frac{f-m}{s}) \approx g(f) := \prod_z \mathbb P[F_z \leq f]$.
    \item[Step 2:] Quartile matching $\forall x \in \mathcal X$ s.t. $\Phi(\frac{f-m_x}{s_x}) \approx \Phi(\frac{f-m}{s}) / \mathbb P[F_x \leq f] \approx g^x(f)$.
\end{itemize}
At both stages, once the quartiles $q_1$ and $q_3$ are known, the selection of mean $m$ and standard deviation $s$ can be done in closed form, since $\Phi(\frac{q_1 - m}{s}) = 0.25$ and $\Phi(\frac{q_3 - m}{s}) = 0.75$ directly imply the value of $m$ and $s$ through\footnote{Basic algebra and symmetry of $\Phi$ yield $m-q_1 = s \Phi^{-1}(0.75)$ and $q_3 - m = s \Phi^{-1}(0.75)$ from which the result follows swiftly by subtracting and adding the equations.}
\begin{equation*}
    m = \frac{q_3 + q_1}{2} \qquad\text{and}\qquad s = \frac{q_3 - q_1}{2 \Phi^{-1}(0.75)}.
\end{equation*}

However, there is a caveat. $\tilde g^x(f) := \Phi((f-m)/s)\ /\ \Phi((f-\mu_{F_x})/{\sigma_{F_x}})$ is not a \textit{cumulative distribution function}, an unfortunate consequence of approximating $g$ by the normal $\Phi((f-m)/s)$. Although it always holds that $m > \mu_{F_x}\ \forall x \in \mathcal X$\footnote{Given $|\mathcal X| > 1$ and $\sigma_{F_x} > 0\ \forall x \in \mathcal X$, this follows from $\mathbb P[\tilde F^* \leq q] < \mathbb P[F_x \leq q]$ implying $0.75 < \mathbb P[F_x \leq q_3]$ and $0.25 < \mathbb P[F_x \leq q_1]$. As such, $\mu_{F_x} = (q_3^{F_x} + q_1^{F_x})/2 < (q_3 + q_1)/2 = m$.}, $s$ can be both larger and smaller than $\sigma_{F_x}$.
As Figure~\ref{fig:approximation_to_gx_is_not_a_cdf} shows, in the latter case $\tilde g^x$ may not even cross the quartiles $0.25$ and $0.75$. The former case is more benign, admitting a continuous monotonously increasing section with range $(0,1]$, outside of which $\tilde g^x$ always exceeds $1$\footnote{This was verified for a large variety of $m, \mu_{F_x}, s,$ and $\sigma_{F_x}$, but not analytically.}. As such, a binary search procedure can still be used to efficiently find its "quartiles", i.e. $f$ such that $\tilde g^x(f) = 0.25$ or $\tilde g^x(f) = 0.75$.

\begin{figure}[ht]
    \centering
    \includegraphics[width=0.5\linewidth]{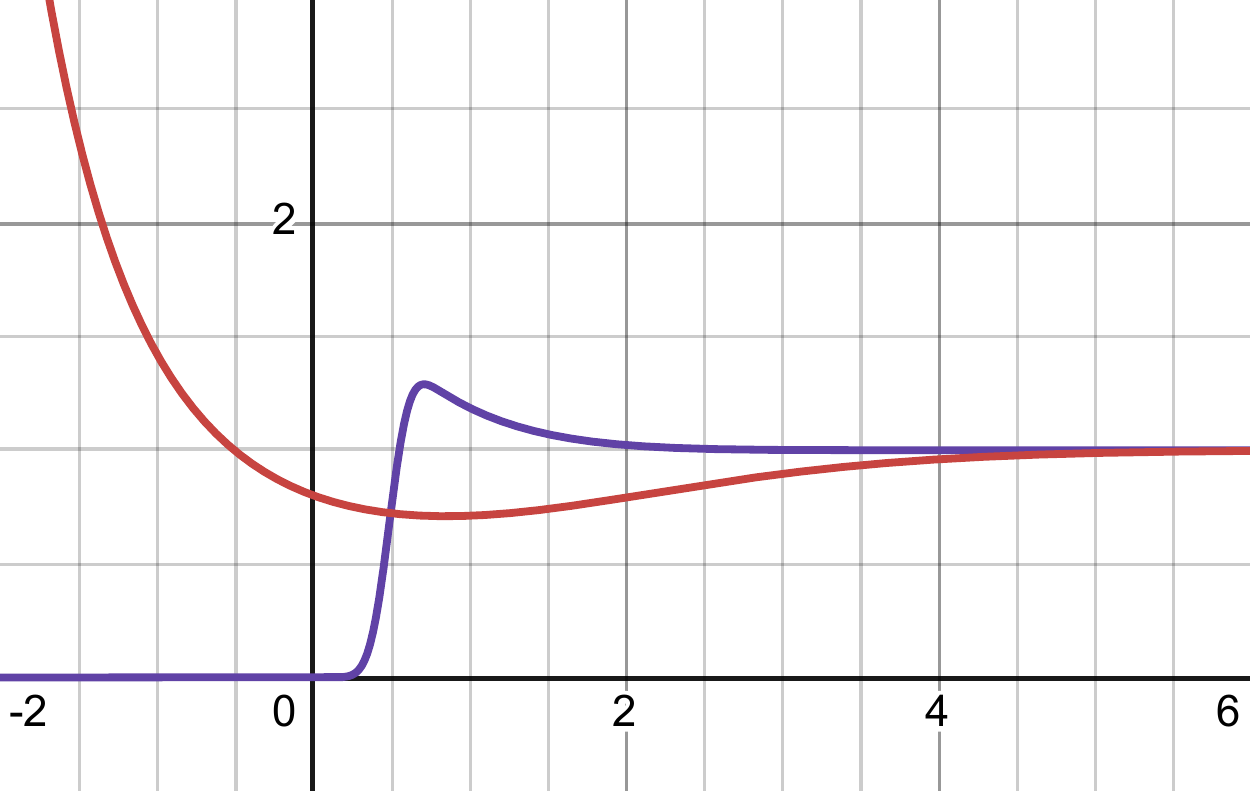}
    \caption{The graph of $\Phi(\tfrac{f-0.5}{2}) / \Phi(\tfrac{f-0}{1})$ (red line) and $\Phi(\tfrac{f-0.5}{0.1}) / \Phi(\tfrac{f-0}{1})$ (purple line). Unless $s \leq \sigma_{F_x}$, $\Phi(\tfrac{f-m}{s}) / \Phi(\tfrac{f-\mu_{F_x}}{\sigma_{F_x}})$ blows up as $f\to -\infty$ and as a consequence may not even cross $0.25$ and $0.75$.}
    \label{fig:approximation_to_gx_is_not_a_cdf}
\end{figure}

So, to ensure termination of quartile matching, we instead match $\Phi(\tfrac{f-m_x}{s_x})$ to $\Phi(\tfrac{f-m}{\min(s, \sigma_{F_x})}) / \Phi(\tfrac{f-\mu_{F_x}}{\sigma_{F_x}})$
before predicting $\mathbb P[x \in X^*] \approx \Phi((\mu_{F_x} - m_x)/\sqrt{\sigma_{F_x}^2 + s_x^2})$, a method we call \alite{}-II due to its reliance on two consecutive steps of quartile matching.

If $\sigma_{F_x} \ll s$, \alite{}-II can lead to a vast underestimation of \textit{probability of maximality}\footnote{For $m > \mu_{F_x}$, decreasing $s$ always decreases $\mathbb E[\Phi((F_x - m)/{s}) / \Phi((F_x - \mu_{F_x})/{\sigma_{F_x}})] \approx \mathbb P[x \in X^*]$ because $\phi((F_x - \mu_{F_x})/{\sigma_{F_x}})/\sigma_{F_x} \cdot 1 / \Phi((F_x - \mu_{F_x})/{\sigma_{F_x}})$ is dominant to the left of $\mu_{F_x} < m$, which is weighted less in the integral as $s \to 0$.}. Fortunately, there is an alternative method of approximation. As explained in Figure~\ref{fig:approximation_errror_of_using_g_instead_of_gx}, using $g$ instead of $g^x$ usually does not introduce significant error except for points $x \in \mathcal X$ so likely maximising that they dominate the shape of $g$.

\begin{figure}[ht]
    \centering
    \includegraphics[width=0.7\linewidth]{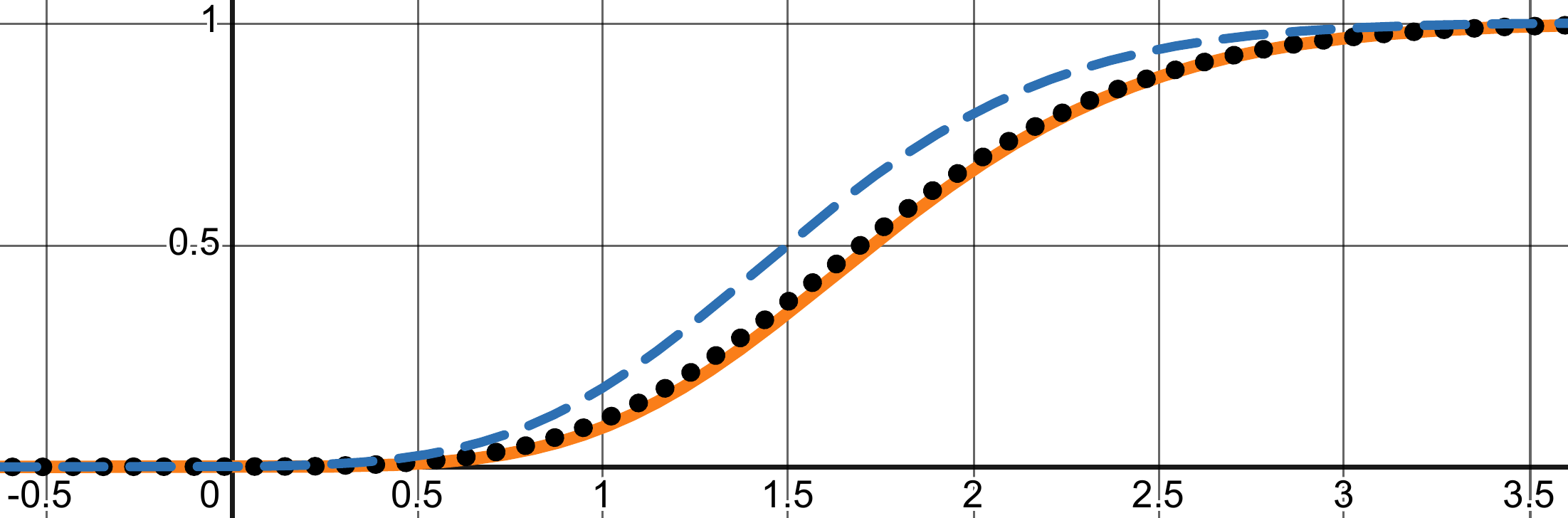}
    \caption{Illustration of the approximation error of using $g(f) = \Phi(f)^{10} \Phi(f-1)$ (orange solid line) instead of $g^{x_1} = \Phi(f)^{10}$ (blue dashed) and $g^{x_2} = \Phi(f)^9 \Phi(f-1)$ (black dotted). We underestimate $g^x$ for the likely maximiser $x_1$, leading to an overly conservative (under)estimation of \textit{probability of maximality}. Contrastingly, $g^{x_2}$ is well approximated by $g$ for the unlikely maximiser $x^2$.}
    \label{fig:approximation_errror_of_using_g_instead_of_gx}
\end{figure}

 \alite{}-I exploits this observation by only relying on the initial quartile matching, where we approximated $g(f) \approx \Phi((f-m)/{s})$. That is, it directly predicts PoM as $\mathbb P[x \in X^*] \approx \Phi((\mu_{F_x} - m)/{\sqrt{\sigma_{F_x}^2 + s^2}})$. As in the case of \alite{}-II, here the approximation of using $g$ instead of $g^x$ biases the \textit{probabilities of maximality} towards $0$.
 
 For the most accurate estimation, we combine \alite{}-I and \alite{}-II by taking the element-wise maximum of their respective predicted PoMs, i.e., $\mathbb P[x \in X^*] \approx \max(\Phi((\mu_{F_x} - m_x)/{\sqrt{\sigma_{F_x}^2 + s_x^2}}), \Phi((\mu_{F_x} - m)/{\sqrt{\sigma_{F_x}^2 + s^2}}))$. Whereas \alite{}-I is targeted at unlikely maximizers with $\mu_{F_x} \ll m$, \alite{}-II is built for the opposite case where $\mu_{F_x} \approx m$. Together, they solve both cases well. Taking the maximum is justified since both \alite{}-I and \alite{}-II involve approximations that lower their predicted \textit{probabilities of maximality}. As a final step, we add a global normalization to $1$, once again relying on Assumption~\ref{ass:unique_maximiser}. We remark that this final step typically does not significantly affect the estimation accuracy since the estimated PoM is already almost normalized.

The complete procedure for estimation with \alite{} is described in Algorithm~\ref{alg:NIE}, along with its sub-procedures in Algorithms~\ref{alg:NIE_I_Search}-\ref{alg:NIE_II_Search_Window}. The logarithmic search windows for the two stages of quartile matching are selected according to the results in Proposition~\ref{thm:logarithmicFStarQuantileSearch} and Proposition~\ref{thm:logarithmicFStarWithoutXQuantileSearch} (plugging in $b \in \{0.25, 0.75\}$), while additionally taking into account that for the second stage we do not have access to the ground-truth quartiles of $g$ and hence its statistics $m$ and $s$ (we only have upper and lower bounds from the first stage). The algorithm runs in $\Theta(\Sigma_{l=1}^{\log_2 k} 2^l\, |\mathcal X|) = \Theta(k |\mathcal X|)$ where $k$ denotes the final depth that is needed for uniform convergence of the lower and upper bounds on $p_x$.

\begin{restatable}[Logarithmic $\tilde F^*$-quantile search]{prop}{logarithmicFStarQuantileSearch}\label{thm:logarithmicFStarQuantileSearch}
    Let $b \in [0.25,1)$ and $\tilde F \sim \mathcal N(\mu_F, \mathrm{diag}(\sigma_{F_1}^2, \ldots, \sigma_{F_{|\mathcal X|}}^2))$ with $|\mathcal X| > 1$. Assume $\exists x : \sigma_{F_x} > 0$. Define $g(f) := \Pi_{z} \mathbb P[\tilde F_z \leq f]$, which is continuous and strictly monotonously increasing. Then $ \exists ! \bar f \in \mathbb R$ s.t. $g(\bar f) = b$. It can be found efficiently using logarithmic search with search window
    \begin{equation*}
        \mu_F^{min} + \sigma_F^{min} \Phi^{-1}(b^{1/|\mathcal X|}) \leq \bar f \leq \mu_F^{max} + \sigma_F^{max} \Phi^{-1}(b^{1/|\mathcal X|}).
    \end{equation*}
    The size of the search window is bounded by $\mu_{F}^{max} - \mu_{F}^{min} + \Phi^{-1}(b^{1/|\mathcal X|}) \sigma_{F}^{max} \in \Theta(\sqrt{\log |\mathcal X|})$. Run $k$ steps of binary search resulting in best approximant $\bar f^k$. Then
    \begin{equation*}
        |\bar f - \bar f^k| \leq \frac{\mu_{F}^{max} - \mu_{F}^{min} + \Phi^{-1}(b^{1/|\mathcal X|}) \sigma_{F}^{max}}{2^{k+1}},
    \end{equation*}
    i.e. we obtain exponential convergence with linear order. So, to ensure $|\bar f \!-\! \bar f^k| \leq \nu$, $k = \log_2((\mu_{F}^{max} - \mu_{F}^{min} + \Phi^{-1}(b^{1/|\mathcal X|}) \sigma_{F}^{max}) / (2 \nu))  \in \Theta(\log (\log (| \mathcal X|) / \nu))$ steps suffice.
\end{restatable}

\newpage
\begin{restatable}[Logarithmic $\tilde F^{* \setminus x}$-quantile search]{prop}{logarithmicFStarWithoutXQuantileSearch}\label{thm:logarithmicFStarWithoutXQuantileSearch}
    Let $b \in (0,1)$, $m, \mu_{F_x} \in \mathbb R$, and $s, \sigma_{F_x} \in \mathbb R_+$ such that $m > \mu_{F_x}$ and $s \leq \sigma_{F_x}$. Define $\tilde g^x(f) := \Phi((f-m)/{s}) / \Phi((f- \mu_{F_x})/{\sigma_{F_x}})$, which is continuous and strictly monotonously increasing on a section with range $(0,1]$ and exceeds $1$ elsewhere. Then $\exists!\, \bar f_x \in \mathbb R$ s.t. $\tilde g^x(\bar f_x) = b$. It can be found efficiently using logarithmic search with search window
    \begin{equation*}
        \min(\mu_{F_x} - \sqrt{2} \sigma_{F_x}, \max(\frac{m + \mu_{F_x}}{2} - \frac{\sigma_{F_x}^2 \ln(2/b)}{m - \mu_{F_x}}, m - \sqrt{\tfrac{2 \ln (2/b)}{1-s^2/\sigma_{F_x}^2}} s)) \ \leq\ \bar f_x \ \leq\ m + \Phi^{-1}(b) \cdot s
    \end{equation*}
    The size $\Delta$ of the search window is independent of $|\mathcal X|$, i.e. $\Delta \in \Theta(1)$. Run $k$ steps of binary search resulting in best approximant $\bar f^k_x$. Then $|\bar f_x - \bar f^k_x| \leq \tfrac{\Delta}{2^{k+1}}$, i.e., we obtain exponential convergence with linear order. So, to ensure $|\bar f_x - \bar f^k_x| \leq \nu$, $k = \log_2(\Delta/(2\nu)) \in \Theta(\log (1/\nu))$ steps suffice.
\end{restatable}

\begin{algorithm}[!ht]
\caption{\alite{}}\label{alg:NIE}
\begin{algorithmic}
\Require $\mu_{F}, \sigma_{F}, \epsilon$
\State $max\text{-}error \gets \epsilon$
\State $d \gets 1$
\While{$max\text{-}error \geq \epsilon$}
    \State $d \gets d \cdot 2$

    \State $(m^{up}, m^{low}, s^{up}, s^{low}) \gets \text{\alite{}-I-S}(d, \mu_F, \sigma_F)$ \Comment{C: $\Theta(d |\mathcal X|)$, M: $\Theta(1)$}

    \State $(m_x^{up}, m_x^{low}, s_x^{up}, s_x^{low})_{x \in \mathcal X} \gets \text{\alite{}-II-S}(d, m^{up}, m^{low}, s^{up}, s^{low}, \mu_F, \sigma_F)$ \Comment{C: $\Theta(d |\mathcal X|)$, M: $\Theta(|\mathcal X|)$}

    \If {$s^{low} < 0$ or $\min_x s_x^{low} < 0$}
        \State jump to the top of this while-loop
    \EndIf

    \State $(p_x^{I,up})_{x \in \mathcal X} \gets \left(\Phi(\max(\frac{\mu_{F_x} - m^{low}}{\sqrt{\sigma_{F_x}^2 + (s^{low})^2}}, \frac{\mu_{F_x} - m^{low}}{\sqrt{\sigma_{F_x}^2 + (s^{up})^2}}))\right)_{x \in \mathcal X}$
    \State $(p_x^{I,low})_{x \in \mathcal X} \gets \left(\Phi(\min(\frac{\mu_{F_x} - m^{up}}{\sqrt{\sigma_{F_x}^2 + (s^{low})^2}}, \frac{\mu_{F_x} - m^{up}}{\sqrt{\sigma_{F_x}^2 + (s^{up})^2}}))\right)_{x \in \mathcal X}$
    \State $(p_x^{II,up})_{x \in \mathcal X} \gets \left(\Phi(\max(\frac{\mu_{F_x} - m_x^{low}}{\sqrt{\sigma_{F_x}^2 + (s_x^{low})^2}}, \frac{\mu_{F_x} - m_x^{low}}{\sqrt{\sigma_{F_x}^2 + (s_x^{up})^2}}))\right)_{x \in \mathcal X}$
    \State $(p_x^{II,low})_{x \in \mathcal X} \gets \left(\Phi(\min(\frac{\mu_{F_x} - m_x^{up}}{\sqrt{\sigma_{F_x}^2 + (s_x^{low})^2}}, \frac{\mu_{F_x} - m_x^{up}}{\sqrt{\sigma_{F_x}^2 + (s_x^{up})^2}}))\right)_{x \in \mathcal X}$
    \State $(p_x^{up}, p_x^{low})_{x \in \mathcal X} \gets (\max(p_x^{I,up}, p_x^{II,up}), \max(p_x^{I,low}, p_x^{II,low}))_{x \in \mathcal X}$

    \State $max\text{-}error \gets \max_{x \in \mathcal X} p_x^{up} - p_x^{low}$
\EndWhile
\State $(p_x)_{x \in \mathcal X} \gets ((p_x^{up} + p_x^{low})/2)_{x \in \mathcal X}$\\
\Return $(p_x / \sum_{z \in \mathcal X} p_z)_{x \in \mathcal X}$
\end{algorithmic}
\end{algorithm}

The shared final depth $k$ of the nested binary search procedures is actually quite small. Indeed, as explained in Proposition~\ref{thm:logarithmicFStarQuantileSearch}, to ensure the quartiles $q_1$ and $q_3$ of $g$ are determined up to accuracy $\nu$ it suffices to run $k \in \Theta(\log (\log( |\mathcal X|)/\nu))$ steps. This describes the efficiency of \alite{}-I. Similarly, according to Proposition~\ref{thm:logarithmicFStarWithoutXQuantileSearch}, the second stage of binary search produces $\nu$-accurate quartiles in just $k = \log_2(\Delta/(2\nu)) \in \Theta(\log(1/\nu))$ steps. Stacking the two will result in $\nu$-accurate quartiles of $\Phi((f-m)/{s}) / \Phi((f-\mu_{F_x})/{\sigma_{F_x}})$ at a shared depth $k$ scaling in $\Theta(\log(\log(|\mathcal X|)/\nu))$. This describes the efficiency of \alite{}-II.

As a final detail, we do not seek $\nu$-accurate quartiles, but rather $\epsilon$-converged predictions of \textit{probability of maximality}. The error propagation from quartiles to predictions is provided in Lemma~\ref{lem:nieErrorPropagation}. It presents the required $\nu$ such that \alite{}-I is $\epsilon$ accurate to the analytical \alite{}-I, which is based on the actual quartiles ($k \to \infty$). According to the lemma it suffices to take $\nu = \epsilon \cdot \bar{s}^2 / (\max\nolimits_x | \mu_{F_x} - \bar m| + \bar s)$, where $\bar m$ and $\bar s$ describe the mean and standard deviation implied by the true quartiles. By using $m_x, \bar m_x$ and $s_x, \bar s_x$ instead of $m, \bar m$ and $s, \bar s$, Lemma~\ref{lem:nieErrorPropagation} applies directly to \alite{}-II as well. Due to the linear propagation of error from $\epsilon$ to $\nu$ predicted by Lemma~\ref{lem:nieErrorPropagation}, we obtain a total runtime complexity in $\Theta(|\mathcal X| \log(\log (| \mathcal X|)/\epsilon))$ and memory consumption $\Theta(|\mathcal X|)$. In terms of asymptotic efficiency we are on par with \flite{}, being essentially independent of $\nu$ and linear in $|\mathcal X|$. However, in practice the constant factor is quite a bit worse, as can be observed in Figure~\ref{fig:runtime_over_domain_size}.

\begin{restatable}[\alite{} error propagation]{lem}{nieErrorPropagation}\label{lem:nieErrorPropagation}
    Let $\mu_F \in \mathbb R^{|\mathcal X|}$, $\sigma_F \in \mathbb R^{|\mathcal X|}_+$, and $\epsilon > 0$. Let $\bar q_1, \bar q_3 \in \mathbb R$ and $q_1,q_3 \in \mathbb R$ be pairs of quartiles such that $|\bar q_1 - q_1| \leq \nu$ and $|\bar q_3 - q_3| \leq \nu$ for $\nu = \epsilon \cdot \bar s^2 /(\max\nolimits_{x}|\mu_{F_x} - \bar m| + \bar s)$. Then with $m = (q_3 + q_1)/{2}$ and $\bar m = (\bar q_3 + \bar q_1)/{2}$ the means and $s = (q_3 - q_1)/{(2 \Phi^{-1}(0.75))}$ and $\bar s = (\bar q_3 - \bar q_1)/(2 \Phi^{-1}(0.75))$ the standard deviations of quartile-matched Gaussians, it holds that for all $x \in \mathcal X$
    \begin{equation}\label{eq:nie_absolute_error}
        \left |\Phi(\frac{\mu_{F_x} - m}{\sqrt{\sigma_{F_x}^2 + s^2}}) - \Phi(\frac{\mu_{F_x} - \bar m}{\sqrt{\sigma_{F_x}^2 + \bar s^2}}) \right | \leq \epsilon + \mathcal O(\epsilon^2).
    \end{equation}
\end{restatable}

\begin{algorithm}
    \caption{\alite{}-I-S(earch)}\label{alg:NIE_I_Search}
    \begin{algorithmic}
        \Require $d, \mu_F, \sigma_F$
    \State $(q_{1}^{low}, q_{1}^{up}, q_{3}^{low}, q_{3}^{up}) \gets \text{\alite{}-I-SW}(\mu_F, \sigma_F)$
        \Comment{C: $\Theta(|\mathcal X|)$, M: $\Theta(1)$}
    \For{$1,\ldots,d$}
        \State $q_1 \gets (q_1^{up} + q_1^{low})/2$
        \State $q_3 \gets (q_3^{up} + q_3^{low})/2$
        \State $g_1 = \prod_z \Phi(\frac{q_1 - \mu_{F_z}}{\sigma_{F_z}})$ \Comment{C: $\Theta(|\mathcal X|)$, M: $\Theta(1)$}
        \State $g_3 = \prod_z \Phi(\frac{q_3 - \mu_{F_z}}{\sigma_{F_z}})$ \Comment{C: $\Theta(|\mathcal X|)$, M: $\Theta(1)$}
        \State $(q_1^{up},\ q_1^{low}) \gets \begin{cases} (q_1, q_1^{low}) & g_1 > 0.25\\
        (q_1^{up}, q_1) & \text{otherwise}\end{cases}$
        \State $(q_3^{up},\ q_3^{low}) \gets \begin{cases} (q_3, q_3^{low}) & g_3 > 0.75\\
        (q_3^{up}, q_3) & \text{otherwise}\end{cases}$
    \EndFor
    \State $(m^{up},\  m^{low}) \gets \left(\frac{q_3^{up} + q_1^{up}}{2},\ \frac{q_3^{low} + q_1^{low}}{2} \right)$
    \State $(s^{up},\ s^{low}) \gets \left(\frac{q_3^{up} - q_1^{low}}{2\Phi^{-1}(0.75)},\ \frac{q_3^{low} - q_1^{up}}{2\Phi^{-1}(0.75)}\right)$\\
    \Return $(m^{up}, m^{low}, s^{up}, s^{low})$
    \end{algorithmic}
\end{algorithm}

\begin{algorithm}
    \caption{\alite{}-I-S(earch)W(indow)}\label{alg:NIE_I_Search_Window}
    \begin{algorithmic}
        \Require $\mu_F, \sigma_F$
        \State $q_1^{low} \gets \mu_F^{min} + \sigma_F^{min} \Phi^{-1}(0.25^{1/|\mathcal X|})$
        \State $q_1^{up} \gets \mu_F^{max} + \sigma_F^{max} \Phi^{-1}(0.25^{1/|\mathcal X|})$
        \State $q_3^{low} \gets \mu_F^{min} + \sigma_F^{min} \Phi^{-1}(0.75^{1/|\mathcal X|})$
        \State $q_3^{up} \gets \mu_F^{max} + \sigma_F^{max} \Phi^{-1}(0.75^{1/|\mathcal X|})$\\
        \Return $(q_1^{low}, q_1^{up}, q_3^{low}, q_3^{up})$\\
    \end{algorithmic}
\end{algorithm}

\begin{algorithm}
    \caption{\alite{}-II-S(earch)}\label{alg:NIE_II_Search}
    \begin{algorithmic}
        \Require $d, m^{up}, m^{low}, s^{up}, s^{low}, \mu_F, \sigma_F$
        \State $(\tilde m^{up}, \tilde m^{low}) \gets (\max(m^{up}, \mu_{F}^{max}), \max(m^{low}, \mu_{F}^{max}))$
        \State $(\tilde s^{up}_x, \tilde s^{low}_x)_{x \in \mathcal X} \gets (\min(s^{up}, \sigma_{F_x}), \min(s^{low}, \sigma_{F_x}))_{x \in \mathcal X}$
        \State $(q_{x,1}^{low}, q_{x,1}^{up}, q_{x,3}^{low}, q_{x,3}^{up})_{x \in \mathcal X} \gets \text{\alite{}-II-SW}(\tilde m^{up}, \tilde m^{low}, \tilde s^{up}, \tilde s^{low}, \mu_F, \sigma_F)$\\
        \Comment{C: $\Theta(|\mathcal X|)$, M: $\Theta(|\mathcal X|)$}
    \For{$1, \ldots, d$}
        \State $(q_{x,1})_{x \in \mathcal X} \gets ((q_{x,1}^{up} + q_{x,1}^{low})/2)_{x \in \mathcal X}$ \Comment{C: $\Theta(|\mathcal X|)$, M: $\Theta(|\mathcal X|)$}
        \State $(q_{x,3})_{x \in \mathcal X} \gets ((q_{x,3}^{up} + q_{x,3}^{low})/2)_{x \in \mathcal X}$ \Comment{C: $\Theta(|\mathcal X|)$, M: $\Theta(|\mathcal X|)$}
        \State $(g_{x,1}^{up})_{x \in \mathcal X} \gets (\max(\Phi(\frac{q_{x,1} - \tilde m^{low}}{\tilde s^{up}_x}) / \Phi(\frac{q_{x,1} - \mu_{F_x}}{\sigma_{F_x}}), \Phi(\frac{q_{x,1} - \tilde m^{low}}{\tilde s^{low}_x}) / \Phi(\frac{q_{x,1} - \mu_{F_x}}{\sigma_{F_x}}))_{x \in \mathcal X}$\\
        \Comment{C: $\Theta(|\mathcal X|)$, M: $\Theta(|\mathcal X|)$}
        \State $(g_{x,1}^{low})_{x \in \mathcal X} \gets (\min(\Phi(\frac{q_{x,1} - \tilde m^{up}}{\tilde s^{up}_x}) / \Phi(\frac{q_{x,1} - \mu_{F_x}}{\sigma_{F_x}}), \Phi(\frac{q_{x,1} - \tilde m^{up}}{\tilde s^{low}_x}) / \Phi(\frac{q_{x,1} - \mu_{F_x}}{\sigma_{F_x}}))_{x \in \mathcal X}$\\
        \Comment{C: $\Theta(|\mathcal X|)$, M: $\Theta(|\mathcal X|)$}
        \State $(g_{x,3}^{up})_{x \in \mathcal X} \gets (\max(\Phi(\frac{q_{x,3} - \tilde m^{low}}{\tilde s^{up}_x}) / \Phi(\frac{q_{x,3} - \mu_{F_x}}{\sigma_{F_x}}), \Phi(\frac{q_{x,3} - \tilde m^{low}}{\tilde s^{low}_x}) / \Phi(\frac{q_{x,3} - \mu_{F_x}}{\sigma_{F_x}}))_{x \in \mathcal X}$\\
        \Comment{C: $\Theta(|\mathcal X|)$, M: $\Theta(|\mathcal X|)$}
        \State $(g_{x,3}^{low})_{x \in \mathcal X} \gets (\min(\Phi(\frac{q_{x,3} - \tilde m^{up}}{\tilde s^{up}_x}) / \Phi(\frac{q_{x,3} - \mu_{F_x}}{\sigma_{F_x}}), \Phi(\frac{q_{x,3} - \tilde m^{up}}{\tilde s^{low}_x}) / \Phi(\frac{q_{x,3} - \mu_{F_x}}{\sigma_{F_x}}))_{x \in \mathcal X}$\\
        \Comment{C: $\Theta(|\mathcal X|)$, M: $\Theta(|\mathcal X|)$}
            \State $(q_{x,1}^{up},\ q_{x,1}^{low})_{x \in \mathcal X} \gets \left(\begin{cases} (q_{x,1}, q_{x,1}^{low}) & g_{x,1}^{low} \geq 0.25\\
            (q_{x,1}^{up}, q_{x_1}) &  g_{x,1}^{up} \leq 0.25\\
            (q_{x,1}^{up}, q_{x_1}^{low}) & \text{otherwise}
            \end{cases}\right)_{x \in \mathcal X}$\\
            \Comment{C: $\Theta(|\mathcal X|)$, M: $\Theta(|\mathcal X|)$}
            \State $(q_{x,3}^{up}, q_{x,3}^{low})_{x \in \mathcal X} \gets \left(\begin{cases} (q_{x,3}, q_{x,3}^{low}) & g_{x,3}^{low} \geq 0.75\\
            (q_{x,3}^{up}, q_{x,3}) & g_{x,3}^{up} \leq 0.75\\
            (q_{x,3}^{up}, q_{x,3}^{low}) & \text{otherwise}
            \end{cases}\right)_{x \in \mathcal X}$\\
            \Comment{C: $\Theta(|\mathcal X|)$, M: $\Theta(|\mathcal X|)$}
    \EndFor
    \State $(m_x^{up},\ m_x^{low})_{x \in \mathcal X} \gets \left(\frac{q_{x,3}^{up} + q_{x,1}^{up}}{2},\ \frac{q_{x,3}^{low} + q_{x,1}^{low}}{2}\right)_{x \in \mathcal X}$
    \State $(s_x^{up},\ s_x^{low})_{x \in \mathcal X} \gets \left(\frac{q_{x,3}^{up} - q_{x,1}^{low}}{2\Phi^{-1}(0.75)},\ \frac{q_{x,3}^{low} - q_{x,1}^{up}}{2\Phi^{-1}(0.75)}\right)_{x \in \mathcal X}$\\
    \Return $(m_x^{up}, m_x^{low}, s_x^{up}, s_x^{low})_{x \in \mathcal X}$
    \end{algorithmic}
\end{algorithm}

\begin{algorithm}
    \caption{\alite{}-II-S(earch)W(indow)}\label{alg:NIE_II_Search_Window}
    \begin{algorithmic}
        \Require $\tilde m^{up}, \tilde m^{low}, \tilde s^{up}, \tilde s^{low}, \mu_F, \sigma_F$
        \State $(q_{x,1}^{low})_{x \in \mathcal X} \gets (\min(\mu_{F_x} - \sqrt{2} \sigma_{F_x}, \max(\frac{\tilde m^{low} + \mu_{F_x}}{2} - \frac{\sigma_{F_x}^2 \ln(2/0.25)}{\tilde m^{low} - \mu_{F_x}}, \tilde m^{low} - \sqrt{\frac{2 \ln(2/0.25)}{1-(\tilde s^{up}_x / \sigma_{F_x})^2}} \tilde s^{up}_x)))_{x \in \mathcal X}$
        \State $(q_{x,1}^{up})_{x \in \mathcal X} \gets \tilde m^{up} + \Phi^{-1}(0.25) \tilde s_x^{low}$
        \State $(q_{x,3}^{low})_{x \in \mathcal X} \gets (\min(\mu_{F_x} - \sqrt{2} \sigma_{F_x}, \max(\frac{\tilde m^{low} + \mu_{F_x}}{2} - \frac{\sigma_{F_x}^2 \ln(2/0.75)}{\tilde m^{low} - \mu_{F_x}}, \tilde m^{low} - \sqrt{\frac{2 \ln(2/0.75)}{1-(\tilde s^{up}_x / \sigma_{F_x})^2}} \tilde s^{up}_x)))_{x \in \mathcal X}$
        \State $(q_{x,3}^{up})_{x \in \mathcal X} \gets \tilde m^{up} + \Phi^{-1}(0.75) \tilde s_x^{up}$\\
        \Return $(q_1^{low}, q_1^{up}, q_3^{low}, q_3^{up})$
    \end{algorithmic}
\end{algorithm}

\clearpage
\section{EXPERIMENTAL DETAILS}

\subsection{\cref{{fig:quadcopter_recall}}}\label{app:fig1_detail}

As a realistic posterior distribution over a large set of candidates, we use the posterior distribution of the final iteration of our quadcopter experiment~(see Appendix~\ref{sec:quadcopter_details} for details).
This posterior distribution over a $4$-dimensional space of feedback control parameters captures our current estimate of the quadcopters' performance under any of those feedback parameters.
Our goal is to select a small set of feedback parameters for final testing that contain the best-performing feedback parameters with high probability.
\Cref{fig:quadcopter_recall} shows that our PoM estimator outperforms previous methods for recall-optimal candidate selection and Figure~\ref{fig:quadcopter_recall_detailed} quantifies the impact of using a faithful PoM estimator over a less-faithful one for this task. We report on expected recall and its standard error.

\begin{figure}[h]
    \centering
         \begin{tabular}{c|c}
              & \textbf{Area Under Curve}\\
             \hline
             \ts{} & $83.7 \pm 3.4$ \% \\
             \ias{} & $83.2 \pm 3.5$ \%\\
             \alite{} & $83.3 \pm 3.5$ \%\\
             \flite{} & $82.9 \pm 3.4$ \%\\
             \est{} & $82.3 \pm 3.4$ \%\\
             \vapor{} & $81.4 \pm 3.3$ \%\\
             \tss{} & $75.0 \pm 2.8$ \%\\
             \ms{} & $67.5 \pm 4.1$ \%\\
         \end{tabular}
         \caption{Faithful estimators to PoM perform marginally better than less faithful ones. As such, the computational efficiency improvements achieved in this work dominate the drop in sample efficiency compared to \ts{}. In contrast, the performance drop to heuristics such as \tss{} and \ms{} is much more pronounced.}

         \label{fig:quadcopter_recall_detailed}
\end{figure}

\subsection{Figure~\ref{fig:entropy_search}}\label{sec:describing_fig_entropy_search}

We sample the objective function $f_{true}$ from a centered Gaussian process on the line segment $[0,1]$ with a squared exponential kernel (length scale $0.02$, amplitude $1.0$). We assume an observation model with independent homoscedastic centered additive Gaussian noise where $\sigma_{noise} = 0.2$. We run calibrated Entropy Search for $50$ steps after evenly discretizing the domain to $|\mathcal X|=250$ points. The experiment is repeated $10$ times and we report on the mean and standard error of the entropy of PoM, which is the objective that Entropy Search seeks to minimize. We use five samples to condition on hypothetical observations. For each conditioning, the PoM entropy reduction is estimated either with \flite{} for convergence parameter $\epsilon = 1/(10 | \mathcal X |)$, or using \ts{}. Running \ts{} to convergence would lead to an exploding runtime, so we always use the fixed budget of $4$ samples. Note that we cannot decrease the cost much further, since for a single Monte Carlo sample the entropy would always degenerate to 0. Even so, on an NVIDIA TITAN RTX GPU a full run of Entropy Search using \lite{} takes just $15.4$ seconds, whereas using the \ts{} backend, it takes $8.2$ minutes. This difference becomes much more pronounced as the size of the Gaussian reward vector is increased.

\subsection{Estimating the State of Large-Scale Bayesian Optimization}\label{sec:describing_fig_large_scale_exp}

\begin{figure}
    \centering
    \incplt[0.4\linewidth]{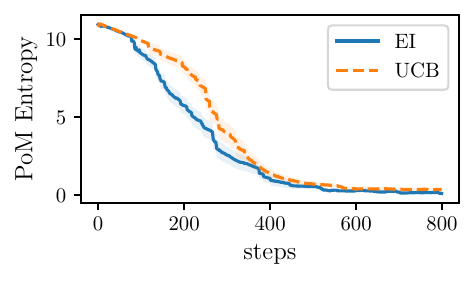}
    \vspace{-10pt}
    \caption{PoM entropy estimation with \lite{} allows tracking the state of large-scale Bayesian optimization. On an NVIDIA A100 GPU, \lite{} reduces time of computation from more than $21$ days to $30$ seconds.}
    \label{fig:large_scale_exp}
\end{figure}

To demonstrate that \lite{} can truly be scaled to large-scale industrial settings, we run uncalibrated Bayesian optimization with a linear kernel for $800$ steps using both the GP-upper confidence bound~\parencite{srinivas2009gaussian} (UCB) and the expected improvement (EI) acquisition function. Here, the ground-truth objective function is described by a (random) hyperplane in $1'000$ dimensions, sampled at $10'000$ points on the unit-sphere. A comparison between \lite{} and a ground-truth surrogate for PoM such as \ts{} is not possible here: even estimation under the \ia{} would require $500$ hours ($21$ days) on an NVIDIA A100 GPU to compute PoMs across the BO-path for a single seed. In contrast, \lite{} only takes a few seconds (about $30$ seconds), i.e., it is about $60'000$ times faster. Our results confirm that \lite{} can be used to interpret the state of convergence of Bayesian optimization and to compare competing optimization schemes in terms of their information-theoretic performance, particularly in large-scale settings where standard approaches fail.

\subsection{Figure~\ref{fig:runtime_over_domain_size}}\label{sec:describing_fig_runtime_over_domain_size}
We densely discretise the drop-wave function $f_{true}(x_1, x_2) := (1 + \cos(12 \sqrt{x_1^2 + x_2^2}))/((x_1^2 + x_2^2) / 2 + 2)$ on the rectangle $[-5, 4]^2$ using a grid with $300^2 = 90'000$ nodes. To obtain different domain sizes, we subsample the grid uniformly at random (without repetition). Next, we run Bayesian optimisation using the \textit{expected improvement} (over best observation) acquisition function. The posterior is derived based on a Gaussian process prior fitted at each step with marginal likelihood maximisation (we fit the length scale and amplitude of a Matern $5/2$ kernel, the constant mean function, and $\sigma_{noise}$). To jump start the kernel selection, we make $50$ random observations prior to starting Bayesian optimisation. We assume additive centred Gaussian noise with $\sigma_{noise} = 0.1$. We report on the mean and standard deviation of the runtime averaged across $100$ steps of Bayesian optimisation for $5$ seeds. All estimators use $\alpha = 1$. We cancel runs exceeding a computational budget of $6$ hours ($216$ seconds per step), which is why \ts{} and \est{} do not have values at all time steps.

\subsection{Figure~\ref{fig:tv_distance_for_BO_of_GP_sample}}\label{sec:describing_fig_tv_distance_for_BO_of_GP_sample}
\(f_{true}\) is sampled from a centred Gaussian process \(\mathcal{GP}\) with squared-exponential kernel (length scale $0.005$, amplitude $1.0$) on the interval $[0,1]$ discretised with $|\mathcal X| = 300$ points. The prior belief over $f_{true}$ coincides with \(\mathcal{GP}\). A Bayesian optimisation scheme according to Thompson sampling is run for $200$ steps with observations $Y_x = f_{true}(x) + \varepsilon$ for i.i.d. $\varepsilon \sim \mathcal N(0, 0.1^2)$. All estimators are ensured to converge to within $\epsilon = 1/(10 \cdot | \mathcal X|)$ of their analytical expressions. We report on the mean and standard error of TV-distance to the ground-truth PoM (estimated using \ts{}) based on $50$ different seeds of optimisation. Figure~\ref{fig:poo_distance_squared_exponential_setup} illustrates the setup along with a possible set of estimated PoMs.

\begin{figure}[ht]
    \centering
    \begin{subfigure}{0.48\linewidth}
        \centering
        \includegraphics[width=\linewidth]{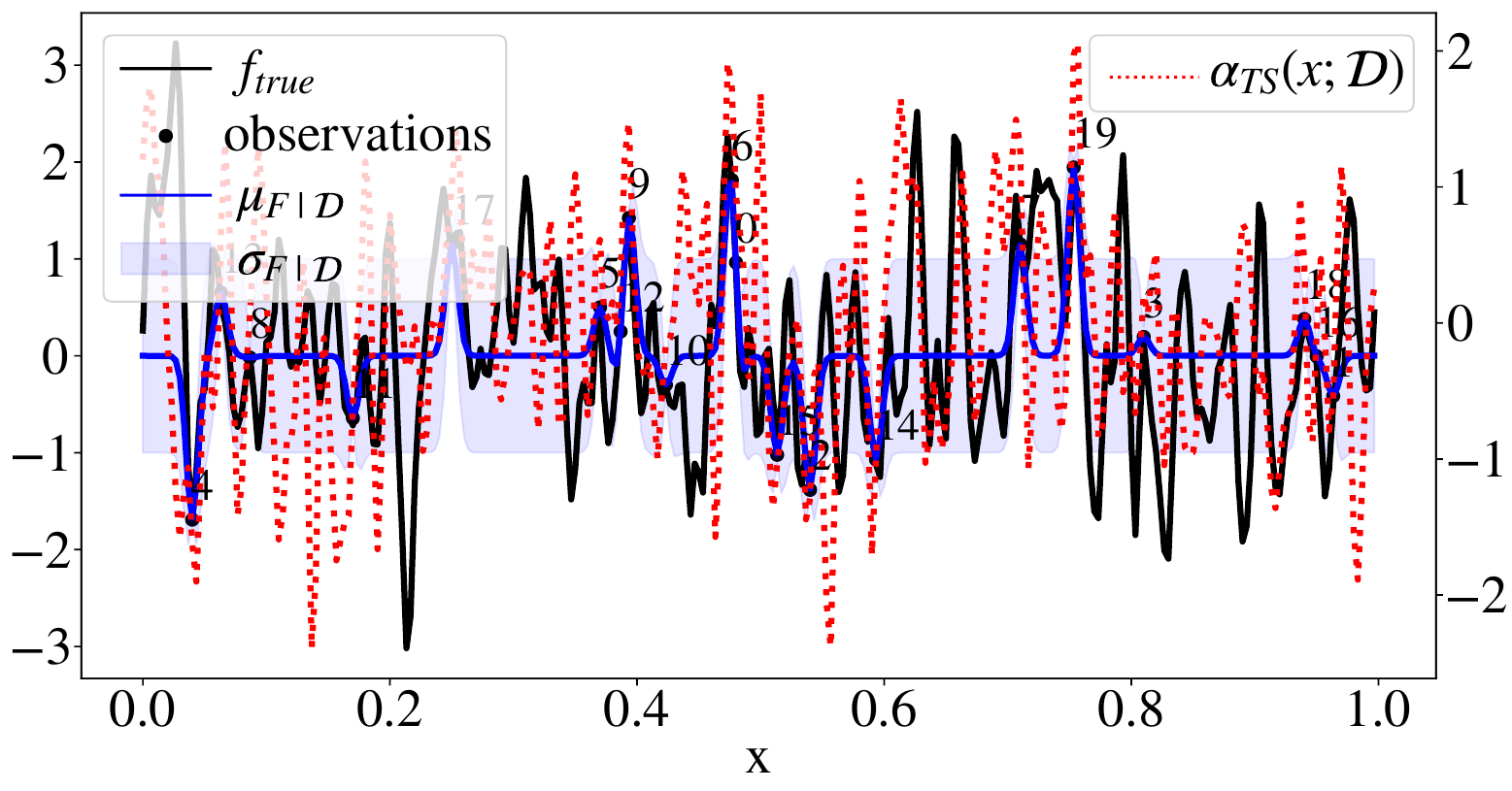}
        \caption{Example $f_{true}$ with associated $p(f|\mathcal D)$ and $\alpha_{TS}(x;\mathcal D)$ after $20$ queries to $f_{true}$.}
    \end{subfigure}\hfill
    \begin{subfigure}{0.48\linewidth}
        \centering
        \includegraphics[width=\linewidth]{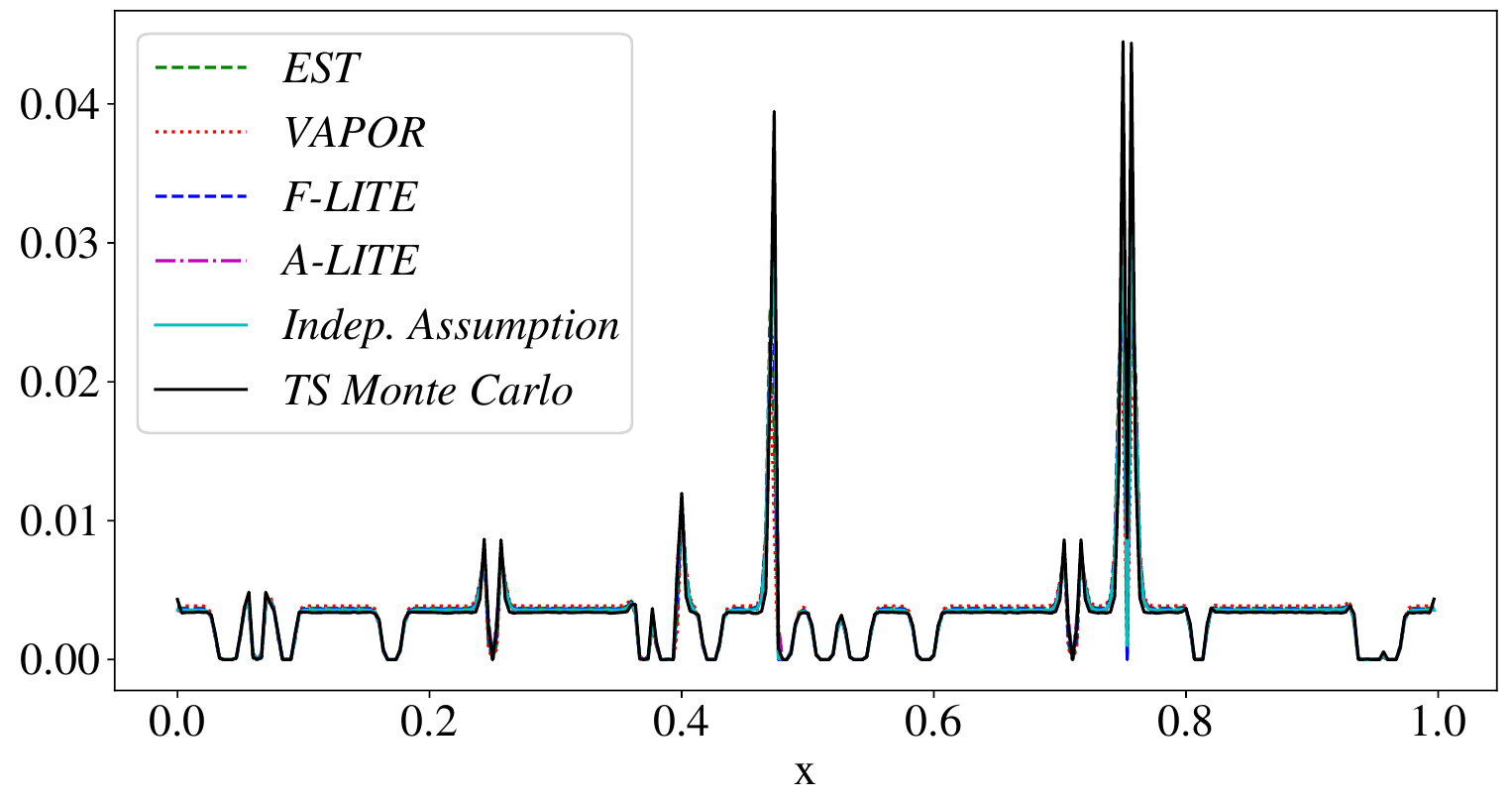}
        \caption{$\mathbb P[x \in X^* | \mathcal D]$ according to different PoM estimators after $20$ queries to example $f_{true}$}
    \end{subfigure}
    \caption{Illustration of the setup for Figure~\ref{fig:tv_distance_for_BO_of_GP_sample}.}
    \label{fig:poo_distance_squared_exponential_setup}
\end{figure}

\subsection{Figure~\ref{fig:tv_distance_quadcopter}}\label{sec:quadcopter_details}

Based on a simulator of the dynamics of a quadcopter, we are able to measure how close the quadcopter got to stabilisation at a target position when starting at a separate fixed location.
We use the same experimental setup as \parencite{hübotter2024transductiveactivelearningtheory}, with the quadcopter simulation of \parencite{quadcopter_sim2023github}.
The quadcopter is steered through a controller with $8$ degrees of freedom, which describe the unknown perturbation to the system. The task is to use Bayesian optimisation to identify the disturbance parameters through feedback from the simulator (with additive centred Gaussian noise at a standard deviation of $\sigma_{noise} = 0.1$). The unknown perturbation is sampled element-wise according to a $\chi^2$-distribution, resulting in a distribution over $f_{true}$. Due to $4$ degrees of freedom removed using a heuristic, Bayesian optimisation must be performed in $4$-dimensional space. To obtain a tractably finite domain, we sample $400$ discrete points uniformly at random in the hypercube $[0,20]^4$. We run Bayesian optimisation for $175$ steps using the \textit{expected improvement} (over best observation) acquisition function. The posterior is derived based on a Gaussian process prior fitted at each step with marginal likelihood maximisation (we fit the length scale and amplitude of a Matern $5/2$ kernel, the constant mean function, and $\sigma_{noise}$). To jump start the kernel selection, we make $25$ random observations prior to starting Bayesian optimisation. We also leave out the first $5$ steps of Bayesian optimisation (warmup steps), during which the estimation of the parameters of the Gaussian process prior are highly volatile. All reported PoM estimators are run to $\epsilon$-convergence for $\epsilon = 1/{|\mathcal X|}$. The ground-truth is estimated using \ts{} with $\epsilon = 1/(10 \cdot | \mathcal X|)$.

\subsection{Figure~\ref{fig:accuracy_over_runtime}}\label{sec:drop-wave_during_accuracy_runtime_details}
We coarsely discretise the drop-wave function $f_{true}(x_1, x_2) := (1 + \cos(12 \sqrt{x_1^2 + x_2^2}))/((x_1^2 + x_2^2) / 2 + 2)$ on the rectangle $[-2.5, 2]^2$ using a grid with $25^2 = 625$ nodes. We run Bayesian optimisation for $30$ steps using the \textit{expected improvement} (over best observation) acquisition function. The posterior is derived based on a Gaussian process prior fitted at each step with marginal likelihood maximisation (we fit the length scale and amplitude of a Matern $5/2$ kernel, the constant mean function, and $\sigma_{noise}$). To jump start the kernel selection, we make $10$ random observations prior to starting Bayesian optimisation. We assume additive centred Gaussian noise with $\sigma_{noise} = 0.1$.
Figure~\ref{fig:accuracy_runtime_operint_points_problem_setting} shows the posteriors and probabilities of maximality belonging to step $10$ of Bayesian optimisation at seed $0$.

\begin{figure}[ht]
    \centering
    \includegraphics[width=\linewidth]{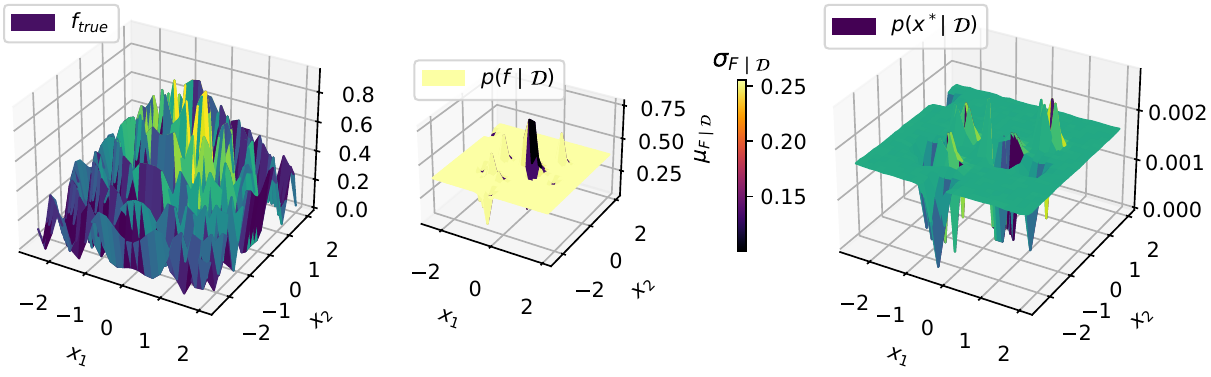}
    \caption{Problem setting for the accuracy/runtime operating points plot in Figure~\ref{fig:accuracy_over_runtime}}
    \label{fig:accuracy_runtime_operint_points_problem_setting}
\end{figure}

We report on the mean and standard error of the runtime and TV-distance averaged across steps $11-30$ of Bayesian optimisation for $5$ different seeds (the first $10$ warm-up steps are removed to obtain a more decisive picture). To evaluate the estimators under different convergence requirements, $\alpha = 1/(\epsilon \cdot | \mathcal X|)$ is swept through $\{0.01, 0.03, 0.1, 0.3, 1.0, 3.0, 10.0\}$. The ground-truth is estimated using \ts{} with $\alpha = 10.0$, which runs in $915$ seconds (per optimisation step) on an NVIDIA TITAN RTX GPU.%

\subsection{Total Variation Distance}\label{sec:tv_distance}
The total variation distance $d_{TV}$ is the central fidelity metric in our experiments. Formally, it is defined as
\begin{defn}[Total variation distance]
    Let $P,Q$ be probability distributions over a measurable space \((\Omega, \mathcal E)\). Then the total variation distance between \(P\) and \(Q\) is defined as
    \begin{equation*}
        d_{TV}(P, Q) := \sup_{\mathcal A \in \mathcal E} | P(A) - Q(A) |.
    \end{equation*}
\end{defn}
Alternatively, it corresponds to the metric derived from the $L^1$ norm over the space of probability mass functions:
 \begin{restatable}[Total variation distance as $L^1$-norm induced metric]{prop}{tvDistanceVsLOneNorm}\label{prop:total_variation_distance_vs_l1_norm}
     Let $P,Q$ be probability measures over a measurable space \((\Omega, \mathcal E)\) and \(\mu\) a \(\sigma\)-finite measure over \((\Omega, \mathcal E)\) s.t. \(P,Q \ll \mu\). Then \(d_{TV}\) can be characterised by
     \begin{equation*}
         d_{TV}(P,Q) = \frac{1}{2}\left\| \frac{dP}{d\mu} - \frac{dQ}{d\mu} \right\|_{L^1(\Omega, \mathcal E, \mu)},
     \end{equation*}
     where \(\tfrac{dP}{d\mu}\) and \(\tfrac{dQ}{d\mu}\) denote Radon-Nykodym derivatives of \(P\) and \(Q\) with respect to the base measure \(\mu\). Important cases are when \(\mu\) is the Lebesgue measure or when it is the counting measure leading to a formulation for probability density functions and probability mass functions, respectively.
 \end{restatable}

\section{ADDITIONAL EXPERIMENTS}\label{sec:additional_experiments}

\subsection{Alternative Synthetic Experiments}\label{sec:fidelity_under_synthetic_posteriors_alternative_experiments}
To add to the results presented in Figure~\ref{fig:estimators_tv_distance_to_independent_ground_truth}, we sample $\mu_F$ and $\sigma_F$ according to other distributions. Figure~\ref{fig:comparison_of_fidelity_of_estimators} reports the TV-distance between the estimated PoM and a ground-truth according to the \ia{}. As in Figure~\ref{fig:estimators_tv_distance_to_independent_ground_truth}, $\mu_{F_x}$ and $\sigma_{F_x}$ are sampled i.i.d. across $x \in \mathcal X$. All estimators are ensured to converge to within $\epsilon = 1/(200 \cdot | \mathcal X|)$. The experiments are repeated across $20$ seeds to report the mean and standard error. Notice how \alite{} and \flite{} consistently outperform \est{} and \vapor{} across a variety of $\mu_F$ and $\sigma_F$.

\begin{figure}[ht]
    \centering
    \begin{subfigure}{.285\linewidth}
        \centering
        \incplt[\linewidth]{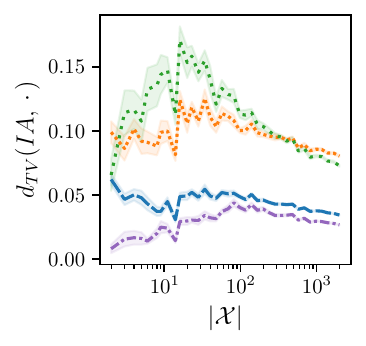}
        \caption{$\mu_{F_x} \sim \mathcal U(0, 5),\ \sigma_{F_x} \sim \mathcal U(\tfrac{1}{2}, 2)$}
    \end{subfigure}
    \begin{subfigure}{.27\linewidth}
        \centering
        \incplt[\linewidth]{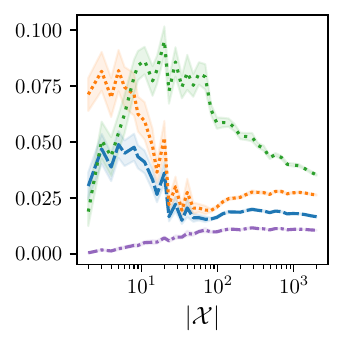}
        \caption{$\mu_{F_x} \sim \mathcal U(0, 5),\ \sigma_{F_x} = \tfrac{1}{2}$}
    \end{subfigure}
    \begin{subfigure}{.27\linewidth}
        \incplt[\linewidth]{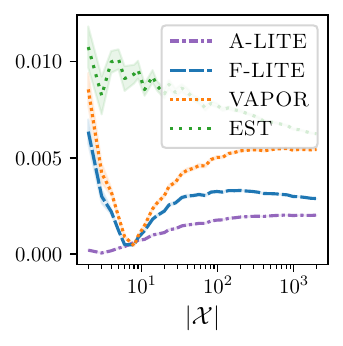}%
        \caption{$\mu_{F_x} \sim \mathcal U(0, \tfrac{1}{10}),\ \sigma_{F_x} = \tfrac{1}{2}$}
    \end{subfigure}
    \caption{TV-distance under alternative synthetic posteriors. As in the main text, \lite{} significantly outperforms competing methods from the literature.}
    \label{fig:comparison_of_fidelity_of_estimators}
\end{figure}

\subsection{$f_{true}$ Sampled from Alternative Gaussian Process}\label{sec:sampling_from_2d_laplacian_GP}
Instead of the one-dimensional Gaussian process with squared exponential kernel that was prominently featured in Figures~\ref{fig:tv_distance_for_BO_of_GP_sample} with a detailed description in Section~\ref{sec:describing_fig_tv_distance_for_BO_of_GP_sample}, we may instead use a two-dimensional Gaussian process with exponential kernel. Accordingly, we sample the test function \(f_{true}\) from a centred Gaussian process \(\mathcal{GP}\) with exponential kernel (length scale $0.1$, amplitude $1.0$) on $[0,1]^2$ discretised to $|\mathcal X| = 400$ points. To ensure calibrated Bayesian optimisation, the prior belief over $f_{true}$ coincides with \(\mathcal{GP}\). We run Bayesian optimisation based on Thompson sampling, where the observations are generated as $Y_x = f_{true}(x) + \varepsilon$ for i.i.d. $\varepsilon \sim \mathcal N(0, 0.1^2)$. Figure~\ref{fig:poo_distances_laplacian_setting} illustrates the setup.

\begin{figure}[ht]
    \centering
    \begin{subfigure}{0.24\linewidth}
        \centering
        \includegraphics[width=\linewidth]{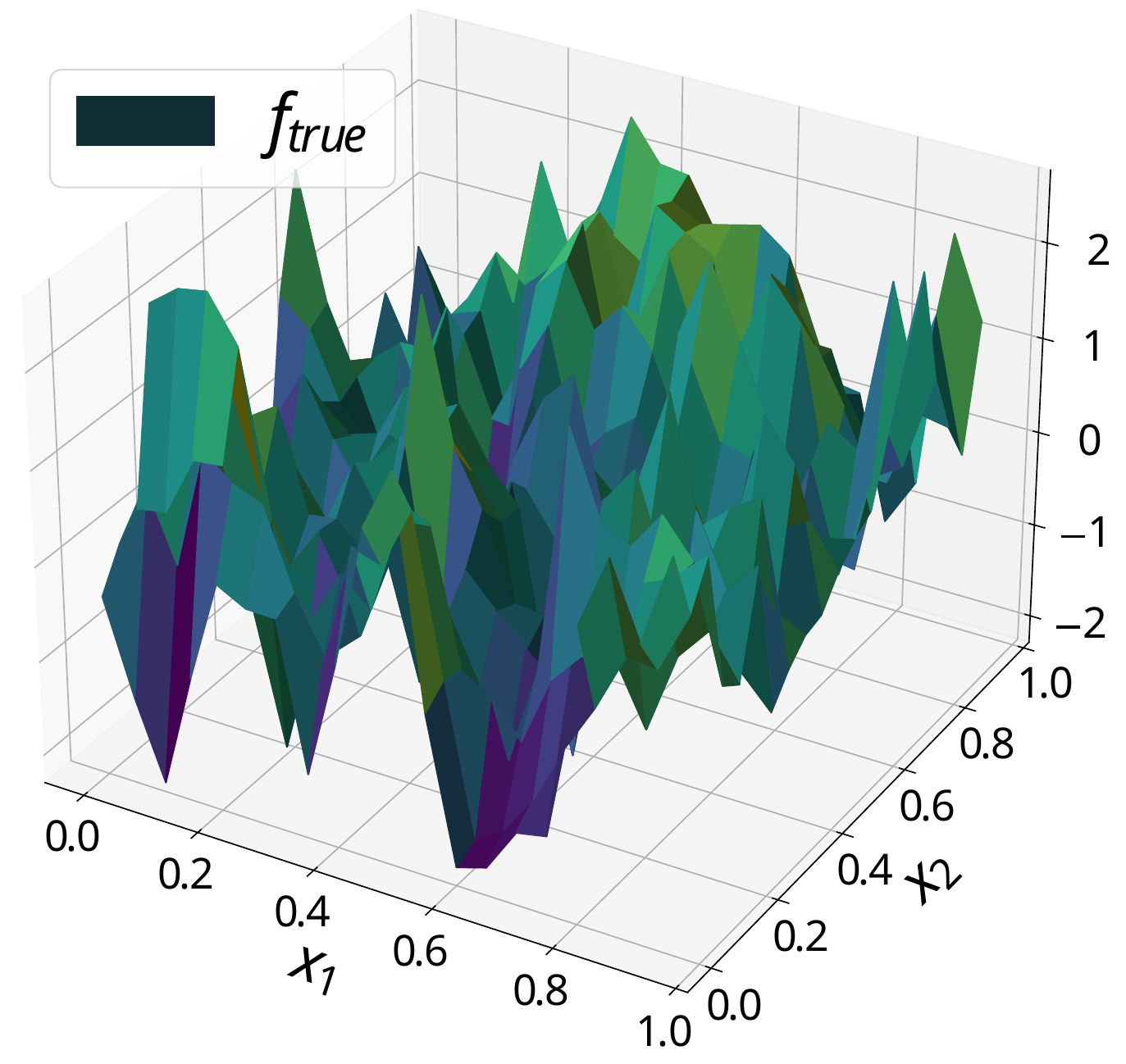}
        \caption{Example $f_{true}$, a sample from $\mathcal{GP}$.}
    \end{subfigure}
    \begin{subfigure}{0.24\linewidth}
        \centering
        \includegraphics[width=\linewidth]{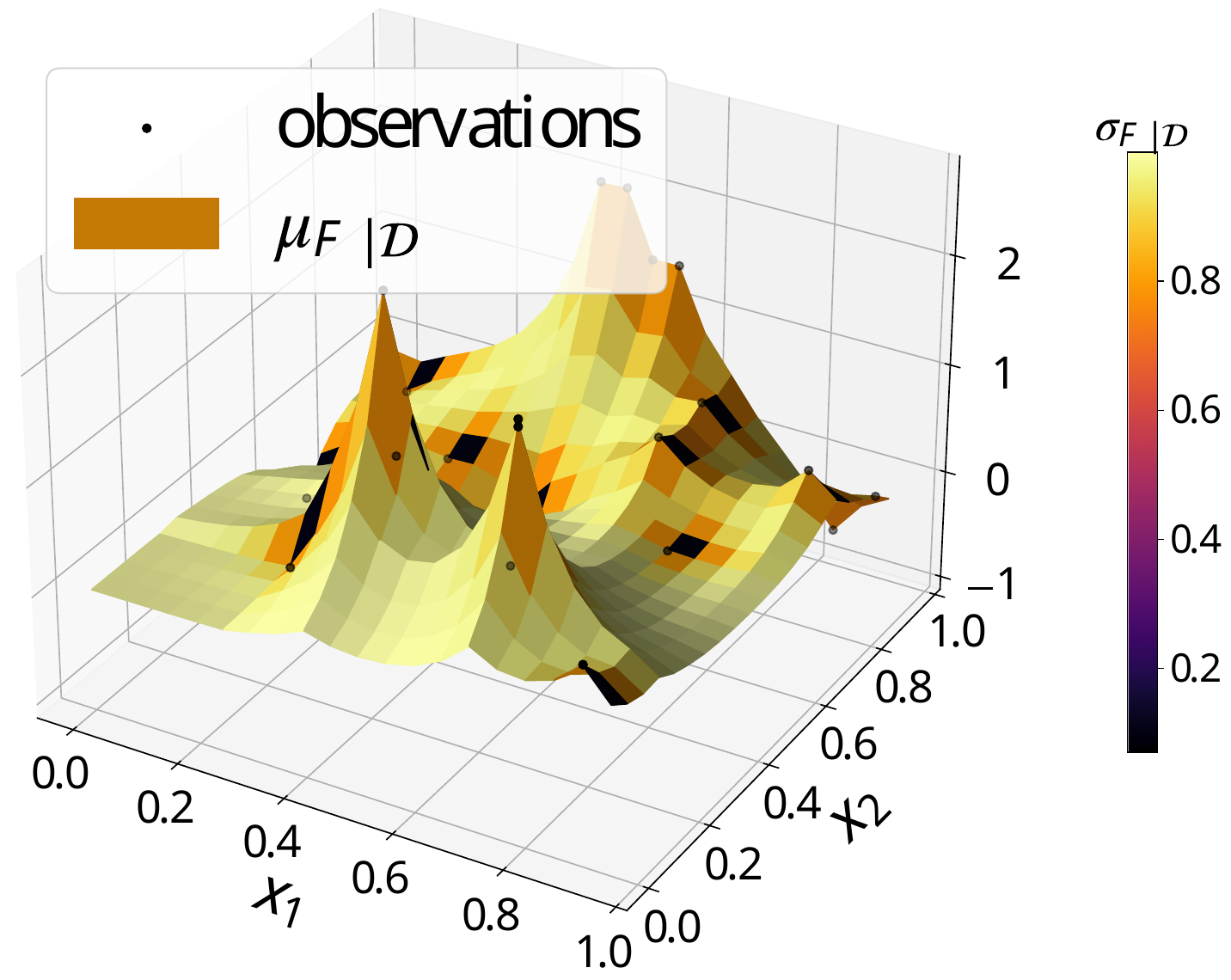}
        \caption{$p(f|\mathcal D)$ after\\ $20$ queries to $f_{true}$.}
    \end{subfigure}
    \begin{subfigure}{0.24\linewidth}
        \centering
        \includegraphics[width=\linewidth]{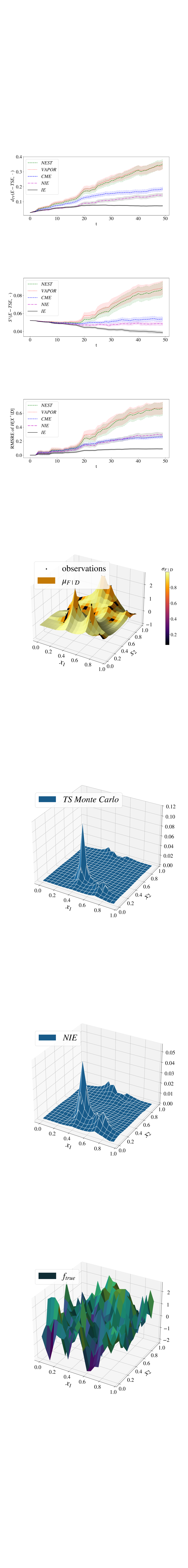}
        \caption{\ts{} after\\ $20$ queries to $f_{true}$.}
    \end{subfigure}\hfill
    \begin{subfigure}{0.24\linewidth}
        \centering
        \includegraphics[width=\linewidth]{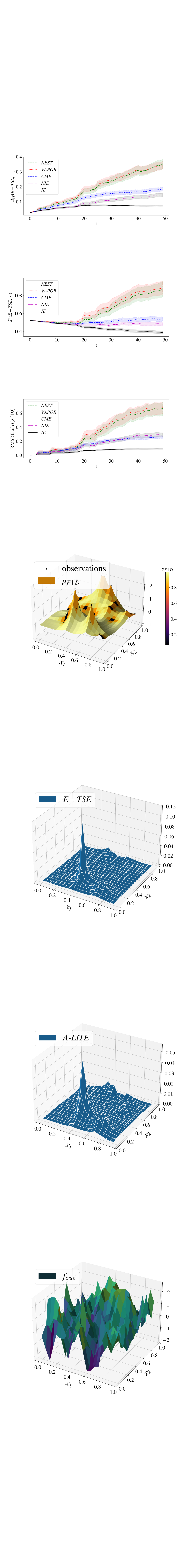}
        \caption{\alite{} after\\ $20$ queries to $f_{true}$}
    \end{subfigure}

    \caption{Illustration of the setup for Bayesian optimisation with $f_{true}$ sampled from $2$-dimensional Gaussian process with exponential kernel.}
    \label{fig:poo_distances_laplacian_setting}
\end{figure}

Figure~\ref{fig:poo_distances_laplacian} reports on the accuracy of the PoM estimators during Bayesian optimisation. We ensure convergence of all estimators to within $\epsilon = 1/({10 \cdot | \mathcal X|})$ of their analytical expressions, including \ts{}, which is used as a ground-truth. To derive the mean and standard error at each step we use $50$ different seeds of Bayesian optimisation.

\begin{figure}[ht]
\begin{subfigure}{.48\linewidth}
        \centering
        \incplt[0.82\linewidth]{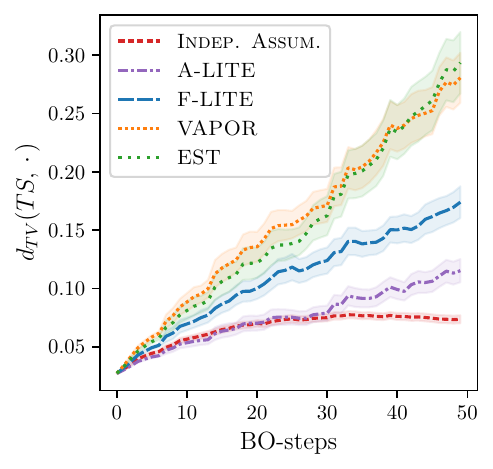}
        \caption{TV-distance.}
    \end{subfigure}
    \begin{subfigure}{.48\linewidth}
\incplt[0.82\linewidth]{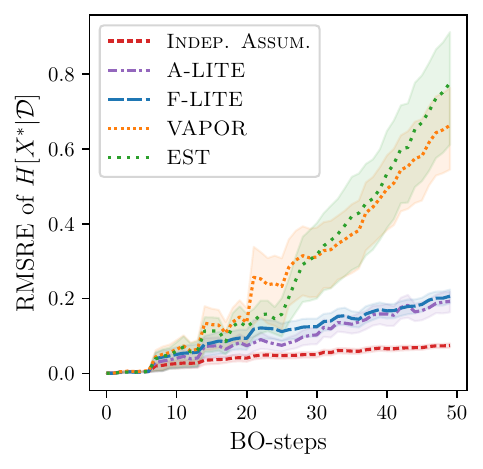}
    \caption{Root mean squared relative error of entropy estimation.}
    \end{subfigure}
    \caption{Fidelity of PoM estimates during Bayesian optimisation with $f_{true}$ sampled from $2$-dimensional Gaussian process with exponential kernel.}
    \label{fig:poo_distances_laplacian}
\end{figure}

\subsection{Drop-Wave}\label{sec:drop-wave_during_sampling}
While the drop-wave function $f_{true}(x_1, x_2) := (1 + \cos(12 \sqrt{x_1^2 + x_2^2}))/((x_1^2 + x_2^2) / 2 + 2)$ is featured in the main text, there we do not report on the evolution during Bayesian optimisation of the PoM fidelity and relative error of entropy estimation. Recall the setting in Section~\ref{sec:drop-wave_during_accuracy_runtime_details}, but now running Bayesian optimisation for $100$ steps instead of $30$. Then Figure~\ref{fig:drop-wave-tv-distance} reports the mean and standard error of the TV-distance to ground-truth PoM during $50$ seeds of Bayesian optimisation. Here, we exclude the first $10$ steps of Bayesian optimisation (warmup steps) and all estimators, including \ts{} for the ground-truth, are ensured to converge to within $\epsilon = 1 / | \mathcal X|$ of their analytical expression

\begin{figure}[ht]
    \centering
    \begin{subfigure}{0.38\linewidth}
        \incplt[\linewidth]{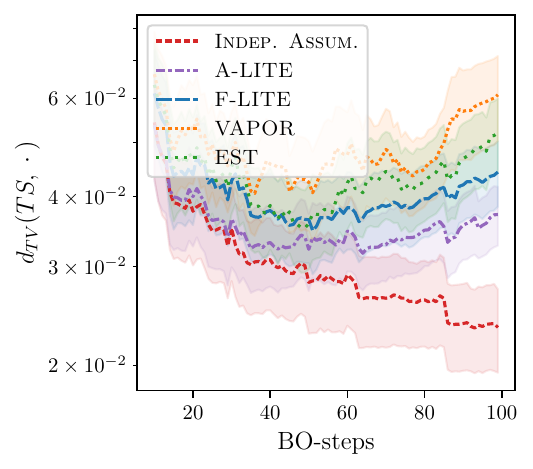}
        \subcaption[]{TV-distance}
        \label{fig:drop-wave-tv-distance}
    \end{subfigure}\hspace{20pt}
    \begin{subfigure}{0.36\linewidth}
        \incplt[\linewidth]{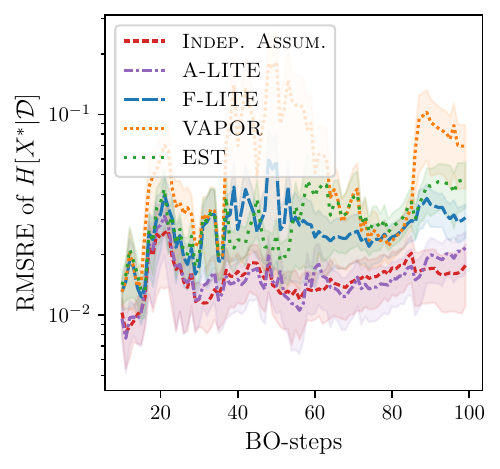}
        \caption{RMSRE of entropy estimation}
    \label{fig:drop-wave-entropy}
    \end{subfigure}
    \caption{Fidelity of PoM estimates during Bayesian optimisation with $f_{true}$ set to drop-wave.}
\end{figure}

Likewise, Figure~\ref{fig:drop-wave-entropy} reports on the mean and standard error of the root mean squared relative error of entropy estimation based on $50$ repetitions of Bayesian optimisation. Still, convergence of all estimators to within $\epsilon = 1/{|\mathcal X|}$ of their analytical expression is ensured, including the estimator for ground-truth (based on \ts{}).

\section{PROOFS}\label{sec:proofs}

\subsection{Assumptions}

\uniqueMaximiser*
\begin{proof}
    Being an assumption, we only have to show the equivalence between almost sure uniqueness and integration to $1$. Denote by \(\mathcal E = \{|\arg\max_{x \in \mathcal X} F_x| = 1\}\). Then
    \begin{align*}
        \sum_{x \in \mathcal X} \mathbb P[x \in X^*] & = \sum_{x \in \mathcal X} \mathbb P[\{x\}=X^*, \mathcal E] + \sum_{x \in \mathcal X} \mathbb P[x \in X^*, \mathcal E^c] \mathbb P[\mathcal E^c ]\nonumber\\
        & = \mathbb P[\mathcal E ] + (1-\mathbb P[\mathcal E ])\sum_{x \in \mathcal X} \mathbb P[x \in X^*, \mathcal E^c]\nonumber\\
        & \geq \mathbb P[\mathcal E ] + 2(1-\mathbb P[\mathcal E ]) = 2 - \mathbb P[\mathcal E ]
    \end{align*}
    with equality if \(\mathbb P[\mathcal E ] = 1\).
\end{proof}

\subsection{Propositions}
\setcounter{prop}{0}
\begin{restatable}{prop}{computingTheIndependenceEstimatorForGaussianProcessesSimplified}\label{cor:computing_the_independence_estimator_for_Gaussian_processes_simplified}
    Let $\tilde F \sim \mathcal N(\mu_F, \mathrm{diag}(\sigma_1^2, \ldots, \sigma_{|\mathcal X|}^2))$, let $\epsilon \in (0, 1/4]$, and define $\tilde \epsilon:= -\Phi^{-1}(2\epsilon)$. Then for
    $$n = \left\lceil \frac{\mu_{F}^{\max} - \mu_{F}^{\min} + 2\;\! \tilde \epsilon\;\! \sigma_{F}^{\max}}{\epsilon \cdot 2 \sqrt{2 \pi} \sigma_{F}^{\min}}\right \rceil + 2 \in \Theta(\sqrt{\log (1/\epsilon)} / \epsilon)$$
    integration points at positions $f_0 = -\infty$, $f_n = \infty$, and
    $f_i = \mu_{F}^{\min} - \tilde \epsilon\;\! \sigma_{F}^{\max}
        + \frac{i-1}{n-2} \left( \mu_{F}^{\max} - \mu_{F}^{\min} + 2\;\! \tilde \epsilon\;\! \sigma_{F}^{\max}\right)$
    for $0 < i < n$, it holds for all $x \in \mathcal X$ that
    $$\Big| \tilde p_x - \sum_{i=0}^{n-1} \frac{g^x(f_{i+1}) + g^x(f_i)}{2} \mathbb P[\tilde F_x \!\in\! (f_i, f_{i+1}]] \Big| \! \leq\! \epsilon.$$
\end{restatable}

\begin{proof}
    The proposition follows from Proposition~\ref{thm:computing_the_independence_estimator} by refining the conditions
    \begin{equation*}
        \max_{x \in \mathcal X}\mathbb P[\tilde F_x \leq f_1] \leq 2 \epsilon, \qquad
        \max_{x \in \mathcal X}\mathbb P[\tilde F_x > f_{l-1}] \leq 2 \epsilon, \qquad \text{and} \qquad
        \max_{x \in \mathcal X} \mathbb P[\tilde F_x \in (f_i, f_{i+1}]] \leq 2 \epsilon \quad \forall i=1, \ldots, l-2
    \end{equation*}
    for the Gaussian case (with \(\epsilon \leq 1/4\)) to the stronger assumptions
    \begin{equation*}
        f_1 \leq \mu_{F}^{min} + \Phi^{-1}(2 \epsilon) \sigma_{F}^{max}, \qquad
        f_{l-1} \geq \mu_{F}^{max} - \Phi^{-1}(2 \epsilon) \sigma_{F}^{max}, \qquad \text{and} \qquad
        f_{i+1} - f_i \leq 2 \sqrt{2 \pi} \sigma_{F}^{min} \epsilon \quad \forall i=1, \ldots, l-2.
    \end{equation*}
    To satisfy these assumptions, we select equidistantly placed \(f_1, \ldots, f_{l-1}\):
    \begin{equation*}
        f_i = \mu_{F}^{min} + \Phi^{-1}(2 \epsilon) \sigma_{F}^{max} + \frac{i-1}{l-2} \left(\mu_{F}^{max} -\mu_{F}^{min} - 2 \Phi^{-1}(2 \epsilon) \sigma_{F}^{max}\right)\ \forall i=1, \ldots, l-1,
    \end{equation*}
    where we ensure sufficiently small steps \(f_{i+1} - f_i\) by taking
    \begin{equation*}
        l = \left\lceil \frac{\mu_{F}^{max} -\mu_{F}^{min} - 2 \Phi^{-1}(2 \epsilon) \sigma_{F}^{max}}{2 \sqrt{2 \pi} \sigma_{F}^{min} \epsilon}\right \rceil + 2.
    \end{equation*}
    Finally, the asymptotic scaling of $n$ follows from Lemma~\ref{lem:asymptotics_of_gaussian_cdf_and_its_inverse}.
\end{proof}

\begin{restatable}[%
]{prop}{logarithmicSearchToObtainGaussianProbabilityOfOptimality}\label{cor:logarithmic_search_to_obtain_gaussian_probability_of_optimality}
    Let $F \sim \mathcal N(\mu_F, \Sigma_F)$, $\sigma_{F_i}^2 = \Sigma_{ii} \ \forall i=1,\ldots, |\mathcal X| > 1$, and \(\kappa^* \in \mathbb R\) s.t. $
        s(\kappa^*) := \sum\nolimits_{x \in \mathcal X} \mathbb P[F_x \geq \kappa^*] = 1$. Then $s(\cdot)$ is cont. monot. decreasing and $\mu_{F}^{min} + \sigma_{F}^{min} \cdot \text{-} \Phi^{-1}(\!\tfrac{1}{|\mathcal X|}\!) \leq \kappa^* \leq \mu_{F}^{max} + \sigma_{F}^{max} \cdot \text{-} \Phi^{-1}(\!\tfrac{1}{|\mathcal X|}\!)$. The search window scales in $\Theta(\sqrt{\log |\mathcal X|})$ and $\, \forall x \in \mathcal X$
    \begin{equation*}
        |\mathbb P[F_x \geq \kappa^*] - \mathbb P[F_x \geq \kappa^k]| \leq \frac{ \mu_{F}^{max} - \mu_{F}^{min} - \Phi^{-1}(|\mathcal X|^{-1}) \sigma_{F}^{max} }{2^{k+1} \sqrt{2\pi} \sigma_{F_x}}, %
    \end{equation*}
    where $\kappa^{k}$ is the estimate at step $k$ according to Algorithm~\ref{alg:CME}. Hence, $k \!=\! \log_2( (\mu_{F}^{max} \!-\! \mu_{F}^{min} \!-\! \Phi^{-1}(|\mathcal X|^{-1}) \sigma_{F}^{max}) / (2 \epsilon))$ $ \in \Theta(\log(\log(|\mathcal X|) / \epsilon))$ steps suffice to ensure that for all $x \in \mathcal X$ it holds that $|\mathbb P[F_x \geq \kappa^*] - \mathbb P[F_x \geq \kappa^k]| \leq \epsilon$.
\end{restatable}
\begin{proof}
    This proposition follows swiftly from Lemma~\ref{lem:logarithmic_search_to_obtain_probability_of_optimality} and Lemma~\ref{lem:rate_of_convergence_of_binary_search_to_obtain_probability_of_optimality} by restricting our attention from general stochastic processes to Gaussian processes, i.e. by using \(\mathbb P[F_x \geq \kappa] = \Phi(\tfrac{\mu_{F_x} - \kappa}{\sigma_{F_x}})\) with Lipschitz constant $1/({\sqrt{2\pi} \sigma_{F_x}})$. A direct consequence of Equation~\eqref{eq:logarithmic_search_to_obtain_probability_of_optimality_suggested_conditions_constraining_relevant_region} in Lemma~\ref{lem:logarithmic_search_to_obtain_probability_of_optimality} is that \(\kappa^* \leq \mu_{F}^{max} - \Phi^{-1}(1/|\mathcal X|) \sigma_{F}^{max}\), since otherwise \(\mathbb P[F_z \geq \kappa^*] < \tfrac{1}{|\mathcal X|}\) for all \(z \in \mathcal X\). Similarly, \(\kappa^* \geq \mu_{F}^{min} - \Phi^{-1}(1/|\mathcal X|) \sigma_{F}^{min}\) since otherwise \(\mathbb P[F_z \geq \kappa^*] > \tfrac{1}{|\mathcal X|}\) for all \(z \in \mathcal X\). So, we have proven the validity of the initialisation of the logarithmic search window. The asymptotic behaviour of the search window follows immediately from Lemma~\ref{lem:asymptotics_of_gaussian_cdf_and_its_inverse}. The error bounds after running $k$ steps of binary search follow from Lemma~\ref{lem:rate_of_convergence_of_binary_search_to_obtain_probability_of_optimality} when taking into account the Lipschitz constant of the Gaussian cdf.
\end{proof}

\gradientsOfCME*
\begin{proof}
    According to the chain rule of differentiation we have
    \begin{equation*}
        \frac{d p_\theta(x)}{d\theta_i} = \frac{d \Phi(\frac{\mu_{F_x} - \kappa^*}{\sigma_{F_x}})}{d\theta_i} = \phi(\frac{\mu_{F_x} - \kappa^*}{\sigma_{F_x}})\frac{d}{d\theta_i} \frac{\mu_{F_x} - \kappa^*}{\sigma_{F_x}}.
    \end{equation*}
    Specialising $\theta_i$ to either $\mu_{F_z}$ or $\sigma_{F_z}$, we get
    \begin{align}\label{eq:cme_derivatives_first_expression}
        \frac{d p_\theta(x)}{d\mu_{F_z}} & = \phi(\frac{\mu_{F_x} - \kappa^*}{\sigma_{F_x}}) \frac{\mathds{1}_{x=z} - \frac{d\kappa^*}{d\mu_{F_z}}}{\sigma_{F_x}}\nonumber\\
        \frac{d p_\theta(x)}{d\sigma_{F_z}} & = \phi(\frac{\mu_{F_x} - \kappa^*}{\sigma_{F_x}}) \frac{-\frac{d\kappa^*}{d\sigma_{F_z}} \sigma_{F_x} + (\kappa^* - \mu_{F_x}) \mathds{1}_{x=z}}{\sigma_{F_x}^2}.
    \end{align}
    So, we are only left to find an expression for $\tfrac{d\kappa^*}{d\mu_{F_z}}$ and $\tfrac{d\kappa^*}{d\sigma_{F_z}}$. To that end, notice how $\kappa^*$ is an implicit function of $\theta = (\mu_{F}, \sigma_{F}) \in \mathbb R^{2|\mathcal X|}$. Indeed, $\kappa^*$ was defined as the unique real number (dependent on $\theta$) such that
    \begin{equation*}
        g(\theta,\kappa^*) := \sum_{x \in \mathcal X} \Phi(\frac{\mu_{F_x} - \kappa^*}{\sigma_{F_x}}) - 1 \overset{!}{=} 0,
    \end{equation*}
    where $g$ is a continuously differentiable function. We may then use the multi-variate chain rule to derive an explicit formula for $\tfrac{d \kappa^*(\theta)}{d\theta_i}$:
    \begin{gather}
        \theta \mapsto \frac{d}{d \theta_i} \overbrace{g(\theta, \kappa^*(\theta))}^{\overset{!}{=} 0} = \frac{d g(\theta,b)}{d \theta} \vert_{b = \kappa^*(\theta)} \frac{d\theta}{d\theta_i} + \frac{d g(\theta,b)}{d b}\vert_{b = \kappa^*(\theta)} \frac{d \kappa^*(\theta)}{d \theta_i} \overset{!}{\equiv} 0\nonumber\\
        \iff\nonumber\\
         \frac{d \kappa^*(\theta)}{d \theta_i} = - \frac{d g(\theta_1, \ldots, \theta_{2 |\mathcal X|},b)}{d \theta_i} \vert_{b = \kappa^*(\theta)} \big/ \frac{d g(\theta,b)}{d b}\vert_{b = \kappa^*(a)}.\label{eq:derivative_of_CME}
    \end{gather}
    Next, we evaluate Equation~\eqref{eq:derivative_of_CME}  for  $\tfrac{d\kappa^*}{d\mu_{F_z}}$ and $\tfrac{d\kappa^*}{d\sigma_{F_z}}$, which result in
    \begin{align}
        \frac{d\kappa^*}{d\mu_{F_z}} & = \phi(\frac{\mu_{F_z} - \kappa^*}{\sigma_{F_z}})\frac{1}{\sigma_{F_z}} \big / \sum_{w \in \mathcal X} \phi(\frac{\mu_{F_w} - \kappa^*}{\sigma_{F_w}})\frac{1}{\sigma_{F_w}} = h_z / \sum_{w \in \mathcal X} h_w,\nonumber\\
        \frac{d\kappa^*}{d\sigma_{F_z}} & = \phi(\frac{\mu_{F_z} - \kappa^*}{\sigma_{F_z}}) (-\frac{\mu_{F_z} - \kappa^*}{\sigma_{F_z}^2}) \big / \sum_{w \in \mathcal X} \phi(\frac{\mu_{F_w} - \kappa^*}{\sigma_{F_w}})\frac{1}{\sigma_{F_w}}\nonumber\\
        & = - \frac{\mu_{F_z} - \kappa^*}{\sigma_{F_z}} \frac{d\kappa^*}{d\mu_{F_z}}.\label{eq:dkappastardmuanddsigma}
    \end{align}
    where $h_z := \phi(\tfrac{\mu_{F_z} - \kappa^*}{\sigma_{F_z}}) \tfrac{1}{\sigma_{F_z}}$. Combining Equation~\eqref{eq:cme_derivatives_first_expression} with Equation~\eqref{eq:dkappastardmuanddsigma}, we get the statement in the theorem.
\end{proof}

\varPrincipleForCme*
\begin{proof}
    First notice that by the definition of $\tilde I$, we obtain the easier objective to work with:
    \begin{equation*}
        \mathcal W(r) = \sum_{x \in \mathcal X} r_x \mu_{F_x} + \phi(\Phi^{-1}(r_x)) \sigma_{F_x}.
    \end{equation*}
    Next, we show that \(\mathcal W(r)\) is concave by computing the Hessian:
    \begin{align*}
        \frac{\partial}{\partial r_x} \mathcal W(r) & = \mu_{F_x} - \sigma_{F_x} \Phi^{-1}(r_x) \phi(\Phi^{-1}(r_x)) \frac{d}{d r_x} \Phi^{-1}(r_x)\\
        & = \mu_{F_x} - \sigma_{F_x} \Phi^{-1}(r_x)\nonumber\\
        \frac{\partial^2}{\partial r_x \partial r_z} \mathcal W(r) & = - \sigma_{F_x} \mathds 1_{x=z} \frac{1}{\phi(\Phi^{-1}(r_x))} \begin{cases}
            < 0 & x = z\\
            = 0 & x \not = z
        \end{cases},
    \end{align*}
    where the inverse function rule was employed twice. From negative definiteness strict concavity follows immediately. We show next that \(r^* \in \text{relint}(\Delta(\mathcal X))\), the relative interior of the probability simplex. Indeed, at the border of the probability simplex the partial derivatives explode:
    \begin{equation*}
        \frac{\partial}{\partial r_x} \mathcal W(r) = \mu_{F_x} - \sigma_{F_x} \Phi^{-1}(r_x) =
            \begin{cases}
            \infty & r_x \to 0^+\\
            \text{finite} & r_x \in (0,1)\\
            -\infty & r_x \to 1^-
        \end{cases}.
    \end{equation*}
    Together with the concavity of \(\mathcal W(\cdot)\) this ensures that \(r^* \in \text{relint}(\Delta(\mathcal X))\). Hence, \(r^*\) is a local optimiser of \(\mathcal W(r)\) on the plane defined by \(\sum_{x \in \mathcal X} r_x = 1\). Consequently, we obtain the Lagrangian function
    \begin{equation*}
        \mathcal L(r, \kappa) : (0,1)^{|\mathcal X|} \times \mathbb R \to \mathbb R\quad r \mapsto \mathcal W(r) + \kappa (1-\sum_{x\in\mathcal X} r_x ).
    \end{equation*}
    Setting its partial derivatives equal to zero, we derive the closed-form solution:
    \begin{align*}
        0 = \mu_{F_x} - \sigma_{F_x} \Phi^{-1}(r_x^*) - \kappa^* \iff r_x^*  = \Phi(\frac{\mu_{F_x} - \kappa^*}{\sigma_{F_x}}),
    \end{align*}
    where \(\kappa^*\) ensures a normalised distribution, i.e. \(\sum_{x\in \mathcal X} r_x^* = 1\).
\end{proof}

\setcounter{prop}{4}

\begin{prop}%
    The maximizer to Equation~\eqref{eq:vapor_variational_objective} on the probability simplex admits the closed-form expression
    \begin{equation*}%
        v_x := v\!\left(\frac{\mu_{F_x} - \nu^*}{\sigma_{F_x}}\right) \text{ with $\nu^*$ such that } \sum_x v_x = 1,
    \end{equation*}
    where $v(c) := \exp(-(\sqrt{c^2 + 4} - c)^2 / 8)$. Moreover, to find $\nu^*$ we can use binary search with $k \in \Theta(\log(\sqrt{\log |\mathcal X|}/\epsilon))$ iterations, ensuring that the $k$-th iterate $v^k$ satisfies $\|v^* - v^k\|_\infty < \epsilon$.
\end{prop}
\begin{proof}
    We show first that \(r^* \in \text{relint}(\Delta(\mathcal X))\). Indeed, at the border of the probability simplex the partial derivatives explode:
    \begin{align*}\label{eq:theorem_max_min_optimisation_definition_and_properties_proof_partial_derivatives}
        \frac{\partial}{\partial r_x} \mathcal V(r) & = \mu_{F_x} + \sigma_{F_x} \left(\sqrt{-2\ln r_x} - \frac{1}{\sqrt{-2\ln  r_x}}\right) = \begin{cases}
            \infty & r_x \to 0^+\\
            \text{finite} & r_x \in (0,1)\\
            - \infty & r_x \to 1^-,
        \end{cases}
    \end{align*}
    which together with the concavity of \(\mathcal V(\cdot)\), shown in Proposition~\ref{thm:vapor_saddle_point_optimisation_problem}, ensures that \(r^* \in \text{relint}(\Delta(\mathcal X))\). Hence, \(r^*\) is a local optimiser of \(\mathcal V(r)\) on the plane defined by \(\sum_{x \in \mathcal X} r_x =1\). Consequently, we obtain the Lagrangian function
    \begin{equation*}
        \mathcal L(r, \nu): (0,1)^{|\mathcal X|} \times \mathbb R \to \mathbb R \quad r \mapsto \mathcal V(r) + \nu \left(1-\sum_{x \in \mathcal X} r_x\right).
    \end{equation*}
    Setting its partial derivatives equal to zero we derive the closed-form solution:
    \begin{gather}
        0 = \mu_{F_x} + \sigma_{F_x} \left(\sqrt{-2\ln  r_x} - \frac{1}{\sqrt{-2\ln  r_x}}\right) - \nu
        \iff
        0 = {\sqrt{-2\ln  r_x}}^2 + {\sqrt{-2\ln  r_x}} \overbrace{\frac{\mu_{F_x}-\nu}{\sigma_{F_x}}}^{c_x} - 1\nonumber\\
        \iff
        \sqrt{-2\ln  r_x} = \frac{- c_x + \sqrt{c_x^2+4}}{2} \iff
        r_x^* = \exp(-[\sqrt{c_x^2 + 4} - c_x]^2 / 8) \quad \text{where } c_x = \frac{\mu_{F_x} - \nu}{\sigma_{F_x}}.\label{eq:theorem_max_min_optimisation_strategy_for_solving_deriving_closed_form_by_first_order_conditions}
    \end{gather}
    Being a Lagrange multiplier, \(\nu\) automatically ensures a normalised probability distribution, i.e. \(\sum\nolimits_{x\in \mathcal X} r_x^* = 1\). 

    To show that $\nu^*$ can be found with binary search using $k \in \Theta(\log(\sqrt{\log |\mathcal X|}/\epsilon))$ steps while ensuring $\|r^* - r^k\|_\infty < \epsilon$, it suffices to demonstrate that $v^* \mapsto \sum\nolimits_{x \in \mathcal X} v(\frac{\mu_{F_x} - \nu^*}{\sigma_{F_x}})$ is continuous and monotonously decreasing, $v^{-1}(r_x) = 1/\sqrt{-2 \ln r_x} - \sqrt{-2 \ln r_x}$, $\mu_{F}^{min} - v^{-1}(\tfrac{1}{|\mathcal X|}) \sigma_{F}^{min} \leq \nu^* \leq \mu_{F}^{max} - v^{-1}(\tfrac{1}{|\mathcal X|}) \sigma_{F}^{max}$, and that $v(\cdot)$ is Lipschitz continuous.

    Lipschitz continuity follows immediately from a bounded derivative $$\frac{d}{dc}v(c) = \exp(-[\sqrt{c^2 + 4} - c]^2 / 8) \frac{\sqrt{c^2 + 4} - c}{4} (\frac{2c}{2 \sqrt{c^2 + 4}}-1) \in [0, 0.4).$$
    Since \(\nu \mapsto c_x\), \(c_x \mapsto (\sqrt{c_x^2 + 4} - c_x)^2\), and \(z \mapsto \exp(-z / 8)\) are each monotonously decreasing, their composition \(\nu \mapsto r_x^\nu\) is also monotonously decreasing. As the sum of decreasing functions \(
        \nu \mapsto \sum\nolimits_{x \in \mathcal X} p_x^\nu\)
    is monotonously decreasing. The binary search window is initialised based on the insight that
    \begin{align*}
        1 \overset{!}{=} \sum_{x \in \mathcal X} r_x^* & \leq |\mathcal X|  v(c_u) \implies c_u \geq v^{-1}(1/|\mathcal X|)\nonumber\\
        1 \overset{!}{=} \sum_{x \in \mathcal X} r_x^* & \geq |\mathcal X|  v(c_l) \implies c_l \leq v^{-1}(1/|\mathcal X|)
    \end{align*}
    Finally, from the equivalences in Equation~\eqref{eq:theorem_max_min_optimisation_strategy_for_solving_deriving_closed_form_by_first_order_conditions} we obtain an inverse to $v(c)$, i.e.
    \begin{equation*}
        v^{-1}(r_x) = \frac{1}{\sqrt{-2 \ln r_x}} - \sqrt{-2 \ln r_x},
    \end{equation*}
    which we remark fulfills $v^{-1}(1/k) \leq 0\ \forall k \geq 2$.
    As a direct consequence we obtain \(\nu \leq \mu_{F}^{max} - v^{-1}(1/|\mathcal X|) \sigma_{F}^{max}\), since otherwise \(c_z < v^{-1}(1/|\mathcal X|)\) for all \(z \in \mathcal X\). Similarly, it holds that \(\nu \geq \mu_{F}^{min} - v^{-1}(1/|\mathcal X|) \sigma_{F}^{min}\), since otherwise \(c_z > v^{-1}(1/|\mathcal X|)\) for all \(z \in \mathcal X\). Hence,
    \begin{equation*}
        \nu \in [\mu_{F}^{min} - v^{-1}(1/|\mathcal X|) \sigma_{F}^{min}, \mu_{F}^{max} - v^{-1}(1/|\mathcal X|) \sigma_{F}^{max}].
    \end{equation*}
\end{proof}

\logarithmicFStarQuantileSearch*
\begin{proof}
    Continuity and monotonicity of $g(f)$ follows from continuity and monotonicity of $\mathbb P[\tilde F_z \leq f]$ for all $z \in \mathcal X$. The existence and uniqueness of $\bar f$ follows swiftly, since $g(f) = \mathbb P[\tilde F^* \leq f]$\footnote{Recall, that here we are in the independent Gaussian process setting.}, as a cumulative distribution function, has range $(0,1)$. Let us derive the search window. It holds that
    \begin{equation*}
        \Phi^{|\mathcal X|}(\frac{f - \mu_{F}^{max}}{\sigma_2}) \leq \prod_{x \in \mathcal X} \Phi(\frac{f - \mu_{F}^{max}}{\sigma_{F_x}}) \leq \underbrace{\prod_{x \in \mathcal X} \Phi(\frac{f - \mu_{F_x}}{\sigma_{F_x}})}_b \leq \prod_{x \in \mathcal X} \Phi(\frac{f - \mu_{F}^{min}}{\sigma_{F_x}}) \leq \Phi^{|\mathcal X|}(\frac{f - \mu_{F}^{min}}{\sigma_1})
    \end{equation*}
    where $\sigma_1 = \sigma_{F}^{min}$ if $f \geq \mu_{F}^{min}$ and $\sigma_1 = \sigma_{F}^{max}$ otherwise, and $\sigma_2 = \sigma_{F}^{max}$ if $f \geq \mu_{F}^{max}$ and $\sigma_2 = \sigma_{F}^{min}$ otherwise. Equivalently, it then holds that
    \begin{gather*}
        \frac{f - \mu_{F}^{max}}{\sigma_2} \leq \Phi^{-1}(b^{1/|\mathcal X|}) \leq \frac{f - \mu_{F}^{min}}{\sigma_1} \iff
        \mu_F^{min} + \sigma_1 \Phi^{-1}(b^{1/|\mathcal X|}) \leq f \leq \mu_F^{max} + \sigma_2 \Phi^{-1}(b^{1/|\mathcal X|}).
    \end{gather*}
    Now, since by assumption $b \geq \frac{1}{4}$ and $|\mathcal X| \geq 2$, it holds that $b^{1/|\mathcal X|} \geq \frac{1}{2}$ and hence $\Phi^{-1}(b^{1/|\mathcal X|}) \geq 0$. Consequently, we obtain the desired search window
    \begin{gather*}
        \mu_F^{min} + \sigma_F^{min} \Phi^{-1}(b^{1/|\mathcal X|}) \leq f \leq \mu_F^{max} + \sigma_F^{max} \Phi^{-1}(b^{1/|\mathcal X|}).
    \end{gather*}
    Regarding the scaling of the search window, notice that the window size is given by $\mu_F^{max} - \mu_{F}^{min} + (\sigma_F^{max} - \sigma_F^{min}) \Phi^{-1}(b^{1/|\mathcal X|})$. Now, we may apply Lemma~\ref{lem:asymptotics_of_gaussian_cdf_and_its_inverse}, which states that
    \begin{equation*}
        \Phi^{-1}(y) \sim \sqrt{-2 \ln (1-y)} \text{ as } y \to 1^-.
    \end{equation*}
    Plugging in $b^{1/|\mathcal X|}$ for $y$ then gives us
    \begin{equation}\label{eq:proof_g_logarithmic_search_intermediate_scaling_equality}
        \Phi^{-1}(b^{1/|\mathcal X|}) \sim \sqrt{-2 \ln(1-b^{1/|\mathcal X|})} \text{ as } |\mathcal X| \to \infty.
    \end{equation}
    According to the L'Hôpital-Bernoulli rule, it holds that $\lim_{a \to 1} \frac{1-a}{- \ln(a)} = \lim_{a \to 1} a = 1$.
    Since $b^{1/|\mathcal X|} \to 1^-$ as $|\mathcal X| \to \infty$, we equivalently get
    \begin{align*}
        {1-b^{1/|\mathcal X|}} \sim {- \ln(b^{1/|\mathcal X|}))} =  {\ln(1/b)} / |\mathcal X| \text{ as } |\mathcal X| \to \infty.
    \end{align*}
    Combining this with Equation~\eqref{eq:proof_g_logarithmic_search_intermediate_scaling_equality}, we obtain
    \begin{equation*}
        \Phi^{-1}(b^{1/|\mathcal X|}) \sim \sqrt{2 \ln(|\mathcal X|) - 2 \ln(\ln(1/b))}.
    \end{equation*}
    Hence, the search window scales in
    \begin{align*}
        \Theta(\mu_F^{max} - \mu_{F}^{min} + (\sigma_F^{max} - \sigma_F^{min}) \Phi^{-1}(b^{1/|\mathcal X|})) %
        & = \Theta(\sqrt{\ln |\mathcal X| }).
    \end{align*}
    Finally, $k$ steps of binary search divide the search window by $2^k$ resulting in an accuracy of
    \begin{equation*}
        |\bar f - \bar f^k| \leq \frac{\mu_F^{max} \!-\! \mu_{F}^{min} + (\sigma_F^{max} \!-\! \sigma_F^{min}) \Phi^{-1}(b^{1/|\mathcal X|})}{2^{k+1}} \!\iff\! k \leq \log_2 \! \left(\frac{\mu_F^{max} \!-\! \mu_{F}^{min} + (\sigma_F^{max} \!-\! \sigma_F^{min}) \Phi^{-1}(b^{1/|\mathcal X|})}{2 |\bar f - \bar f^k|}\right)\!.
    \end{equation*}
    Therefore, for $k = \log_2((\mu_F^{max} \!-\! \mu_{F}^{min} + (\sigma_F^{max} \!-\! \sigma_F^{min}) \Phi^{-1}(b^{1/|\mathcal X|})) / (2 \nu))$ it must hold that $|\bar f - \bar f^k| \leq \nu$. Inserting the asymptotic scaling of the search window finishes the proof.
\end{proof}

\logarithmicFStarWithoutXQuantileSearch*
\begin{proof}
    Since $\tilde g^x$ is continuous and strictly monotonously increasing on a section with range $(0,1]$ and larger than $1$ elsewhere, see the illustration in Figure~\ref{fig:approximation_to_gx_is_not_a_cdf}, it follows immediately that for $b \in (0,1)$ $\exists! \bar f_x \in \mathbb R$ s.t. $\tilde g^x(\bar f_x) = b$. Let us next establish an upper bound on $\bar f_x$. It holds that
    \begin{align*}
        \tilde g^x(\bar f_x) & =  \Phi(\frac{\bar f_x-m}{s}) / \Phi(\frac{\bar f_x-\mu_{F_x}}{\sigma_{F_x}}) > \Phi(\frac{\bar f_x-m}{s}) \geq b
    \end{align*}
    for $\bar f_x \geq m + \Phi^{-1}(b) \cdot s$,
    directly implying the upper bound on the search window in this theorem. For the lower bound we make use of Lemma~\ref{lem:asymptotics_of_gaussian_cdf_and_its_inverse}, which states that $\forall a < 0$ one has
    \begin{equation}\label{eq:proof_SW_II_asymptotics_of_gaussian_restated}
        \phi(a) \left(\frac{1}{-a} - \frac{1}{-a^3} \right) \leq \Phi(a) \leq \frac{\phi(x)}{-a}.
    \end{equation}
    Assuming $f \leq \mu_{F_x} - \sqrt{2} \sigma_{F_x}$, which is automatically less than $m$, one has $1 - 1/(\tfrac{f-\mu_{F_x}}{\sigma_{F_x}})^2 \geq \tfrac{1}{2}$. Together with Equation~\eqref{eq:proof_SW_II_asymptotics_of_gaussian_restated}, we then get
    \begin{equation}\label{eq:establish_lower_bound_on_quartiles_of_tilde_g_x}
        \tilde g^x(f) = \frac{\Phi(\frac{f - m}{s})}{\Phi(\frac{f - \mu_{F_x}}{\sigma_{F_x}})} \leq \frac{\phi(\frac{f - m}{s})}{\phi(\frac{f - \mu_{F_x}}{\sigma_{F_x}})} \frac{2\frac{f-\mu_{F_x}}{\sigma_{F_x}}}{\frac{f-m}{s}} \leq 2 \frac{\phi(\frac{f - m}{s})}{\phi(\frac{f - \mu_{F_x}}{\sigma_{F_x}})},
    \end{equation}
    where in the last inequality we used that for $f \leq \mu_{F_x} < m$ it holds that $\tfrac{f-\mu_{F_x}}{f-m} = \tfrac{\mu_{F_x} - f}{m - f} < 1$ and that for $s \leq \sigma_{F_x}$ it holds that $\tfrac{s}{\sigma_{F_x}} \leq 1$. We want to figure out for what $f$ the right hand side of Equation~\eqref{eq:establish_lower_bound_on_quartiles_of_tilde_g_x} cannot reach $b$, i.e.
    \begin{align*}
        b & > 2 \exp((f-\mu_{F_x})^2 \big/ 2 \sigma_{F_x}^2 - (f-m)^2 \big/ 2 s^2),
    \end{align*}
    which is implied by either of the conditions below:
    \begin{align*}
        \ln(b/2) & > (f-\mu_{F_x})^2 \big/ 2 \sigma_{F_x}^2 - (f-m)^2 \big/ 2 \sigma_{F_x}^2 = \frac{(f-\mu_{F_x})^2 - (f-m)^2}{2 \sigma_{F_x}^2}\\
        \ln(b/2) & > (f-m)^2 \big/ 2 \sigma_{F_x}^2 - (f-m)^2 \big/ 2 s^2 = (f-m)^2 \cdot (\frac{1}{2 \sigma_{F_x}^2} - \frac{1}{2 s^2}).
    \end{align*}
    These conditions, in turn, are satisfied for
    \begin{equation*}
        f < \frac{\sigma_{F_x}^2 \ln(b/2)}{m - \mu_{F_x}} + \frac{m + \mu_{F_x}}{2} \quad\text{ and }\quad f < m - \sqrt{\ln(b/2) / (\frac{1}{2 \sigma_{F_x}^2} - \frac{1}{2 s^2})},
    \end{equation*}
    leading to the stated lower bound on the search window in this theorem. Clearly, the size of the search window only depends on $b, m, s, \mu_F$, and $\sigma_F$, i.e., it is independent of $|\mathcal X|$. The rest of the theorem follows immediately.
\end{proof}

\setcounter{prop}{7}

\begin{prop}\label{thm:computing_the_independence_estimator}
    Suppose an independent stochastic process\footnote{That is, for any \(x_1, \ldots, x_n \subseteq \mathcal X\) it holds that \(F_{x_1}, \ldots, F_{x_n}\) are mutually independent.} \(\{F_x: \Omega \to \mathbb R \ \vert\ x \in \mathcal X\}\) on a finite domain \(\mathcal X\). Let \(\epsilon > 0\) and assume
    \(f_0, \ldots, f_l \in \overline{\mathbb R}\) with \(f_i \leq f_{i+1}\) such that \(f_0 = - \infty\), \(f_l = \infty\), $\max\nolimits_{x \in \mathcal X}\mathbb P[F_x \leq f_1] \leq 2 \epsilon$, $\max\nolimits_{x \in \mathcal X}\mathbb P[F_x > f_{l-1}] \leq 2 \epsilon$, and $\max\nolimits_{x \in \mathcal X} \mathbb P[F_x \in (f_i, f_{i+1}]] \leq 2 \epsilon \ \forall i=1, \ldots, l-2$. Then it holds for \(X^* = \arg\max_{z \in \mathcal X} F_z\) that
    \begin{equation*}
        \bigg| \mathbb P[x \in X^*] - \sum_{i=0}^{l-1} \frac{g_x(f_{i+1}) + g_x(f_i)}{2} \mathbb P[F_x \in (f_i, f_{i+1}]]\bigg| \leq \epsilon,
    \end{equation*}
    where \(g_x(f) := \prod_{z \in \mathcal X \setminus \{x\}} \mathbb P[F_z \leq f]\).
\end{prop}
\begin{proof}
    First, recall that mutually independent random variables \(Z_1, \ldots Z_n\) are characterized by \(\mathbb P[Z_1 \in \mathcal A_1, \ldots, Z_n \in \mathcal A_n] = \prod\nolimits_{i=1}^n \mathbb P[Z_i \in \mathcal A_i]\) for any Borel sets \(A_1, \ldots, A_n\). Hence, conditionals \(Z_2, \ldots, Z_n | Z_1\) are also mutually independent:
    \begin{align*}
        \mathbb P[Z_2 \in \mathcal A_2, \ldots, Z_n \in \mathcal A_n | Z_1 \in \mathcal A_1] = & \mathbb P[Z_1 \in \mathcal A_2, \ldots, Z_n \in \mathcal A_n] / \mathbb P[Z_1 \in \mathcal A_1]\nonumber\\
        = & \prod_{i=1}^n \mathbb P[Z_i \in \mathcal A_i] / P[Z_1 \in \mathcal A_i] = \prod_{i=2}^n \mathbb P[Z_i \in \mathcal A_i].
    \end{align*}
    Conditional independence then allows us to derive a tractable integral for \(\mathbb P[x \in X^*]\), which we write as a sum of integrals over an \(l\)-piece partition of \(\mathbb R\):
    \begin{align*}
        \mathbb P[x \in X^*] & = \mathbb P[F_z \leq F_x\ \forall z \in \mathcal X \setminus \{x \}] = \mathbb E [\mathbb P[F_z \leq F_x\ \forall z \in \mathcal X \setminus \{x \} | F_x]] = \mathbb E [\prod\nolimits_{z \in \mathcal X \setminus \{x \}} \mathbb P[F_z \leq F_x  | F_x]]\nonumber\\
    & = \mathbb E [g_x(F_x)] = \int_{\mathbb R} g_x(f) d\mathbb P[F_x \in \cdot\ ] = \sum_{i=0}^{l-1} \int_{(f_i, f_{i+1}]} g_x(f) d\mathbb P[F_x \in \cdot\ ].
    \end{align*}
    Each of these integrals can then be numerically evaluated using the trapezoidal rule. Moreover, we can upper bound the approximation error of numerical integration. Indeed, due to the triangle inequality, the fact that \(g_x\) increases monotonously, and through a telescoping sum, one has
    \begin{align*}
        & \bigg| \mathbb P[x \in X^*] - \sum_{i=0}^{l-1} \frac{g_x(f_{i+1}) + g_x(f_i)}{2} \mathbb P[F_x \in (f_i, f_{i+1}]]\bigg|\nonumber\\
        \leq & \sum_{i=0}^{l-1} \bigg|\int_{(f_i, f_{i+1}]} g_x(f) d\mathbb P[F_x \in \cdot\ ] - \frac{g_x(f_{i+1}) + g_x(f_i)}{2} \mathbb P[F_x \in (f_i, f_{i+1}]]\bigg|\nonumber\\
        \leq & \sum_{i=0}^{l-1} \frac{g_x(f_{i+1})-g_x(f_i)}{2} \mathbb P[F_x \in (f_i, f_{i+1}]]  \leq \sum_{i=0}^{l-1} \frac{g_x(f_{i+1})-g_x(f_i)}{2} \max_{i = 0, \ldots, l-1} \mathbb P[F_x \in (f_i, f_{i+1}]]\nonumber\\
        = & \max_{i = 0, \ldots, l-1} \frac{\mathbb P[F_x \in (f_i, f_{i+1}]]}{2}.
    \end{align*}
    Finally, for the partitioning \(\mathbb R = (-\infty, f_1] \cup \bigcup\nolimits_{l=1}^{l-2} (f_i, f_{i+1}] \cup (f_{l-1}, \infty)\) to ensure that \(\mathbb P[F_x \in (f_i, f_{i+1}]] \leq 2 \epsilon\) for all \(x \in \mathcal X\) simultaneously, we require that
    \begin{gather*}
        \max_{x \in \mathcal X}\mathbb P[F_x \leq f_1] \leq 2 \epsilon, \qquad
        \max_{x \in \mathcal X}\mathbb P[F_x > f_{l-1}] \leq 2 \epsilon, \qquad
        \text{and} \qquad
        \max_{x \in \mathcal X} \mathbb P[F_x \in (f_i, f_{i+1}]] \leq 2 \epsilon \quad \forall i=1, \ldots, l-2.
    \end{gather*}
    These are exactly the conditions that the Theorem demands.
\end{proof}

\begin{restatable}[]{prop}{maximumOfGaussiansConvergesInProbability}\label{cor:maximum_of_gaussians_converges_in_probability}
    Assume i.i.d. \(Z_1, Z_2, \ldots \sim \mathcal N(\mu,\sigma^2)\). Then \(\exists (a_n)_{n \in \mathbb N} \) s.t.
    \begin{equation*}
        \forall \epsilon > 0 \lim_{n \to \infty} \mathbb P[|\max_{i \leq n} Z_i - a_n| > \epsilon] = 0.
    \end{equation*}
    One such sequence is given by \(a_n = \mu + \sigma \cdot \Phi^{-1}(1- \frac{1}{n})\). The rate of convergence is illustrated in Figure~\ref{fig:max_i_i_d_gaussians}.
    \begin{figure}[ht]
        \centering
        \begin{subfigure}[b]{0.4\textwidth}
            \centering
            \includegraphics[width=\linewidth]{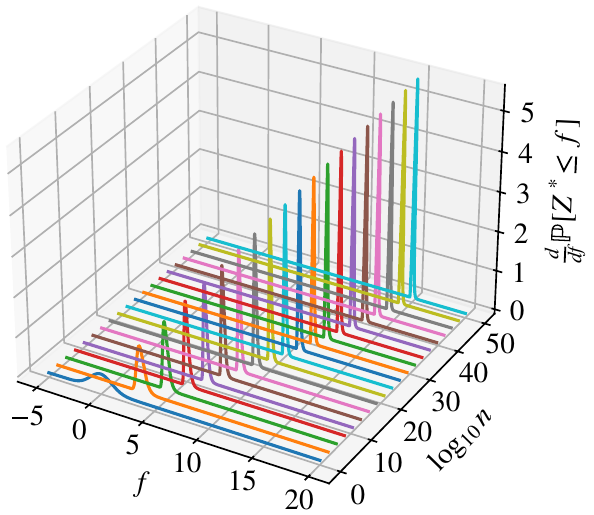}
            \caption{Probability density function of maximum of i.i.d. standard normals.}
            \label{subfig:distribution_max_i_i_d_gaussians}
        \end{subfigure}\hfill
        \begin{subfigure}[b]{0.5\textwidth}
            \centering
            \includegraphics[width=\linewidth]{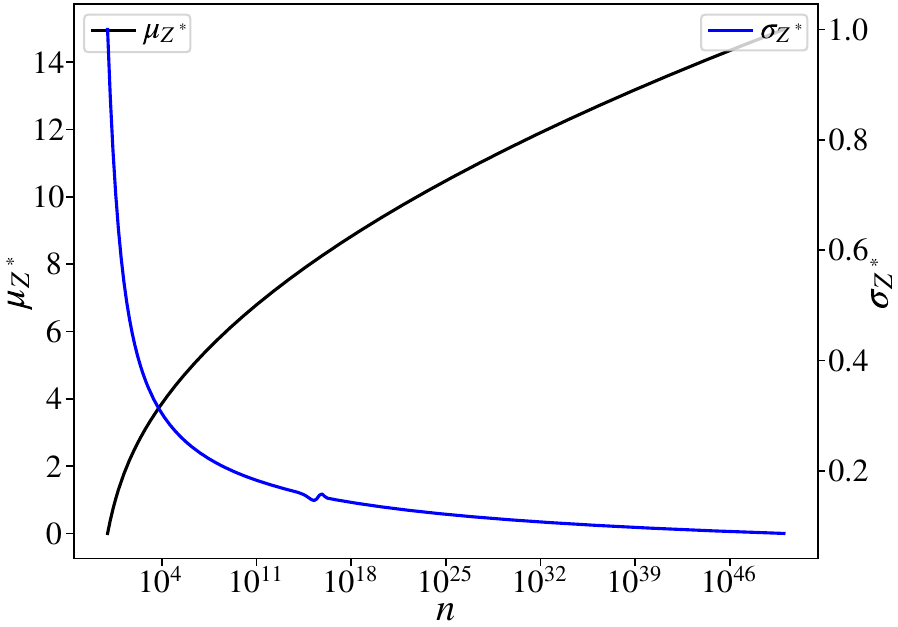}
            \caption{Mean and standard deviation of maximum of i.i.d. standard normals.}
            \label{subfig:mean_std_max_i_i_d_gaussians}
        \end{subfigure}
        \caption{The distribution of the maximum of i.i.d. \(Z_1, Z_2, \ldots, Z_n \sim \mathcal N(0,1)\) very slowly approaches that of a deterministic quantity as \(n\) is increased.}
        \label{fig:max_i_i_d_gaussians}
    \end{figure}
\end{restatable}
\begin{proof}
    By shifting and scaling we can assume without loss of generality that $\mu=0$ and $\sigma=1$. Furthermore, 
    \begin{align*}
        \lim_{x\to\infty} \frac{1- \Phi(x+\epsilon)}{1- \Phi(x)} & = \lim_{x\to\infty} \frac{\phi(x+\epsilon)}{\phi(x)} = \lim_{x\to\infty} \exp(-\frac{(x+\epsilon)^2 - x^2}{2}) = \lim_{x\to\infty} \exp(-\epsilon x - \epsilon^2/2) = 0\qquad \forall \epsilon > 0
    \end{align*}
    where L'Hôpital's rule was applied. We now directly apply Lemma~\ref{thm:law_of_large_numbers_for_maximum}.
\end{proof}

\subsection{Lemmas}

\exampleNecessityFullCovarianceMatrixForUnbiasedEstimation*
 \begin{proof}
     We first verify that $\Sigma_F^s := I + s(e_i e_j^T + e_j e_i^T)$ is indeed symmetric positive semi-definite. Symmetry is trivial. On the other hand, positive semi-definiteness follows from
     \begin{equation*}
         z^T \Sigma_F z = \|z\|^2 + 2 s \cdot z^T e_i \cdot e_j^T z = 1 + 2 s z_i z_j \geq 0\quad \forall z : \|z\| = 1
     \end{equation*}
     where we have used that $z_i^2 + z_j^2 \leq 1$ implies $|z_i| \leq \sqrt{1-z_j^2}$ which in turn gives $|z_i|\cdot |z_j| \leq \sqrt{z_j^2 - z_j^4} \leq \frac{1}{2}$ for any $z_j \in [-1, 1]$. The more general case allowing $\|z\| \not= 1$ follows from linearity. Next, let us verify the \textit{probability of maximality} in the limit. To that end, consider the explicit Cholesky decomposition of $\Sigma_F^s$ given by
     \begin{equation*}
         (\Sigma_F^s)^{1/2} = I + s \cdot e_i e_j^T + (\sqrt{1-s^2}-1)  \cdot e_i e_i^T,
     \end{equation*}
     which can be verified by evaluating $(\Sigma_F^s)^{1/2} ((\Sigma_F^s)^{1/2})^T$ to $\Sigma_F^s$ through rigorous algebra. Alternatively, we may consider the element-wise representation as
     \begin{equation*}
         ((\Sigma_F^s)^{1/2})_{k,h} = \begin{cases}
             1 & (k,h) \in \{(a,a) : a \not = i\}\\
             \sqrt{1-s^2} & (k,h) = (i,i)\\
             s & (k,h) = (i,j)\\
             0 & otherwise
         \end{cases}.
     \end{equation*}
     So, for $\varepsilon \sim \mathcal N(0, I)$ it holds that $F \overset{d}{=} (\Sigma_F^s)^{1/2} \varepsilon $, which can be parsed as $F_z = \varepsilon_z$ for all $z \not = i$ and $F_i = \sqrt{1-s^2} \cdot \varepsilon_i + s \cdot \varepsilon_j$. Now it should be clear that as $s \to 1^-$, $F_i \to \varepsilon_j = F_j$. However, for any $s < 1$ the maximiser $X^*$ is almost surely unique. Consequently, the probability of maximality will be evenly distributed in the limit of $s \to 1^-$ except for the halving of the probability mass among index $i$ and $j$, since up to an infinitesimally small perturbation in the form of $\varepsilon_i$ the entries $F_i$ and $F_j$ are identical. This proves the Lemma.
 \end{proof}

 \nieErrorPropagation*
\begin{proof}
    We start by noting that with $p := 1/\Phi^{-1}(0.75) \approx 1.48$ it holds that
    \begin{align*}
        |m - \bar m| & \leq \frac{1}{2}(|q_3 - \bar q_3| + |q_1 - \bar q_1|) \leq \nu,\nonumber\\
        |s - \bar s| & \leq \frac{1}{2 \Phi^{-1}(0.75)} (|q_3 - \bar q_3| + |q_1 - \bar q_1|) \leq p\,\nu,\nonumber\\
        |s^2 - \bar s^2| & = (s + \bar s) \cdot |s - \bar s| \leq (s + \bar s) p\, \nu \leq 2 \bar s p\, \nu + p^2 \nu^2.
    \end{align*}
    We will use these inequalities at various places throughout this proof.
    By a Taylor series expansion around $\delta = 0$ we have
    \begin{equation*}
        \left|\frac{1}{\sqrt{1+z}} - \frac{1}{\sqrt{1+z+\delta}}\right| = \frac{|\delta|}{2(1+z)^{3/2}} + \mathcal O(\frac{\delta^2}{(1+z)^{5/2}}).
    \end{equation*}
    Setting $z = \bar s^2 / \sigma_{F_x}^2$, $\delta %
    = \frac{s^2 - \bar s^2}{\sigma_{F_x}^2}$, and multiplying with $1/\sigma_{F_x}$ yields
    \begin{align*}
        \left|\frac{1}{\sqrt{\sigma_{F_x}^2+s^2}} - \frac{1}{\sqrt{\sigma_{F_x}^2+\bar s^2}}\right| & = \frac{|s^2 - \bar s^2|}{2(\sigma_{F_x}^2+\bar s^2)^{3/2}} + \mathcal O(\frac{(s^2 - \bar s^2)^2} {(\sigma_{F_x}^2+\bar s^2)^{5/2}}) \leq \frac{|s^2 - \bar s^2|}{2 \bar s^3} + \mathcal O(\frac{(s^2 - \bar s^2)^2} {\bar s^5})\nonumber\\
        & \leq \frac{p}{\bar s^2} \nu + \mathcal O(\nu^2) = {\varepsilon_1}.
    \end{align*}
    Now we can directly get a hold on the difference between the entries of $\Phi$ in Equation~\eqref{eq:nie_absolute_error} using the triangle inequality of the absolute value:
    \begin{align*}
        \varepsilon_2 & = |\frac{\mu_{F_x} - m}{\sqrt{\sigma_{F_x}^2 + s^2}} - \frac{\mu_{F_x} - \bar m}{\sqrt{\sigma_{F_x}^2 + \bar s^2}}| =
        |\frac{\mu_{F_x} - m}{\sqrt{\sigma_{F_x}^2 + s^2}} - \frac{\mu_{F_x} - m}{\sqrt{\sigma_{F_x}^2 + \bar s^2}} + \frac{\mu_{F_x} - m}{\sqrt{\sigma_{F_x}^2 + \bar s^2}} -\frac{\mu_{F_x} - \bar m}{\sqrt{\sigma_{F_x}^2 + \bar s^2}}|\nonumber\\
        & \leq \varepsilon_1 |\mu_{F_x} - m| + |m - \bar m| / \sqrt{\sigma_{F_x}^2 + \bar s^2} \leq \varepsilon_1 (|\mu_{F_x} - \bar m| + \nu) + |m - \bar m| / \bar s \nonumber\\
        & \leq p\frac{|\mu_{F_x} - \overline m| + \bar s}{\bar s^2} \nu + \mathcal O(\nu^2) \leq p\, \epsilon + \mathcal O(\epsilon^2)
    \end{align*}
    Finally, by the mean value theorem, $\exists\, c \in [\frac{\mu_{F_x} - m}{\sqrt{\sigma_{F_x}^2 + s^2}}, \frac{\mu_{F_x} - \bar m}{\sqrt{\sigma_{F_x}^2 + \bar s^2}}]$ such that
    \begin{equation*}
        \left |\Phi(\frac{\mu_{F_x} - m}{\sqrt{\sigma_{F_x}^2 + s^2}}) - \Phi(\frac{\mu_{F_x} - \bar m}{\sqrt{\sigma_{F_x}^2 + \bar s^2}}) \right | \leq \varepsilon_2 \phi(c) \leq \varepsilon_2 / \sqrt{2 \pi} \leq \epsilon + \mathcal O(\epsilon^2).
    \end{equation*}
\end{proof}

\begin{lem}[Asymptotics of Gaussian cumulative distribution function and its inverse]\label{lem:asymptotics_of_gaussian_cdf_and_its_inverse}
    For all $x < 0$ it holds that \begin{equation*}
        \phi(x) \left(\frac{1}{-x} - \frac{1}{-x^3}\right) \leq \Phi(x) \leq \frac{\phi(x)}{-x}.
    \end{equation*}
    Moreover, we have the following asymptotic behavior:
    \begin{align*}
        \Phi(x) & \sim \frac{\phi(x)}{-x} \text{ as } x \to - \infty,\\
        \Phi^{-1}(y) & \sim - \sqrt{-2\ln y} \text{ as } y \to 0^+,\\
        \Phi^{-1}(y) = -\Phi^{-1}(1-y) & \sim \sqrt{-2 \ln(1-y)} \text{ as } y \to 1^{-},
    \end{align*}
    where \(a_n \sim b_n \iff \lim_{n \to \infty} b_n / a_n = 1\).
    \begin{figure}[ht]
        \centering
        \begin{subfigure}[b]{0.46\textwidth}
            \centering
            \includegraphics[width=\linewidth]{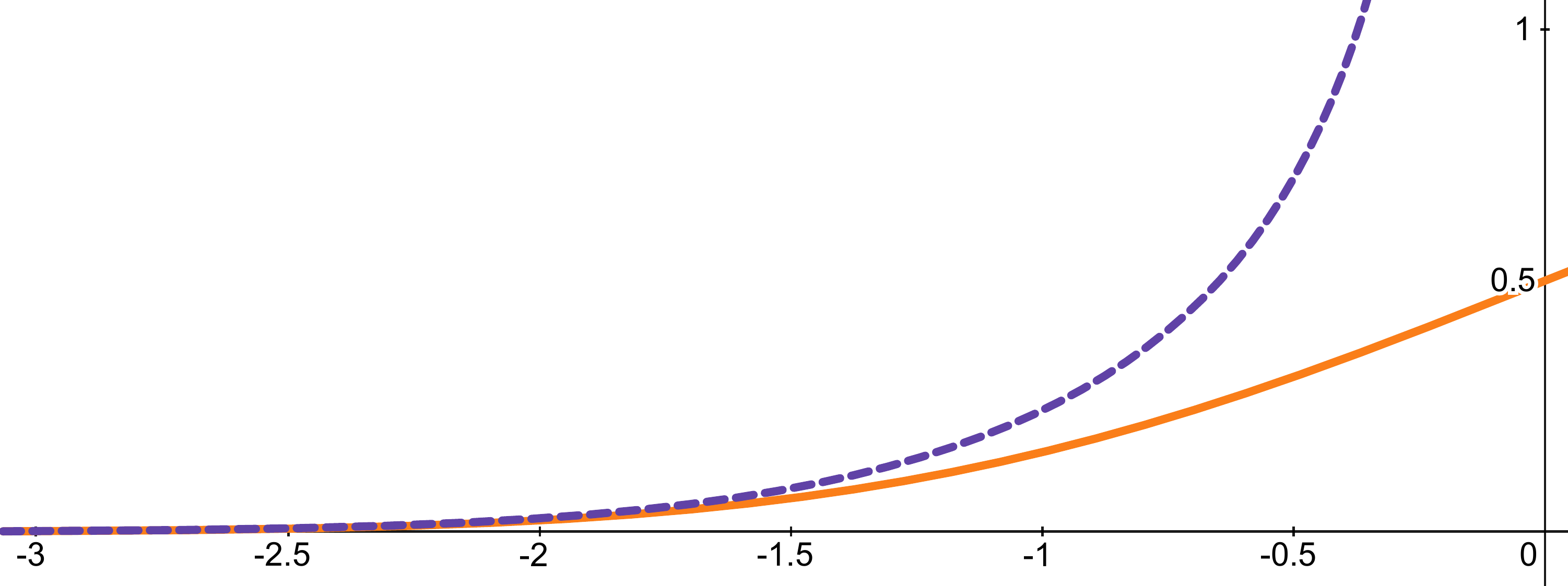}
            \caption{Linear-linear plot of \(\Phi(x)\) (solid orange line) and \(-\phi(x)/x\) (dashed purple line).}
            \label{subfig:asymptotic_of_Phi}
        \end{subfigure}\hfill
        \begin{subfigure}[b]{0.5\textwidth}
            \centering
            \includegraphics[width=\linewidth]{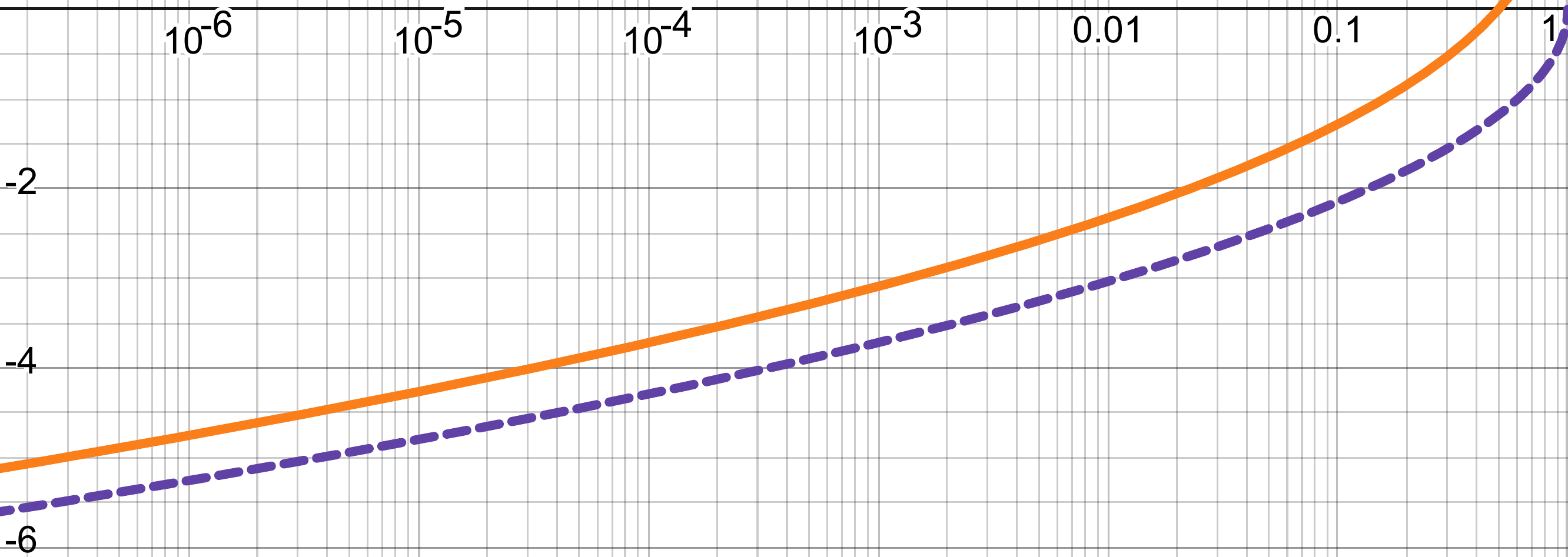}
            \caption{Linear-log Plot of \(\Phi^{-1}(y)\) (solid orange line) and \(-\sqrt{-2 \ln y}\) (dashed purple line).}
            \label{subfig:asymptotic_of_inv_Phi}
        \end{subfigure}
        \caption{Finite-value behavior of the Gaussian cumulative distribution function and its inverse compared against their respective asymptotics.}
        \label{fig:nonlimiting_behaviour_of_gaussian_CDF_and_inverse_compared_against_asymptotics}
    \end{figure}
\end{lem}
\begin{proof}
    Integration by parts provides upper and lower bounds on \(\Phi(x)\) for \(x < 0\):
    \begin{align*}
        \Phi(x) & = \int_{-\infty}^x \phi(s) ds = \int_{-\infty}^x \frac{\frac{d}{ds} \phi(s)}{-s} ds = \frac{\phi(x)}{-x} - \underbrace{\int_{-\infty}^x \frac{\phi(s)}{s^2} ds}_{\geq 0}\\
        & = \frac{\phi(x)}{-x} - \int_{-\infty}^x \frac{\frac{d}{ds}\phi(s)}{-s^3} ds = \frac{\phi(x)}{-x} - \frac{\phi(x)}{-x^3} + 3 \underbrace{\int_{-\infty}^x \frac{\phi(s)}{s^4} ds}_{\geq 0}.
    \end{align*}
    Regarding \(\Phi^{-1}(x)\) we perform a change of variable based on \(\lim_{x \to -\infty} \Phi(x) = 0\):
    \begin{align*}
        \lim_{y \to 0^+} \frac{-\sqrt{-2\ln y}}{\Phi^{-1}(y)} = \lim_{x \to -\infty} \frac{-\sqrt{-2\ln \Phi(x)}}{\Phi^{-1}(\Phi(x))} = \lim_{x \to -\infty} \frac{-\sqrt{-2\ln \Phi(x)}}{x}.
    \end{align*}
    Now, since \(-2\ln \phi(x) = x^2 +\ln 2\pi\) one may reuse the upper and lower bounds on \(\Phi\) to obtain that \(\lim\nolimits_{y \to 0^+} -\sqrt{-2\ln y} \big / \Phi^{-1}(y) = 1\).
    Finally, by point symmetry $y = \Phi(x) = 1 - \Phi(-x)$ and hence $\Phi^{-1}(y) = x$ and $-\Phi^{-1}(1-y) = x$, which together give the last equation of the theorem.
\end{proof}

\begin{lem}\label{lem:logarithmic_search_to_obtain_probability_of_optimality}
    Let \(\{F_x : x \in \mathcal X\}\) be a stochastic process with \(1 < |\mathcal X| < \infty\). In order to find \(\kappa^* \in \mathbb R\) such that $\sum\nolimits_{x \in \mathcal X} \mathbb P[F_x \geq \kappa^*] \overset{!}{=} 1$,
    we can use logarithmic search with search window derived from
    \begin{equation}\label{eq:logarithmic_search_to_obtain_probability_of_optimality_suggested_conditions_constraining_relevant_region}
        \max_{z \in \mathcal X} \mathbb P[F_z \geq \kappa^*] \geq \frac{1}{|\mathcal X|} \text{ and } \min_{z \in \mathcal X} \mathbb P[F_z \geq \kappa^*] \leq \frac{1}{|\mathcal X|}.
    \end{equation}
\end{lem}
\begin{proof}
    Due to \(\kappa \mapsto \mathbb P[F_x \geq \kappa]\) being monotonously decreasing \(\forall x \in \mathcal X\), it follows that \(\kappa \mapsto \sum\nolimits_{x \in \mathcal X} \mathbb P[F_x \geq \kappa]\) is also monotonously decreasing. Consequently, logarithmic search allows one to quickly find the normalising \(\kappa^*\).
    The search window is initialised based on the insight that
    \begin{equation*}
        |\mathcal X| \min_{z \in \mathcal X} \mathbb P[F_z \geq \kappa] \leq \underbrace{\sum_{x \in \mathcal X} \mathbb P[F_x \geq \kappa]}_{\overset{!}{=} 1} \leq |\mathcal X| \max_{z \in \mathcal X} \mathbb P[F_z \geq \kappa]
    \end{equation*}
\end{proof}

\begin{lem}\label{lem:rate_of_convergence_of_binary_search_to_obtain_probability_of_optimality}
    Let \(\{F_x : x \in \mathcal X\}\) be a stochastic process with \(L_x\)-Lipschitz continuous \(\kappa \mapsto \mathbb P[F_x \geq \kappa]\) and \(1< |\mathcal X| < \infty\). Let \(\kappa^* \in \mathbb R\) be unknown such that \(\sum\nolimits_{x \in \mathcal X} \mathbb P[F_x \geq \kappa^*] = 1\). Run \(k\) steps of binary search based on a search window \(\kappa^* \in [a,b]\) according to Lemma~\ref{lem:logarithmic_search_to_obtain_probability_of_optimality} resulting in best approximant \(\kappa^{(k)}\). Then
    \begin{align*}
        |\kappa^* - \kappa^{(k)}| & \leq \frac{b-a}{2^{k+1}}\\
        |\mathbb P[F_x \geq \kappa^*] - \mathbb P[F_x \geq \kappa^{(k)}]| & \leq L_x \frac{b-a}{2^{k+1}}\quad \forall x \in \mathcal X,\\
        |\sum_{x \in \mathcal X} \mathbb P[F_x \geq k^{(k)}] - 1| & \leq |\mathcal X| L_x \frac{b-a}{2^{k+1}}.
    \end{align*}
\end{lem}
\begin{proof}
    Binary search reduces the size of the search window after \(k\) steps to \(\tfrac{b-a}{2^k}\). Define \(\kappa^{(k)}\) as the half-point of the search interval giving \(|\kappa^* - \kappa^{(k)}| \leq \tfrac{b-a}{2^{k+1}}\). Then the Lemma follows immediately by definition of Lipschitz continuity and the triangle inequality of the absolute value.
\end{proof}

\begin{restatable}%
[]{lem} {asymptoticsOfQuasiShannonInformation}\label{lem:asymptotics_of_quasi_shannon_information} The quasi-surprisal shares the asymptotics of the surprisal, i.e.
    \begin{equation*}
        \tilde I(1) = 0 = - \ln(1) \text{ and } \tilde I(r_x) \sim - \ln r_x \text{ as } r_x \to 0^+.
    \end{equation*}
\end{restatable}
\begin{proof}
    By the definition of $\tilde I$ it holds that \(\tilde I(r_x) := \frac{1}{2}(\phi(\Phi^{-1}(r_x))/r_x)^2\). Since \(t_x := \Phi^{-1}(r_x) \to - \infty\) as \(r_x \to 0^+\) and according to Lemma~\ref{lem:asymptotics_of_gaussian_cdf_and_its_inverse} both \(\phi(t_x) \sim -t_x \Phi(t_x)\) as \(t_x \to - \infty\) and \(\Phi^{-1}(r_x) \sim - \sqrt{-2\ln r_x}\) as \(r_x \to 0^+\) it follows that
    \begin{equation*}
        \phi(\Phi^{-1}(r_x))/r_x \sim - \Phi^{-1}(r_x) \cancel{r_x / r_x} \sim \sqrt{- 2 \ln r_x} \text{ as } r_x \to 0^+.
    \end{equation*}
\end{proof}

\begin{lem}[Law of large numbers for maximum, see Theorem 1 in \cite{Gnedenko1992}]\label{thm:law_of_large_numbers_for_maximum}
Suppose i.i.d. \(X_1, X_2, \ldots\) with \(\mathbb P[X_i \leq x] = F(x) \ \forall i \in \mathbb N\), where \(x \mapsto F(x)\) is continuous at all \(x > x_0\) for some \(x_0 \in \mathbb R\) and \(F(x) < 1\ \forall x \in \mathbb R\). Then
\begin{equation}\label{eq:law_of_large_numbers_for_maximum_condition}
    \exists (a_n)_{n \in \mathbb N} : \forall \epsilon > 0\ \lim_{n \to \infty} \mathbb P[|\max_{i \leq n} X_i - a_n| > \epsilon] = 0
    \iff \forall \epsilon > 0\ \lim_{x\to\infty} \frac{1- F(x+\epsilon)}{1- F(x)} = 0,
\end{equation}
giving a necessary and sufficient condition for "convergence in probability to a deterministic sequence". If such a sequence exists it can be selected as \(a_n = \inf F^{-1}(\{1- \frac{1}{n}\})\) \(\forall n \geq n_0,  a_n = 0\  \forall n < n_0\), where \(n_0\) is s.t. \(1-\frac{1}{n_0} > F(x_0)\).
\end{lem}
\begin{proof}
    Define for any \(n \in \mathbb N\) the shorthand notation \(F^n(x) = (F(x))^n\). Using
    \begin{align*}
        \mathbb P[|\max_{i \leq n} X_i - a_n | > \epsilon] & = \mathbb P[\max_{i \leq n} X_i - a_n > \epsilon] + \mathbb P[\max_{i \leq n} X_i - a_n < -\epsilon] = (1-F^n(a_n + \epsilon)) + \lim_{\delta \to 0^+} F^n(a_n - \epsilon - \delta)
    \end{align*}
    we get another representation for the left-hand-side of Equation~\eqref{eq:law_of_large_numbers_for_maximum_condition} given by
    \begin{equation*}
        1 \overset{!}{=}\lim_{n \to \infty} 1 - \mathbb P[|\max_{i \leq n} X_i - a_n| > \epsilon] = \lim_{n \to \infty} F^n(a_n + \epsilon) - \lim_{\delta \to 0^+} F^n(a_n - \epsilon - \delta),
    \end{equation*}
    which is equivalent to the conditions\footnote{Here we used that \(\lim_{n \to \infty} F^n(a_n + \epsilon) = 1 \Rightarrow \lim_{n \to \infty} F(a_n + \epsilon) = 1 \Rightarrow a_n \to \infty\) and that \(x \mapsto F(x)\) is continuous for \(x > x_0\), thus \(\lim_{\delta \to 0^+} F(a_n - \epsilon - \delta) = F(a_n - \epsilon)\) as \(n \to \infty\).}
    \begin{gather*}
        \lim_{n \to \infty} F^n(a_n + \epsilon) = 1 \land \lim_{n \to \infty} F^n(a_n - \epsilon) = 0\nonumber\\
        \iff\nonumber\\
        \lim_{n \to \infty} n \ln F(a_n + \epsilon) = 0 \land \lim_{n \to \infty} n \ln F(a_n - \epsilon) = -\infty.
    \end{gather*}
    Using a Taylor expansion of \(x \mapsto \ln x\) around \(1\), i.e. \(\ln F(x) = \ln(1- (1-F(x))) = -(1-F(x)) (1+o(1))\), one can simplify the conditions further to
    \begin{equation}\label{eq:law_of_large_numbers_for_maximum_simplified_conditions_for_LHS}
        \lim_{n \to \infty} n (1-F(a_n + \epsilon)) = 0 \land \lim_{n \to \infty} n (1- F(a_n - \epsilon)) = \infty.
    \end{equation}
    Given this simplified form, we now show both directions of the equivalence relation in Equation~\eqref{eq:law_of_large_numbers_for_maximum_condition}.

    Let us start with sufficiency, i.e. "\(\impliedby\)". Assume
    \begin{equation*}
        \forall \epsilon > 0\ \lim_{x\to\infty} \frac{1- F(x+\epsilon)}{1- F(x)} = 0.
    \end{equation*}
    Based on the continuity of the cumulative distribution function \(F(x)\) for \(x\) large enough, define\footnote{Strictly speaking the preimage could be empty. However, for \(n \geq n_0\) where \(n_0\) is such that \(1- \frac{1}{n_0} > F(x_0)\) continuity of \(F\) prevents this from occurring. Set \(a_n = 0\ \  \forall n < n_0\).} \(a_n := \inf F^{-1}(\{1-\frac{1}{n}\})\) which satisfies \(\lim_{n\to\infty} a_n = \infty\).
    Then
    \begin{equation*}
        \forall \epsilon > 0\ \ n(1- F(a_n+\epsilon)) = \frac{1- F(a_n+\epsilon)}{1- F(a_n)} \to 0 \text{ as } n \to \infty
    \end{equation*}
    and
    \begin{equation*}
        n(1- F(a_n-\epsilon)) = \frac{1- F(a_n-\epsilon)}{1- F(a_n)} = 1 \big/\frac{1- F(a_n)}{1- F(a_n-\epsilon)} \to \infty \text{ as } n \to \infty,
    \end{equation*}
    which are exactly the conditions given in Equation~\eqref{eq:law_of_large_numbers_for_maximum_simplified_conditions_for_LHS}.

    Let us proceed with necessity, i.e. "\(\implies\)". Assume \(\exists (a_n)_{n\in\mathbb N}\) such that
    \begin{equation*}
        \forall \epsilon > 0\ \ \lim_{n \to \infty} n (1-F(a_n + \epsilon)) = 0 \land \lim_{n \to \infty} n (1- F(a_n - \epsilon)) = \infty.
    \end{equation*}
    It follows that \(\lim\nolimits_{n \to \infty} F(a_n + \epsilon) = 1\) and hence \(\lim\nolimits_{n\to\infty} a_n = \infty\). Without loss of generality take \(a_n\) increasing. For any \(x \geq a_1\) determine matching \(n_x\) s.t. \(a_{n_x-1} \leq x \leq a_{n_x}\). Then \(\lim_{x \to \infty} n_x = \infty\) and it holds that for any \(\epsilon > 0\)
    \begin{equation*}
        0 \leq \frac{1- F(x+\epsilon)}{1- F(x-\epsilon)} \leq \frac{1- F(a_{n_{x}-1}+\epsilon)}{1- F(a_{n_x}-\epsilon)} = \frac{n_x (1-F(a_{{n_x}-1} + \epsilon))}{n_x (1- F(a_{n_x}-\epsilon))} \overset{x\to\infty}{\longrightarrow} 0
    \end{equation*}
    and consequently we obtain the desired result:
    \begin{equation*}
        \forall \epsilon > 0\ \lim_{x \to \infty} \frac{1- F(x+\epsilon)}{1- F(x)} = \lim_{x \to \infty} \frac{1- F(x+\epsilon/2)}{1- F(x-\epsilon/2)} = 0.
    \end{equation*}

\end{proof}

\subsection{Theoretical Basis for VAPOR in the Bandit Setting}

\begin{defn}[Cumulant generating function]\label{def:cumulant_generating_function}
    Let \(X:\Omega \to \mathbb R\) be a random variable on \((\Omega, \Sigma, \mathbb P)\). Then we define the cumulant generating function as
    \begin{equation*}
        \Psi_X: \mathbb D \to [0, \infty) \quad  \beta \mapsto \ln \mathbb E [\exp(\beta(X-\mathbb E [X])]
    \end{equation*}
    where \(\mathbb D \subseteq \mathbb R\) denotes the interior of the interval of well-definedness. 
\end{defn}

\begin{restatable}[VAPOR, adapted from Lemma 4 in \cite{tarbouriech2024probabilistic}]{prop}{vaporSaddlePointOptimisation}\label{thm:vapor_saddle_point_optimisation_problem}
    Let $F \in \mathbb R^{|\mathcal X|}$ be a random vector with \(\sigma_{F_x }\)-sub-Gaussian entries \(F_x\). Then the maximin optimisation problem
\begin{align}\label{eq:theorem_max_min_optimisation_statement}
    	\max_{p \in \Delta(\mathcal X)} \min_{\tau \in \mathbb R_+^{|\mathcal X|}} \mathcal V_\tau(p) \text{ for } \mathcal V_\tau(p) := & \sum_{x \in \mathcal X} p_x \cdot \left( \mu_{F_x} + \sigma_{F_x}^2 /(2\tau_x) - \ln ( p_x) \cdot \tau_x \right) = H_\tau(p) + \sum_{x \in \mathcal X} p_x \cdot ( \mu_{F_x} + \sigma_{F_x}^2 / (2 \tau_x))
    \end{align}
    has inner minimiser \(\tau^*_x = \sigma_{F_x} / \sqrt{-2\ln p_x} \), which simplifies the optimisation problem to
    \begin{align*}
        \max_{p \in \Delta(\mathcal X)} \!\mathcal V(p) \text{ for } \mathcal V(p) := \mathcal V_{\tau^*}(p)  & = \!\!\sum_{x \in \mathcal X} p_x \cdot \left( \mu_{F_x} + \sqrt{2 \ln (1/p_x}) \sigma_{F_x} \right) = \left\langle p, \mu_F + \sqrt{2 I(p)} \odot \sigma_F \right\rangle
    \end{align*}
    Crucially, the objective \(p \mapsto \mathcal V(p)\) is a concave\footnote{It is even strictly concave if \(\sigma_{F_x} > 0\ \forall x \in \mathcal X\).} functional. Finally, under Assumption~\ref{ass:unique_maximiser} we get the lower bounds:
    \begin{equation*}
        \max_{p \in \Delta(\mathcal X)} \mathcal V(p) \geq \mathcal V(\mathbb P[\ \cdot\ \in X^*])  \geq  \mathbb E[F^*].
    \end{equation*}
\end{restatable}
\begin{proof} The minimiser \(\tau^*\) and the minimum \(\mathcal V(p) = \mathcal V_{\tau^*}(p)\) of the (inner) minimisation in Equation~\eqref{eq:theorem_max_min_optimisation_statement} follow directly from Lemma~\ref{thm:upper_bound_on_conditional_expectation_for_subgaussians}. Let us now show (strict) concavity of \(\mathcal V(\cdot)\). It is straightforward to show that the function $g_x(a) = a(\mu_{F_x}+\sigma_{F_x} \sqrt{-2\ln a})$ is concave for \(\sigma_{F_x} \geq 0\ \forall x \in \mathcal X\) and strictly concave if \(\sigma_{F_x} > 0\ \forall x \in \mathcal X\). Now, let \(p,q \in \Delta(\mathcal X)\) and \(\lambda \in (0,1)\). Then
    \begin{align*}
        \mathcal V(\lambda p + (1-\lambda) q) & = \sum_{x \in \mathcal X} g_x(\lambda p_x + (1-\lambda) q_x)) \geq \sum_{x \in \mathcal X} \lambda g_x(p_x) + (1-\lambda) g_x(q_x) = \lambda \mathcal V(p) + (1-\lambda) \mathcal V(q),
    \end{align*}
    where the inequality is strict given \(\sigma_{F_x} > 0\ \forall x \in \mathcal X\). We have shown (strict) concavity of the objective \(\mathcal V:\Delta(\mathcal X) \to \mathbb R\). In order to establish \(\max_{p \in \Delta(\mathcal X)} \mathcal V(p)\) being lower bounded by \(\mathbb E[F^*]\) we again invoke Lemma~\ref{thm:upper_bound_on_conditional_expectation_for_subgaussians}, which yields
\begin{equation*}
        \mathbb E[F_x |\ x \in X^*] \leq \mu_{F_x } + \sigma_{F_x} \sqrt{-2\ln \mathbb P[x \in X^*]}
\end{equation*}
for any \(x \in \mathcal X\) s.t. \(\mathbb P[x \in X^*] > 0\). With this upper bound and the additional assumption of an almost surely unique optimum (Assumption~\ref{ass:unique_maximiser}), we obtain
\begin{align*}
        \mathbb E [F^*] & = \mathbb E[ E [F^* | X^*]] = \sum_{x \in \mathcal X} \mathbb P[x \in X^*] \mathbb E[F_x | x \in X^*] \leq \sum_{x \in \mathcal X} \mathbb P[x \in X^*] \left(\mu_{F_x } + \sigma_{F_x} \sqrt{-2\ln \mathbb P[x \in X^*]}\right)\\
        & = \mathcal V(\mathbb P[\ \cdot\ \in X^*]),
\end{align*}
finishing the proof.
\end{proof}

\begin{restatable}[Variational form of the KL-divergence, generalised from Theorem 3.2 in \cite{gray2011entropy}, which is limited to discrete probability spaces]{lem}{varFormOfKLDivergence}\label{thm:variational_form_of_KL_divergence}
    Fix two probability distributions \(p:\Sigma \to [0,1]\) and \(q:\Sigma \to [0,1]\) over the measurable space \((\Omega, \Sigma)\) such that \(p\) is absolutely continuous with respect to \(q\) (\(p \ll q\)). Then
    \begin{equation*}
        D_{KL}[p||q] = \sup_{X}\{\mathbb E_{p} [X] - \ln \mathbb E_{q} [\exp X]\},
    \end{equation*}
    where the supremum is taken over all measurable \(X:\Omega \to \mathbb R\) such that \(\mathbb E_{p} [X]\) and \(\mathbb E_{q} [\exp X]\) are well-defined.
\end{restatable}
\begin{proof}
    Since \(p \ll q\), there exists a Radon-Nykodym derivative\footnote{The Radon-Nykodym derivative is uniquely defined up to a set of \(q\)-measure zero.} \(\frac{d p}{dq}(\omega)\) such that $p(\mathcal A) = \int_{\mathcal A} \frac{d p}{dq}(\omega) dq(\omega)$. Setting \(X(\omega) = \ln (\tfrac{d p}{dq}(\omega))\) gives
    \begin{align*}
        \mathbb E_p [X] - \ln(\mathbb E_q [\exp X]) & = \int_{\Omega} \ln \left( \frac{dp}{dq}(\omega)\right) dp(\omega) - \ln\left(\int_{\Omega} \frac{dp}{dq}(\omega) dq(\omega) \right) = D_{KL}[p||q],
    \end{align*}
    from which well-definedness of \(\mathbb E_{p} [X]\) and \(\mathbb E_{q} [\exp X]\) also follows. Hence, we derived that \(D_{KL}[p||q] \leq \sup\nolimits_{X}\{\mathbb E_{p} [X] - \ln \mathbb E_{q} [\exp X]\}\). On the other hand, let \(X:\Omega \to \mathbb R\) be any random variable such that \(\mathbb E_{p} [X]\) and \(\mathbb E_{q} [\exp X]\) are well-defined. Then
    \begin{align}
        D_{KL}[p||q] - \left(\mathbb E_p [X] - \ln(\mathbb E_q [\exp X])\right) & = \mathbb E_p [\ln \left(\frac{dp}{dq}(\omega)\right) ] - \mathbb E_p [\ln\frac{\exp X}{\mathbb E_q [\exp X]}]\nonumber\\
        = \mathbb E_p [\ln \left(\frac{d p}{d q} (\omega) \frac{\mathbb E_q [\exp X]}{\exp (X)}\right)] & = \mathbb E_p [\ln \frac{d p}{d \lambda}] = D_{KL}[p || \lambda ]\geq 0,\label{eq:proof_variational_form_of_KL_divergence}
    \end{align}
    where we have defined the probability measure
    \begin{equation*}
        \lambda(\mathcal A) = \int_{\mathcal A} \exp (X(\omega)) / \mathbb E_q [\exp X] dq(\omega)
    \end{equation*}
    with Radon-Nykodym derivative
    \begin{equation*}
        \frac{d\lambda}{dq}(\omega) = \frac{\exp (X(\omega))}{\mathbb E_q [\exp X]} \text{ inducing another derivative } \frac{dq}{d\lambda}(\omega) = \frac{\mathbb E_q [\exp X]}{\exp (X(\omega))},
    \end{equation*}
    due to \(\lambda \ll q\) and \(q \ll \lambda\) holding both. The final key is that with \(p \ll q \ll \lambda\) one further obtains
    \begin{equation*}
        \frac{d p}{d q} (\omega) \frac{\mathbb E_q [\exp X]}{\exp (X)} = \frac{d p}{d q} \frac{dq}{d\lambda}(\omega) = \frac{d p}{d \lambda} (\omega),
    \end{equation*}
    justifying Equation~\eqref{eq:proof_variational_form_of_KL_divergence}.
\end{proof}

\begin{restatable}[Conditioned KL-divergence]{lem}{conditionedKLDivergence}\label{thm:upper_bound_on_conditional_KL_divergence}
    Consider the probability space \((\Omega, \Sigma, \mathbb P)\) and an event \(\mathcal B \in \Sigma\) of non-zero probability, i.e. \(\mathbb P[\mathcal B] > 0\). Then
    \begin{equation*}
        D_{KL}[\mathbb P[\ \cdot\ | \mathcal B]\ ||\ \mathbb P] = - \ln \mathbb P[\mathcal{B}].
    \end{equation*}
\end{restatable}
\begin{proof} By the definition of conditional expectation it holds
\begin{equation*}
    \mathbb P[\mathcal A\ | \mathcal B] = \frac{\mathbb P[\mathcal A \cap \mathcal B]}{\mathbb P[\mathcal B]} = \int_{\mathcal A} \frac{\mathds{1}_{\omega \in \mathcal B}}{\mathbb P[\mathcal B]}\ d\mathbb P(\omega) \qquad \forall \mathcal A \in \Sigma.
\end{equation*}
where we recognise absolute continuity \(\mathbb P[\ \cdot\ | \mathcal B] \ll \mathbb P\) and the Radon-Nykodym derivative
\begin{equation*}
    \frac{d \mathbb P[\ \cdot\ |\mathcal B]}{d \mathbb P}(\omega) = \frac{\mathds{1}_{\omega \in \mathcal B}}{\mathbb P[\mathcal B]}.
\end{equation*}
Hence, we obtain the following expression for the Kullback-Leibler divergence:
\begin{align*}
    D_{KL}[\mathbb P[\ \cdot\ | \mathcal B]\ ||\ \mathbb P] & = \int_\Omega \ln \frac{d \mathbb P[\ \cdot\  |\mathcal B]}{d \mathbb P} d \mathbb P[\ \cdot\ | \mathcal B] = \int_\Omega \ln \frac{\mathds{1}_{\omega \in \mathcal B}}{\mathbb P[\mathcal B]} d \mathbb P[\ \cdot\ | \mathcal B] = - \ln \mathbb P[\mathcal B]
\end{align*}
\end{proof}

\begin{lem}[Information theoretic upper bound on conditional expectation, see Theorem 1 in~\cite{o2021variational} and Lemma 11 in~\cite{tarbouriech2024probabilistic}]\label{thm:information_theoretic_upper_bound_on_conditional_expectation}
    Let \(X:\Omega \to \mathbb R\) be a random variable on \((\Omega, \Sigma, \mathbb P)\) %
    such that the cumulative generating function restricted to \(\mathbb R^+\)
    \begin{equation*}
        \Psi_X: \mathbb D \subseteq \mathbb R^+ \to [0,\infty)\quad \beta \mapsto \ln \mathbb E [\exp(\beta(X-\mathbb E [X]))]
    \end{equation*}
    exists. Assume further that \(\mathbb P[\mathcal B] > 0\) such that \(\mathbb P[\cdot | \mathcal B]\) is well-defined. Then with \(\Psi_X^*\) the convex conjugate of \(\Psi_X\) it holds
    \begin{equation*}
        \mathbb E[X | \mathcal B] \leq \mathbb E [X] + (\Psi_X^*)^{-1} (D_{KL}[\mathbb P[\cdot | \mathcal B]\ ||\ \mathbb P]).
    \end{equation*}

\end{lem}
\begin{proof}
    Since Given \(\Psi_X\) exists, \(\Psi^*_X\) is well-defined, as the cumulant generating function is non-negative and convex. 
    Let us apply Lemma~\ref{thm:variational_form_of_KL_divergence} to \(p(\mathcal E) = \mathbb P[\mathcal E| \mathcal B]\), \(q(\mathcal E) = \mathbb P[\mathcal E]\), and restrict the supremum over the random variables \(\{\lambda(X-\mathbb E X)\}_{\lambda \in \mathbb R_+}\). Then
    \begin{align*}
        D_{KL}[\mathbb P[\cdot| \mathcal B] || \mathbb P] & \geq \sup_{\lambda \in \mathbb R_+} \{ \lambda \mathbb E [X-\mathbb E[X] \ |\ \mathcal B] - \ln \mathbb E [\exp(\lambda(X-\mathbb E X))]\}\\
        & = \sup \{ \lambda (\mathbb E [X | \mathcal B] - \mathbb E[X]) - \Psi_X(\lambda) : \lambda \in \mathbb R_+\}\\
        & = \Psi_X^*(\mathbb E [X | \mathcal B] - \mathbb E[X])
    \end{align*}

    Furthermore, since \(\lambda \in \mathbb R^+\) it follows that \(\Psi_X^*\) is strictly increasing and thus admits a strictly increasing inverse which finishes the proof:
    \begin{equation*}
        (\Psi_X^*)^{-1}(D_{KL}[\mathbb P[\cdot| \mathcal B]\ ||\ \mathbb P]) \geq \mathbb E [X | \mathcal B] - \mathbb E[X].
    \end{equation*}

\end{proof}

\begin{lem}[Upper bound on the inverse of \(\Psi^*\) for sub-Gaussians]\label{lem:lower_bound_on_convex_conjugate_of_cumulant_generating_function_for_subgaussians} Let \(X:\Omega \to \mathbb R\) be a \(\sigma\)-sub-Gaussian random variable on \((\Omega, \Sigma, \mathbb P)\), i.e. \(\mathbb E[\exp(\lambda(X-\mathbb E[X]))] \leq \exp(\sigma^2 \lambda^2/2)\ \forall \lambda \in \mathbb R\). Then the cumulant generating function restricted to \(\mathbb R^+\) exists globally, i.e. \(\Psi_X: \mathbb R^+ \to [0, \infty)\) is well-defined, its convex dual \(\Psi_X^*\) is strictly increasing (and hence admits a strictly increasing inverse), and it holds that
\begin{equation*}
    (\Psi_{X}^*)^{-1}(\lambda) \leq \sigma \sqrt{2\lambda} \qquad \forall \lambda \geq 0
\end{equation*}
\end{lem}
\begin{proof}
    Since \(X\) is sub-Gaussianity, then the cumulant generating function \(\Psi_X:\mathbb R^+ \to [0, \infty)\) exists on all of \(\mathbb R^+\) with
    \begin{equation*}\label{eq:proof_upper_bound_on_inverse_cumulant_generating_function_for_subgaussians_sub_gaussianity_consequence}
        \Psi_{X}(\beta) \leq \beta^2 \sigma^2 / 2 \qquad \forall \beta > 0.
    \end{equation*}
    We then get its strictly increasing convex dual
    \begin{equation*}
        \Psi_{X}^*(\alpha) = \sup\{ \alpha \beta - \Psi_{X}(\beta) : \beta \in \mathbb R^+\}\qquad \alpha \in \mathbb R.
    \end{equation*}
    Consequently, \(\alpha \mapsto \Psi_{X}^*(\alpha)\) admits an inverse derived by
    \begin{align*}
        \alpha = (\Psi_{X}^*)^{-1}(\lambda) & \iff \lambda = \sup\{ \alpha \beta - \Psi_{X}(\beta) : \beta \in \mathbb R^+\}\\
        & \iff \lambda \geq \alpha \beta - \Psi_{X}(\beta)\ \forall \beta \in \mathbb R^+\\
        & \quad\!\, \land \forall \epsilon > 0\ \exists\ \beta_\epsilon \in \mathbb R^+ \text{ s.t. } \lambda - \epsilon < \alpha \beta_\epsilon - \Psi_{X}(\beta_\epsilon)\\
        & \iff \alpha \leq \lambda / \beta + \Psi_X(\beta)/\beta \ \forall \beta \in \mathbb R^+\\
        & \quad\!\, \land \forall \epsilon > 0\ \exists\ \beta_\epsilon \in \mathbb R^+ \text{ s.t. } \alpha + \epsilon > \lambda/\beta_\epsilon + \Psi_{X}(\beta_\epsilon) / \beta_\epsilon \\
        & \iff  (\Psi_{X}^*)^{-1}(\lambda) = \alpha = \inf \{\lambda/\beta  + \Psi_X(\beta)/\beta: \beta \in \mathbb R^+\}.
    \end{align*}
    Finally, plugging in Equation~\eqref{eq:proof_upper_bound_on_inverse_cumulant_generating_function_for_subgaussians_sub_gaussianity_consequence} we obtain
    \begin{equation*}
        (\Psi_{X}^*)^{-1}(\lambda) \leq \inf \{\lambda/\beta  + \beta \sigma^2/2: \beta \in \mathbb R^+\} = \sigma \sqrt{2\lambda}\ \forall \lambda \geq 0.
    \end{equation*}
\end{proof}

\begin{lem}[Upper bound on conditional expectation for sub-Gaussians]\label{thm:upper_bound_on_conditional_expectation_for_subgaussians}
    Let \(X:\Omega \to \mathbb R\) be a \(\sigma\)-sub-Gaussian random variable on \((\Omega, \Sigma, \mathbb P)\), i.e. it satisfies \(\mathbb E[\exp(\lambda(X-\mathbb E[X]))] \leq \exp(\sigma^2 \lambda^2/2)\ \forall \lambda \in \mathbb R\), and let \(\mathcal B \in \Sigma\) such that \(\mathbb P[\mathcal B] > 0\). Then
    \begin{equation*}
        \mathbb E[X | \mathcal B] \leq \mathbb E[X] + \sigma \sqrt{-2\ln \mathbb P[\mathcal B]} = \mathbb E[X] + \min\{\frac{\sigma^2}{2s} - s \ln \mathbb P[\mathcal B]: s > 0\},
    \end{equation*}
    with minimiser \(s^* = \sigma / \sqrt{-2\ln \mathbb P[\mathcal B]}\).
\end{lem}
\begin{proof}
    According to Lemma~\ref{lem:lower_bound_on_convex_conjugate_of_cumulant_generating_function_for_subgaussians} the cumulant generating function \(\Psi_X\) restricted to \(\mathbb R^+\) exists everywhere and its convex conjugate admits an inverse with upper bound:
    \begin{equation*}
        (\Psi_X^*)^{-1}(\lambda) \leq \sigma \sqrt{2\lambda}.
    \end{equation*}
    Moreover, according to Lemma~\ref{thm:information_theoretic_upper_bound_on_conditional_expectation} and Lemma~\ref{thm:upper_bound_on_conditional_KL_divergence} it holds that
    \begin{equation*}
        \mathbb E[X | \mathcal B] \leq \mathbb E[X] + (\Psi_X^*)^{-1}(-\ln \mathbb P[\mathcal B]).
    \end{equation*}
    Combining the two equations yields the first desired statement:
    \begin{equation*}
        \mathbb E[X | \mathcal B] \leq \mathbb E[X] + \sigma \sqrt{-2\ln \mathbb P[\mathcal B]}.
    \end{equation*}
    Finally, a separate examination of the first order condition of
    \begin{equation*}
        \min\{\frac{\sigma^2}{2s} - s \ln \mathbb P[\mathcal B]: s > 0\}
    \end{equation*}
    results in the minimum \(\sigma \sqrt{-2 \ln \mathbb P[\mathcal B]}\) for the minimiser \(s^* = \sigma / \sqrt{-2\ln \mathbb P[\mathcal B]}\).

\end{proof}

\end{document}